\documentclass[11pt]{article}

\usepackage[a4paper,left=20mm,right=20mm,top=20mm,bottom=20mm]{geometry}

\usepackage{fancyhdr}
\usepackage{url}

\usepackage[english]{babel}

\usepackage[utf8]{inputenc}
\usepackage{xspace}
\usepackage{amsmath,amsthm,amssymb}
\usepackage{lmodern}

\usepackage[algo2e,ruled,vlined,linesnumbered]{algorithm2e}
\newcommand{\assign}{\leftarrow}

\usepackage{xcolor}
\usepackage{tikz}
\usepackage{graphicx}

\allowdisplaybreaks[4]
\newtheorem{theorem}{Theorem}
\newtheorem{lemma}[theorem]{Lemma}

\newtheorem{definition}[theorem]{Definition}

\newcommand{\om}{\textsc{OneMax}\xspace}
\newcommand{\onemax}{\om}
\newcommand{\lo}{\textsc{LeadingOnes}\xspace}
\newcommand{\leadingones}{\lo}

\newcommand{\R}{\ensuremath{\mathbb{R}}}

\newcommand{\N}{\ensuremath{\mathbb{N}}}

\newcommand{\eps}{\varepsilon}

\newcommand{\A}{\ensuremath{\mathcal{A}}}
\newcommand{\F}{\ensuremath{\mathcal{F}}}
\DeclareMathOperator{\E}{E}

\newcommand{\OM}{\textsc{OM}}
\newcommand{\LO}{\textsc{LO}}
\newcommand{\BV}{\textsc{BV}}
\newcommand{\binval}{\textsc{BinaryValue}}
\newcommand{\jump}{\textsc{Jump}\xspace}
\usepackage{framed,mdframed}
\usepackage{multirow}
\usepackage{array}
\usepackage{bigdelim}
\usepackage{dsfont}
\usepackage{hyperref}

\begin{document}

\lhead{\emph{C. Doerr}}
\chead{Survey Article}
\rhead{Black-Box Complexity}
\cfoot{\thepage}

\newcommand{\note}[1]{{\color{red}\it #1}}

\title{Complexity Theory for Discrete Black-Box Optimization Heuristics}
	\author{
		Carola Doerr}
		\date{
Sorbonne Universit\'e, CNRS, Laboratoire d'informatique de Paris 6, LIP6, 75252 Paris, France
\vspace{1.5ex}
\today
}
\maketitle

\sloppy{
\begin{abstract}
A predominant topic in the theory of evolutionary algorithms and, more generally, theory of randomized black-box optimization techniques is 
\emph{running time analysis.} Running time analysis aims at understanding the performance of a given heuristic on a given problem by bounding the number of function evaluations that are needed by the heuristic to identify a solution of a desired quality. 
As in general algorithms theory, this running time perspective is most useful when it is complemented by a meaningful
\emph{complexity theory} that studies the limits of algorithmic solutions. 

In the context of discrete black-box optimization, several \emph{black-box complexity models} have been developed to analyze the best possible performance that a black-box optimization algorithm can achieve on a given problem. The models differ in the classes of algorithms to which these lower bounds apply. This way, black-box complexity contributes to a better understanding of how certain algorithmic choices (such as the amount of memory used by a heuristic, its selective pressure, or properties of the strategies that it uses to create new solution candidates) influences performance. 

In this chapter we review the different black-box complexity models that have been proposed in the literature, survey the bounds that have been obtained for these models, and discuss how the interplay of running time analysis and black-box complexity can inspire new algorithmic solutions to well-researched problems in evolutionary computation. We also discuss in this chapter several interesting open questions for future work. 
\end{abstract}

\vspace{4.5ex}
\textbf{Reference:} This survey article is to appear (in a slightly modified form) in the book ``Theory of Randomized Search Heuristics in Discrete Search Spaces'', which will be published by Springer in 2018. The book is edited by Benjamin Doerr and Frank Neumann.\\
Missing numbers in the pointers to other chapters of this book will be added as soon as possible.

\newpage
\tableofcontents

\pagebreak
\section{Introduction and Historical Remarks}\label{sec:BBCintro}

One of the driving forces in theoretical computer science is the fruitful interplay between \emph{complexity theory} and the \emph{theory of algorithms.} While the former measures the minimal computational effort that is needed to solve a given problem, the latter aims at designing and analyzing efficient algorithmic solutions, which prove that a problem can be solved with a certain computational effort. When for a problem the lower bound for the resources needed to solve it are identical to (or not much smaller than) the upper bounds attained by some specific algorithm, we can be certain to have an (almost) optimal algorithmic solution for this problem. Big gaps between lower and upper bounds, in contrast, indicate that more research effort is needed to understand the problem: it may be that more efficient algorithms for the problem exist, or that the problem is indeed ``harder'' than what the  lower bound suggests. 

Many different complexity models co-exist in the theoretical computer science literature. The arguably most classical one measures the number of arithmetic operations that an algorithm needs to perform on the problem data until it obtains a solution for the problem. A solution can be a ``yes/no'' answer (decision problem), a classification of a problem instance according to some criteria (classification problem), a vector of decision variables that maximize or minimize some objective function (optimization problem), etc. In the optimization context, we are typically only interested in algorithms that satisfy some minimal quality requirements such as a guarantee that the suggested solutions (``the output'' of the algorithm) are always optimal, optimal with some large enough probability, or that they are not worse than an optimal solution by more than some additive or multiplicative factor $C$, etc. 

In the \emph{white-box} setting, in which the algorithms have full access to the data describing the problem instance, complexity theory is a well-established and very intensively studied research objective. In \emph{black-box optimization,} where the algorithms do not have access to the problem data and can learn about the problem at hand only through the evaluation of potential solution candidates, complexity-theory is a much less present topic, with a rather big fluctuation in the number of publications. In the context of heuristic solutions to black-box optimization problems, which is the topic of this book, complexity-theory has been systematically studied only after 2010, under the notion of \emph{black-box complexity}. Luckily, black-box complexity theory can build on results in related research domains such as information theory, discrete mathematics, cryptography, and others. 

In this book chapter we review the state of the art in this currently very active area of research, which is concerned with bounding the best possible performance that an optimization algorithm can achieve in a black-box setting.  

\subsection{Black-Box vs. White-Box Complexity}\label{sec:BBCintrowhite}

Most of the traditional complexity measures assume that the algorithms have access to the problem data and count the number of steps that are needed until they output a solution. In the black-box setting, these complexity measures are not very meaningful, as the algorithms are asked to optimize a problem without having direct access to it. As a consequence, the performance of a black-box optimization algorithm is therefore traditionally measured by the number of function evaluations that the algorithm does until it queries for the first time a solution that satisfies some specific performance criteria. In this book, we are mostly interested in the expected number of evaluations needed until an \emph{optimal} solution is evaluated for the first time. It is therefore natural to define black-box complexity as the \emph{minimal number of function evaluations that any black-box algorithm needs to perform, on average, until it queries for the first time an optimal solution.} 

We typically regard classes of problems, e.g., the set of traveling salesperson instances of planar graphs with integer edge weights. For such a class $\F \subseteq \{f:S\to \R \}$ of problem instances, we take a worst-case view and measure the expected number of function evaluations that an algorithm needs to optimize any instance $f \in \F$. That is, the black-box complexity of a problem $\F$ is $\inf_{A} \sup_{f \in \F} \E[T(A,f)]$, the best (among all algorithms $A$) worst-case (among all problem instances $f$) expected number $\E[T(A,f)]$ of function evaluations that are needed to optimize any $f \in \F$. A formal definition will be given in Section~\ref{sec:BBCunrestricted}. 

The black-box complexity of a problem can be much different from its white-box counterparts. We will discuss, for example, in Sections~\ref{sec:BBCNPhard} and ~\ref{sec:BBCunbiasedpartition} that there are a number of NP-hard problems whose black-box complexity is of small polynomial order.

\subsection{Motivation and Objectives}
The ultimate objective of black-box complexity is to support the investigation and the design of efficient black-box optimization techniques. This is achieved in several complementing ways. 

A first benefit of black-box complexity is that it enables the above-mentioned evaluation of how well we have understood a black-box optimization problem, and how suitable the state-of-the-art heuristics are. Where large gaps between lower and upper bounds exist, we may want to explore alternative algorithmic solutions, in the hope to identify more efficient solvers. Where lower and upper bounds match or are close, we can stop striving for more efficient algorithms. 

Another advantage of black-box complexity studies is that they allow to investigate how certain algorithmic choices influence the performance: By restricting the class of algorithms under consideration, we can judge how these restrictions increase the complexity of a black-box optimization problem. In the context of evolutionary computation, interesting restrictions include the amount of memory that is available to the algorithms, the number of solutions that are sampled in every iteration, the way new solution candidates are generated, the selection principles according to which it is decided which search points to keep for future reference, etc. Comparing the unrestricted with the restricted black-box complexity of a problem (i.e., its black-box complexity with respect to \emph{all} vs. with respect to a \emph{subclass} of all algorithms) quantifies the performance loss caused by these restrictions. This way, we can understand, for example, the effects of not storing the set of all previously evaluated solution candidates, but only a small subset.

The black-box complexity of a problem can be significantly smaller than the performance of a best known ``standard'' heuristic. In such cases, the small complexity is often attained by a problem-tailored black-box algorithm, which is not representative for common black-box heuristics. Interestingly, it turns out that we can nevertheless learn from such highly specific algorithms, as they often incorporate some ideas that could be beneficial much beyond the particular problem at hand. As we shall demonstrate in Section~\ref{sec:BBCalgo}, even for very well-researched optimization problems, such ideas can give rise to the design of novel heuristics which are provably more efficient than standard solutions. This way, black-box complexity serves as a source of inspiration for the development of novel algorithmic ideas that lead to the design of better search heuristics.

\subsection{Relationship to Query Complexity}
As indicated above, black-box complexity is studied in several different contexts, which reach far beyond evolutionary computation. In the sixties and seventies of the twentieth century, for example, this complexity measure was very popular in the context of combinatorial games, such as coin-weighing problems of the type ``given $n$ coins of two different types, what is the minimal number of weighing that is needed to classify the coins according to their weight?''. Interpreting a weighing as a function evaluation, we see that such questions can be formulated as black-box optimization problems. 

Black-box complexity also plays an important role in cryptography, where a common research question concerns the minimal amount of information that suffices to break a secret code. Quantum computing, communication complexity, and information-theory are other research areas where (variants of) black-box complexity are intensively studied performance measures. While in these settings the precise model is often not exactly identical to one faced in black-box optimization, some of the tools developed in these related areas can be useful in our context. 

A significant part of the literature studies the performance of deterministic algorithms. Randomized black-box complexities are much less understood. They can be much smaller than their deterministic counterparts. Since deterministic algorithms form a subclass of randomized ones, any lower bound proven for the randomized black-box complexity of a problem also applies to any deterministic algorithm. In some cases, a strict separation between deterministic and randomized black-box complexities can be proven. This is the case for the \leadingones function, as we shall briefly discuss in Section~\ref{sec:BBCunrestrictedLO}. For other problems, the deterministic and randomized black-box complexity coincide. Characterizing those problems for which the access to random bits can provably decrease the complexity is a widely open research question.

In several contexts, in particular the research domains mentioned above, black-box complexity is typically referred to as \emph{query} or \emph{oracle complexity}, with the idea that the algorithms do not evaluate the function values of the solution candidates themselves but rather query them from an oracle. This interpretation is mostly identical to the black-box scenario classically regarded in evolutionary computation.

\subsection{Scope of this Book Chapter}

Here in this chapter, as in the remainder of this book, we restrict our attention to discrete optimization problems, i.e., the maximization or minimization of functions $f:S\rightarrow \R$ that are defined over finite search spaces $S$. As in the previous chapters, we will mostly deal with the optimization of pseudo-Boolean functions $f:\{0,1\}^n \rightarrow \R$, permutation problems $f:S_n \rightarrow \R$, and functions $f:[0..r-1]^n \rightarrow \R$ defined for strings over an alphabet of bounded size, where we abbreviate here and in the following by $[0..r-1]:=\{0,1,\ldots,r-1\}$ the set of non-negative integers smaller than $r$, $[n]:=\{1,2,\ldots,n\}$, and by $S_n$ the set of all permutations (one-to-one maps) $\sigma: [n] \to [n]$.

We point out that black-box complexity notions are also studied for infinite search spaces $S$. In the context of continuous optimization problems, black-box complexity aims to bound the best possible \emph{convergence rates} that a derivative-free black-box optimization algorithm can achieve, cf.~\cite{TeytaudG06,FournierT11} for examples.

\subsection{Target Audience and Complementary Material}
This book chapter is written with a person in mind who is familiar with black-box optimization, and who brings some background in theoretical running time analysis. We will give an exhaustive survey of existing results. Where appropriate, we provide proof ideas and discuss some historical developments. Readers interested in a more gentle introduction to the basic concepts of black-box complexity are referred to~\cite{JansenBBCchapter}. A slide presentation on selected aspects of black-box complexity, along with a summary of complexity bounds known back in spring 2014, can be found in the tutorial~\cite{DoerrDBBCTutorial}.

\subsection{Overview of the Content}
Black-box complexity is formally defined in Section~\ref{sec:BBCunrestricted}. We also provide there a summary of useful tools. In Section~\ref{sec:BBCNPhard} we discuss why classical complexity statements like NP-hardness results do not necessarily imply hardness in the black-box complexity model. 

In Sections~\ref{sec:BBCunrestrictedResults} to~\ref{sec:BBCcombined} we review the different black-box complexity models that have been proposed in the literature. For each model we discuss the main results that have been achieved for it. For several benchmark problems, including most notably \onemax, \leadingones, and \textsc{jump}, but also combinatorial problems like the minimum spanning tree problem and shortest paths problems, bounds have been derived for various complexity models. For \onemax and \leadingones we compare these different bounds in Section~\ref{sec:BBCtables}, to summarize where gaps between upper and lower bounds exist, and to highlight increasing complexities imposed by the restrictive models. 

We will demonstrate in Section~\ref{sec:BBCalgo} that the complexity-theoretic view on black-box optimization can inspire the design of more efficient optimization heuristics. This is made possible by questioning some of the state-of-the-art choices that are made in evolutionary computation and neighboring disciplines. 

We finally show in Section~\ref{sec:BBCmastermind} that research efforts originally motivated by the study of black-box complexity has yield improved bounds for long-standing open problems in classical computer science. 

In Section~\ref{sec:BBCconclusions} we conclude this chapter with a summary of open questions and problems in discrete black-box complexity and directions for future work.

\section{The Unrestricted Black-Box Model}
\label{sec:BBCunrestricted}

In this section we introduce the most basic black-box model, which is the \emph{unrestricted} one. This model contains all black-box optimization algorithms. Any lower bound in this model therefore immediately applies to any of the restricted models which we discuss in Sections~\ref{sec:BBCmemory} to~\ref{sec:BBCcombined}. We also discuss in this section some useful tools for the analysis of black-box complexities and demonstrate that the black-box complexity of a problem can be much different from its classical white-box complexity.

The unrestricted black-box model has been introduced by Droste, Jansen, Wegener in~\cite{DrosteJW06}. The only assumption that it makes is that the algorithms do not have any information about the problem at hand other than the fact that it stems from some function class $\F \subseteq \{f:S\to \R\}$. The only way an unrestricted black-box algorithm can learn about the instance $f$ is by evaluating the function values $f(x)$ of potential solution candidates $x \in S$. We can assume that the evaluation is done by some oracle, from which $f(x)$ is queried. In the unrestricted model, the algorithms can update after any such query the strategy by which the next search point(s) are generated. In this book, we are mostly interested in the performance of \emph{randomized black-box heuristics,} so that these strategies are often \emph{probability distributions} over the search space from which the next solution candidates are sampled. This process continues until an optimal search point $x \in \arg\max f$ is queried for the first time. 

The algorithms that we are interested in are thus those that maintain a probability distribution $D$ over the search space $S$. In every iteration, a new solution candidate $x$ is sampled from this distribution and the function value $f(x)$ of this search point is evaluated. After this evaluation, the probability distribution $D$ is updated according to the information gathered through the sample $(x,f(x))$. The next iteration starts again by sampling a search point from this updated distribution $D$, and so on. This structure is summarized in Algorithm~\ref{alg:BBCunrestricted}, which models \emph{unrestricted randomized black-box algorithms.} A visualization is provided in Figure~\ref{fig:BBCunrestricted}.

\begin{algorithm2e}%
 \textbf{Initialization:}
 Sample $x^{(0)}$ according to some probability distribution $D^{(0)}$ over $S$ and query $f(x^{(0)})$\;
 \textbf{Optimization:}	
\For{$t=1,2,3,\ldots$}{
  \label{line:BBCunrestrictedmut} Depending on $\big((x^{(0)},f(x^{(0)})), \ldots, (x^{(t-1)},
  f(x^{(t-1)}))\big)$ choose a probability distribution $D^{(t)}$ over $S$ and sample $x^{(t)}$ according to $D^{(t)}$\;
  Query $f(x^{(t)})$\;
 }\caption{Blueprint of an unrestricted randomized black-box algorithm.}
\label{alg:BBCunrestricted}
\end{algorithm2e}

Note that in Algorithm~\ref{alg:BBCunrestricted} in every iteration only one new solution candidate is sampled. In contrast, many evolutionary algorithms and other black-box optimization techniques generate and evaluate several search points \emph{in parallel.} It is not difficult to see that lower bounds obtained for the here-described unrestricted black-box complexity immediately apply to such \emph{population-based} heuristics, since an unrestricted algorithm is free to ignore information obtained from previous iterations. As will be commented in Section~\ref{sec:BBCparallel}, the \emph{parallel} black-box complexity of a function can be (much) larger than its sequential variant. Taking this idea to the extreme, i.e., requiring the algorithm to neglect information obtained through previous queries yields so-called \emph{non-adaptive} black-box algorithms. A prime example for a non-adaptive black-box algorithm is random sampling (with and without repetitions). Non-adaptive algorithms play only a marginal role in evolutionary computation. From a complexity point of view, however, it can be interesting to study how much adaptation is needed for an efficient optimization, cf. also the discussions in Sections~\ref{sec:BBCunrestrictedOM} and~\ref{sec:BBCmastermind}. For most problems, the adaptive and non-adaptive complexity differ by large factors. For other problems, however, the two complexity notions coincide, see Section~\ref{sec:BBCunrestrictedneedle} for an example.

Note also that unrestricted black-box algorithms have access to the full history of previously evaluated solutions. The effects of restricting the available memory to a \emph{population} of a certain size will be the focus of the memory-restricted black-box models discussed in Section~\ref{sec:BBCmemory}. 
 
In line~\ref{line:BBCunrestrictedmut} we do not specify how the probability distribution $D^{(t)}$ is chosen. Thus, in principle, the algorithm can spend significant time to choose this distribution. This can result in small polynomial black-box complexities for NP-hard problems, cf. Section~\ref{sec:BBCNPhard}. Droste, Jansen, and Wegener~\cite{DrosteJW06} therefore suggested to restrict the set of algorithms to those that execute the choice of the distributions $D^{(t)}$ in a polynomial number of algebraic steps (i.e., polynomial time in the input length, where ``time'' refers to the classically regarded complexity measure). They call this model the \emph{time-restricted model.} In this book chapter we will not study this time-restricted model. That is, we allow the algorithms to spend arbitrary time on the choice of the distributions $D^{(t)}$. This way, we obtain very general lower bounds. Almost all upper bounds stated in this book chapter nevertheless apply also to the time-restricted model. The polynomial bounds for NP-hard problems form, of course, an exception to this rule. 

We finally comment on the fact that Algorithm~\ref{alg:BBCunrestricted} runs forever. As we have seen in previous chapters of this book, the pseudocode in Algorithm~\ref{alg:BBCunrestricted} is a common representation of black-box algorithms in the theory of heuristic optimization. Not specifying the termination criterion is justified by our performance measure, which is the expected number of function evaluations that an algorithm performs until (and including) the first iteration in which an optimal solution is evaluated, cf. Definition~\ref{def:BBCcomplexity} below. Other performance measures for black-box heuristics have been discussed in the literature, cf.~\cite{CarvalhoD17,JansenZ14,DoerrJWZ13}, but in the realm of black-box complexity, the \emph{average optimization time} is still the predominant performance indicator. Confer Section~\ref{sec:BBCconclusions} for a discussion of extending existing results to other, possibly more complex performance measures.  

\subsection{Formal Definition of Black-Box Complexity}
\label{sec:BBCunrestrictedformal}

In this section, we give a very general definition of black-box complexity. More precisely, we formally define the the black-box complexity of a class $\F$ of functions with respect to some class $\A$ of algorithms. The unrestricted black-box complexity will be the complexity of $\F$ with respect to all black-box algorithms that follow the blueprint provided in Algorithm~\ref{alg:BBCunrestricted}. 

For a black-box optimization algorithm $A$ and a function $f:S \to \R$, let $T(A,f)\in \R\cup\{\infty\}$ be the number of function evaluations that algorithm $A$ does until and including the evaluation in which it evaluates for the first time an optimal search point $x \in \arg\max f$. As in previous chapters, we call $T(A,f)$ the \emph{running time of $A$ for $f$} or, synonymously, the \emph{optimization time of $A$ for $f$}. When $A$ is a randomized algorithm, $T(A,f)$ is a random variable that depends on the random decisions made by $A$. We are mostly interested in its expected value $\E[T(A,f)]$. 

With this performance measure in place, the definition of the black-box complexity of a class $\F$ of functions $S \to \R$ with respect to some class $\A$ of algorithms now follows the usual approach in complexity theory. 
\begin{definition}
\label{def:BBCcomplexity}
For a given black-box algorithm $A$, the \emph{$A$-black-box complexity of $\F$} is 
$$\E[T(A,\F)]:=~\sup_{f \in \F}{\E[T(A,f)]},$$ 
the worst-case expected running time of $A$ on $\F$. 

The \emph{$\A$-black-box complexity of $\F$} is 
$$\E[T(\A,\F)] := \inf_{A\in\A}{\E[T(A,\F)]},$$ 
the minimum (``best'') complexity among all $A \in \A$ for $\F$.  
\end{definition}
%

Thus, formally, the \emph{unrestricted black-box complexity} of a problem class $\F$ is $\E[T(\A,\F)]$ for $\A$ being the collection of all unrestricted black-box algorithms; i.e., all algorithms that can be expressed in the framework of Algorithm~\ref{alg:BBCunrestricted}. 

The following lemma formalizes the intuition that every lower bound for the unrestricted black-box model also applies to any restricted black-box model. 

\begin{lemma}\label{lem:BBCunrestrictedlower}
Let $\F \subseteq \{ f: S\to\R \}$. For every collection $\A'$ of black-box optimization algorithms for $\F$, the $\A'$-black-box complexity of $\F$ is at least as large as its unrestricted one.  
\end{lemma}
Formally, this lemma holds because $\A'$ is a subclass of the set $\A$ of all unrestricted black-box algorithms. The infimum in the definition of $\E[T(\A',\F)]$ is therefore taken over a smaller class, thus giving values that are at least as large as $\E[T(\A,\F)]$. 

\begin{figure}[t]
\begin{framed}
\begin{center}
\includegraphics[scale=1.5]{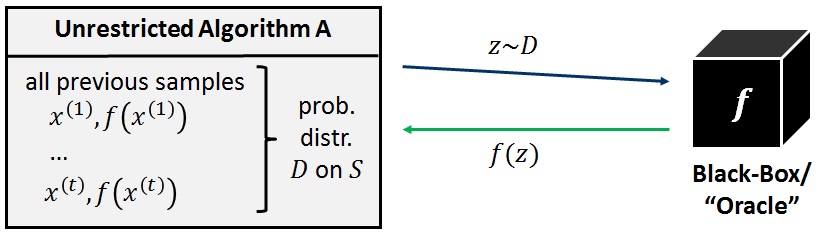}
\end{center}
\caption{In the unrestricted black-box model, the algorithms can store the full history of previously queried search points. For each of these already evaluated candidate solutions $x$ the algorithm has access to its absolute function value $f(x) \in \R$. There are no restrictions on the structure of the distributions $D$ from which new solution candidates are sampled. 
}
\end{framed}
\label{fig:BBCunrestricted}
\end{figure}

\subsection{Tools for Proving Lower Bounds}\label{sec:BBClower}

Lemma~\ref{lem:BBCunrestrictedlower} shows that the unrestricted black-box complexity of a class $\F$ of functions is a lower bound for the performance of any black-box algorithm on $\F$. In other words, no black-box algorithm can optimize $\F$ more efficiently than what the unrestricted black-box complexity of $\F$ indicates. We are therefore particularly interested in proving \emph{lower bounds} for the black-box complexity of a problem. This is the topic of this section. 

To date, the most powerful tool to prove lower bounds for randomized query complexity models like our unrestricted black-box model is the so-called \emph{minimiax principle} of Yao~\cite{Yao77}. In order to discuss this principle, we first need to recall that we can interpret every randomized unrestricted black-box algorithm as a probability distribution over deterministic ones. In fact, randomized black-box algorithms are often defined this way. 

\emph{Deterministic black-box algorithms} are those for which the probability distribution in line~\ref{line:BBCunrestrictedmut} of Algorithm~\ref{alg:BBCunrestricted} are one-point distributions. That is, for every $t$ and for every sequence $\big( (x^{(0)},f(x^{(0)})), \ldots, (x^{(t-1)},f(x^{(t-1)}))\big)$ of previous queries, there exists a search point $s\in S$ such that $D^{(t)}\Big(\big((x^{(0)},f(x^{(0)})), \ldots, (x^{(t-1)},f(x^{(t-1)}))\big)\Big)(s)=1$ and $D^{(t)}\Big(\big((x^{(0)},f(x^{(0)})), \ldots, (x^{(t-1)},f(x^{(t-1)}))\big)\Big)(y)=0$ for all $y \neq s$. In other words, we can interpret deterministic black-box algorithms as \emph{decision trees.} A decision tree for a class $\F$ of functions is a rooted tree in which the nodes are labeled by the search points that the algorithm queries. The first query is the label of the root note, say $x^{(0)}$. The edges from the root node to its neighbors are labeled with the possible objective values $\{g(x^{(0)}) \mid g \in \F\}$. After evaluating $f(x^{(0)})$, the algorithm follows the (unique) edge $\{x^{(0)}, x^{(1)}\}$ which is labeled with the value $f(x^{(0)})$. The next query is the label of the endpoint $x^{(1)}$ of this edge. We call $x^{(1)}$ a level-1 node. The level-2 neighbors of $x^{(0)}$ (i.e., all neighbors of $x^{(1)}$ except the root node $x^{(0)}$) are the potential search points to be queried in the next iteration. As before, the algorithm chooses as next query the neighbor $x^{(2)}$ of $x^{(1)}$ to which the unique edge labeled with the value $f(x^{(1)})$ leads. This process continues until an optimal search point has been queried. The optimization time $T(A,f)$ of the algorithm $A$ on function $f$ equals the depth of this node plus one (the plus one accounts for the evaluation of the root node). 

We easily see that, in this model, it does not make sense to query the same search point twice. Such a query would not reveal any new information about the objective function $f$. For this reason, on every rooted path in the decision tree every search point appears at most once. This shows that the \emph{depth of the decision tree} is bounded by $|S|-1$. The \emph{width} of the tree, however, can be as large as the size of the set $\F(S):=\{ g(s) \mid g \in \F, s \in S\}$, which can be infinite or even uncountable; e.g., if $\F$ equals the set of all linear or monotone functions $f:\{0,1\}^n \to \R$. As we shall see below, Yao's minimax principle can only be applied to problems for which $\F(S)$ is finite. Luckily, it is often possible to identify subclasses $\F'$ of $\F$ for which $\F'(S)$ is finite and whose complexity is identical or not much smaller than that of the whole class $\F$.

When $S$ and $\F(S)$ are finite, the number of (non-repetitive) deterministic decision trees, and hence the number of deterministic black-box algorithms for $\F$ is finite. In this case, we can apply Yao's minimax principle. This theorem, intuitively speaking, allows to restrict our attention to bounding the expected running time $\E[T(A,f)]$ of a best-possible \emph{deterministic} algorithm $A$ on a \emph{random} instance $f$ taken from $\F$ according to some probability distribution $p$. By Yao's minimax principle, this best possible expected running time is a lower bound for the expected performance of a best possible \emph{randomized} algorithm on an \emph{arbitrary} input. In our words, it is thus a lower bound for the unrestricted black-box complexity of the class $\F$. 

Analyzing deterministic black-box algorithms is often considerably easier than directly bounding the performance of any possible randomized algorithm. An a priori challenge in applying this theorem is the identification of a probability distribution $p$ on $\F$ for which the expected optimization time of a best-possible deterministic algorithm is large. Luckily, for many applications some rather simple distributions on the inputs suffice; for example the uniform one, which assigns to each problem instance $f\in\F$ equal probability. Another difficulty in the application is the above-mentioned identification of subclasses $\F'$ of $\F$ for which $\F'(S)$ is finite.

Formally, Yao's minimax principle reads as follows.  
\begin{theorem}[Yao's minimax principle]
\label{thm:BBCYao}
Let $\Pi$ be a problem with a finite set $\mathcal{I}$ of input instances (of a fixed size) permitting a finite set $\A$ of deterministic algorithms. Let $p$ be a probability distribution over $\mathcal{I}$ and $q$ be a probability distribution over $\A$. Then, 
\begin{align*}
	\min_{A \in \A} \E[T(I_p, A)] \leq \max_{I \in \mathcal{I}} \E[T(I,A_q)]\, , 
\end{align*}
where $I_p$ denotes a random input chosen from $\mathcal{I}$ according to $p$, $A_q$ a random algorithm chosen from $\A$ according to $q$, and $T(I,A)$ denotes the running time of algorithm $A$ on input $I$.
\end{theorem}
The formulation of Theorem~\ref{thm:BBCYao} is taken from the book by Motwani and Raghavan~\cite{MotwaniR97}, where an extended discussion of this principle can be found. 

A straightforward, but still quite handy application of Yao's minimax principle gives the following lower bound.
\begin{theorem}[Simple information-theoretic lower bound, Theorem~2 in~\cite{DrosteJW06}]
\label{thm:BBCinfotheo}
Let $S$ be finite. Let $\F$ be a set of functions $\{f:S\to \R\}$ such that for every $s \in S$ there exists a function $f_s \in \F$ for which the size of $f_s(S):=\{ f_s(x) \mid x \in S\}$ is bounded by $k$ and for which $s$ is a unique optimum; i.e., $\arg\max f_s = \{s\}$ and $|f_s(S)| \le k$. The unrestricted black-box complexity of $\F$ is at least $\lceil \log_{k}(|S|) \rceil - 1$.
\end{theorem}
To prove Theorem~\ref{thm:BBCinfotheo} it suffices to select for every $s \in S$ one function $f_s$ as in the statement and to regard the uniform distribution over the set $\{f_s \mid s \in S\}$. Every deterministic black-box algorithm that eventually solves any instance $f_s$ has to have at least one node labeled $s$. We therefore need to distribute all $|S|$ potential optima on the decision tree that corresponds to this deterministic black-box algorithm. Since the outdegree of every node is bounded from above by $k$, the average distance of a node to the root is at least $\lceil \log_{k}(|S|) \rceil -2$.

An informal interpretation of Theorem~\ref{thm:BBCinfotheo}, which in addition ignores the rounding of the logarithms, is as follows. In the setting of Theorem~\ref{thm:BBCinfotheo}, optimizing a function $f_s$ corresponds to \emph{learning} $s$. A binary encoding of the optimum $s$ requires $\log_2(|S|)$ bits. With every query, we obtain at most $\log_2(k)$ bits of information; namely the amount of bits needed to encode which of the at most $k$ possible objective value is assigned to the queried search point. We therefore need to query at least $\log_2(|S|)/\log_2(k)=\log_k(|S|)$ search points to obtain the information that is required to decode $s$. This hand-wavy interpretation often gives a good first idea of the lower bounds that can be proven by Theorem~\ref{thm:BBCinfotheo}. 

This intuitive proof for Theorem~\ref{thm:BBCinfotheo} shows that it works best if at every search point \emph{exactly} $k$ answers are possible, and each of them is equally likely. This situation, however, is not typical for black-box optimization processes, where usually only a (possibly small) subset of function values are likely to appear next. As a rule of thumb, the larger the difference of the potential function value to the function value of the current-best solution, the less likely an algorithm is to obtain it in the next iteration. Such \emph{transition probabilities} are not taken into account in Theorem~\ref{thm:BBCinfotheo}. The theorem does also not cover very well the situation in which at a certain step less than $k$ answers are possible. Even for fully symmetric problem classes this situation is likely to appear in the later parts of the optimization process, where those problem instances that are still aligned with all previously evaluated function values all map the next query to one out of less than $k$ possible function values. Covering these two shortcomings of Theorem~\ref{thm:BBCinfotheo} is one of the main challenges in black-box complexity. One step into this direction is the \emph{matrix lower bound theorem} presented in~\cite{BuzdalovDK16} and the subsequent work~\cite{Buzdalov16}. As also acknowledged there, however, the verification of the conditions under which these two generalizations apply is often quite tedious, so that the two methods are unfortunately not yet easily and very generally applicable. So far, they have been used to derive lower bounds for the black-box complexity of the \onemax and the \textsc{jump} benchmark functions, cf. Sections~\ref{sec:BBCunrestrictedOM} and~\ref{sec:BBCunrestrictedjump}.

Another tool that will be very useful in the subsequent sections is the following theorem, which allows to transfer lower bounds proven for a simpler problem to a problem that is derived from it by a composition with another function. Most notably, it allows to bound the black-box complexity of functions of unitation (i.e, functions for which the function value depends only on the number of ones in the string) by that of the \onemax problems. We will apply this theorem to show that the black-box complexity of the jump functions is at least as large as that of \onemax, cf. Section~\ref{sec:BBCunrestrictedjump}. 

\begin{theorem}[Generalization of Theorem~2 in~\cite{DoerrDK15jump}]
\label{thm:BBClowerBoundForMapped} 
For all problem classes $\F$, all classes of algorithms $\A$, and all maps $g: \mathbb{R} \rightarrow \mathbb{R}$ that are such that for all $f \in \F$ it holds that $\{x \mid g(f(x)) \mbox{ optimal}\,\} = \{x \mid f(x) \mbox{ optimal}\,\}$ the $\A$-black-box complexity of $g(\F):=\{g \circ f \mid f \in \F \}$ is at least as large as that of $\F$. 
\end{theorem}
The intuition behind Theorem~\ref{thm:BBClowerBoundForMapped} is that with the knowledge of $f(x)$, we can compute $g(f(x))$, so that every algorithm optimizing $g(\F)$ can also be used to optimize $\F$, by evaluating the $f(x)$-values, feeding $g(f(x))$ to the algorithm, and querying the solution candidates that this algorithm suggests.

\subsection{Tools to Prove Upper Bounds}
\label{sec:BBCtoolsupper}


We now present general upper bounds for the black-box complexity of a problem. We recall that, by definition, a small upper bound for the black-box complexity of a problem $\F$ shows that there exists an algorithm which solves every problem instance $f \in \F$ efficiently. When the upper bound of a problem is smaller than the expected performance of well-understood search heuristics, the question whether these state-of-the-art heuristics can be improved or whether the unrestricted black-box model is too generous arises. 

The simplest upper bound for the black-box complexity of a class $\F$ of functions is the expected performance of random sampling without repetitions. 

\begin{lemma}\label{lem:BBCuppersimple}
For every finite set $S$ and every class $\F\subset \{f:S\to \R\}$ of real-valued functions over $S$, the unrestricted black-box complexity of $\F$ is at most $(|S|+1)/2$. 
\end{lemma}
This simple bound can be tight, as we shall discuss in Section~\ref{sec:BBCunrestrictedneedle}. A similarly simple upper bound is presented in the next subsection. 

\subsubsection{Function Classes vs. Individual Instances}
\label{sec:BBCclasses}
In all of the above we have discussed the black-box complexity of a \emph{class} of functions, and not of individual problem instances. This is justified by the following observation, which also explains why in the following we will usually regard generalizations of the benchmark problems typically studied in the theory of randomized black-box optimization. 

\begin{lemma}\label{lem:BBCone}
For every function $f:S\to \R$, the unrestricted black-box complexity of the class $\{f\}$ that consists only of $f$ is one. The same holds for any class $\F$ of functions that all have their optimum in the same point, i.e., for which there exists a search point $x \in S$ such that, for all $f \in \F$, $x \in \arg\max f$ holds. 

More generally, if $\F$ is a collection of functions $f:S\to \R$ and $X \subseteq S$ is such that for all $f \in \F$ there exists at least one point $x\in X$ such that $x \in \arg\max f$, the unrestricted black-box complexity of $\F$ is at most $(|X|+1)/2$.

For every finite set $\F$ of functions, the unrestricted black-box complexity is bounded from above by $(|\F|+1)/2$. 
\end{lemma} 
The proof of this lemma is quite straightforward. For the first statement, the algorithm which queries any point in $\arg\max f$ in the first query certifies this bound. Similarly, the second statement is certified by the algorithm that queries $x$ in the first iteration. The algorithm which queries the points in $X$ in random order proves the third statement. Finally, note that the third statement implies the fourth by letting $X$ be the set that contains for each function $f\in\F$ one optimal solution $x_f \in \arg\max f$.  

Lemma~\ref{lem:BBCone} indicates that function classes $\F$ for which $\cup_{f \in \F} \arg\max f$, or, more precisely for which a small set $X$ as in the third statement of Lemma~\ref{lem:BBCone} exists, are not very interesting research objects in the unrestricted black-box model. We therefore typically choose the generalizations of the benchmark problems in a way that any set $X$ which contains for each objective function $f \in \F$ at least one optimal search point has to be large. We shall often even have $|X|=|\F|$; i.e., the optima of any two functions in $\F$ are pairwise different. 

We will see in Section~\ref{sec:BBCunbiased} that Lemma~\ref{lem:BBCone} does not apply to all of the restricted black-box models. In fact, in the unary unbiased black-box model regarded there, the black-box complexity of a single function can be of order $n \log n$. That is, even if the algorithm ``knows'' where the optimum is, it may still need $\Omega(n \log n)$ steps to generate it. 

\subsubsection{Upper Bounds via Restarts}
\label{sec:BBCrestart}

In several situations, rather than bounding the expected optimization time of a black-box heuristic, it can be easier to show that the probability that it solves a given problem within $s$ iterations is at least $p$. If $p$ is large enough (for an asymptotic bound it suffices that this success probability is constant), then a restarting strategy can be used to obtain upper bounds on the black-box complexity of the problem. Either the algorithm is done after at most $s$ steps, or it is initialized from scratch, independently of all previous runs. This way, we obtain the following lemma. 
%
%

\begin{lemma}[Remark~7 in~\cite{DoerrKLW13}]
\label{lem:BBChighprobability}
Suppose for a problem $\F$ there exists an unrestricted black-box algorithm $A$ that, with constant success probability, solves any instance $f \in \F$ in $s$ iterations (that is, it queries an optimal solution within $s$ queries). Then the unrestricted black-box complexity of $\F$ is at most $O(s)$.
\end{lemma}

Lemma~\ref{lem:BBChighprobability} also applies to almost all of the restricted black-box models that we will discuss in Sections~\ref{sec:BBCmemory} to~\ref{sec:BBCcombined}. In general, it applies to all black-box models in which restarts are allowed. It does not apply to the (strict version of the) elitist black-box model, which we discuss in Section~\ref{sec:BBCelitist}.

\subsection{Polynomial Bounds for NP-Hard Problems}\label{sec:BBCNPhard}
Our discussion in Section~\ref{sec:BBCintrowhite} indicates that the classical complexity notions developed for white-box optimization and decision problems are not very meaningful in the black-box setting. This is impressively demonstrated by a number of NP-hard problems that have a small polynomial black-box complexity. We present such an example, taken from~\cite[Section~3]{DrosteJW06}.

One of the best-known NP-complete problems is \textsc{MaxClique}. For a given graph $G=(V,E)$ of $|V|=n$ nodes and for a given parameter $k$, it asks whether there exists a complete subgraph $G'=(V'\subseteq V, E':=E\cap \{\{u,v\} \in E \mid u,v \in V'\})$ of size $|V'| \ge k$. A complete graph is the graph in which every two vertices are connected by a direct edge between them. The optimization version of \textsc{MaxClique} asks to find a complete subgraph of largest possible size. A polynomial-time optimization algorithm for this problem implies P=NP. 

The unrestricted black-box complexity of \textsc{MaxClique} is, however, only of order $n^2$. This bound can be achieved as follows. In the first $\binom{n}{2}$ queries, the algorithm queries the presence of individual edges. This way, it learns the structure of the problem instance. From this information, all future solution candidates can be evaluated without any oracle queries. That is, a black-box algorithm can now compute an optimal solution \emph{offline;} i.e., without the need for further function evaluations. This offline computation may take exponential time, but in the black-box complexity model, we do not charge the algorithm for the time needed between two queries. The optimal solution of the \textsc{MaxClique} instance can be queried in the $\big(\binom{n}{2}+1\big)$-st query.

\begin{theorem}[Section~3 in~\cite{DrosteJW06}]
The unrestricted black-box complexity of \textsc{MaxClique} is at most $\binom{n}{2}+1$ and thus $O(n^2)$.
\end{theorem}
Several similar results can be obtained. For most of the restricted black-box complexity models this has been explicitly done, cf. also Section~\ref{sec:BBCunbiasedpartition}. 

One way to avoid such small complexities would be to restrict the time that an algorithm can spend between any two queries. This suggestion was made in~\cite{DrosteJW06}. In our opinion, this requirement would, however, carry a few disadvantages such as a mixture of different complexity measures. We will therefore, here in this book chapter, not explicitly verify that the algorithms run in polynomial time. Most upper bounds are nevertheless easily seen to be obtained by polynomial-time algorithms. Where polynomial bounds are proven for NP-hard problems, there must be at least one iteration for which the respective algorithm, by today's knowledge, needs excessive time.

\section{Known Black-Box Complexities in the Unrestricted Model} 
\label{sec:BBCunrestrictedResults}

We survey existing results for the unrestricted black-box model, and proceed by problem type. For each considered benchmark problem we first introduce its generalization to classes of similar problem instances. We discuss which of the original problem characteristics are maintained in these generalizations. We will see that for some classical benchmark problems, different generalizations have been proposed in the literature. 

\subsection{Needle}
\label{sec:BBCunrestrictedneedle}

Our first benchmark problem is an example that shows that the simple upper bound given in Lemma~\ref{lem:BBCuppersimple} can be tight. The function that we generalize is the \textsc{Needle} function, which assigns $0$ to all search points $s\in S$ except for one distinguished optimum, which has a function value of one. In order to obtain the above-mentioned property that every function in the generalize class has a different optimum than any other function (cf. discussion after Lemma~\ref{lem:BBCone}), while at the same time maintaining the problem characteristics, the following generalization is imposes itself. For every $s \in S$ we let $f_s:S \to \R$ be the function which assign function value $1$ to the unique optimum $s \in S$ and $0$ to all other search points $x\neq s$. We let $\textsc{Needle}(S):=\{f_s \mid s \in S\}$ be the set of all such functions. 

Confronted with such a function $f_s$ we do not learn anything about the target string $s$ until we have found it. It seems quite intuitive that the best we can do in such a case is to query search points at random, without repetitions. That this is indeed optimal is the statement of the following theorem, which can be easily proven by Yao's minimax principle applied to $\textsc{Needle}(S)$ with the uniform distribution. 

\begin{theorem}[Theorem~1 in~\cite{DrosteJW06}]
\label{thm:BBCneedle}
For every finite set $S$, the unrestricted black-box complexity of $\textsc{Needle}(S)$ is $(|S|+1)/2$.
\end{theorem}  

\subsection{OneMax}\label{sec:BBCunrestrictedOM}

The certainly best-studied benchmark function in the theory of randomized black-box optimization is \onemax. \onemax assigns to each bit string~$x$ of length~$n$ the number $\sum_{i=1}^n{x_i}$ of ones in it. The natural generalization of this particular function to a non-trivial class of functions is as follows.

\begin{definition}[\onemax]
\label{def:BBCOM}
For all $n \in \N$ and all $z \in \{0,1\}^n$ let 
$$\OM_z: \{0,1\}^n \rightarrow [0..n], x \mapsto \OM_z(x)=|\{i \in [n] \mid x_i=z_i\}|,$$
the function that assigns to each length-$n$ bit string $x$ the number of bits in which $x$ and $z$ agree. Being the unique optimum of $\OM_z$, the string $z$ is called its \textbf{target string}. 

We refer to $\onemax_n :=\left\{\OM_z \mid z \in \{0,1\}^n \right\}$ as the set of all (generalized) \onemax functions. We often omit the subscript $n$.
\end{definition}
We easily observe that for every $n$ the original \onemax function $\OM$ counting the number of ones corresponds to $\OM_{(1,\ldots,1)}$. It is furthermore not difficult to prove that for every $z \in \{0,1\}^n$ the \emph{fitness landscape} of $\OM_z$ is \emph{isomorphic} to that of $\OM$. This can be seen by observing that $\OM_z(x)=\OM(x \oplus z \oplus (1,\ldots,1))$ for all $x,z \in \{0,1\}^n$, which shows that $\OM_z = \OM \circ \alpha_z$ for the Hamming automorphism $\alpha_z:\{0,1\}^n \to \{0,1\}^n, x \mapsto x \oplus z \oplus (1,\ldots,1)$. As we shall discuss in Section~\ref{sec:BBCunbiased}, a Hamming automorphism is a one-to-one map $\alpha:\{0,1\}^n \to \{0,1\}^n$ such that for all $x$ and all $z$ the Hamming distance of $x$ and $z$ is identical to that of $\alpha(x)$ and $\alpha(z)$. 
This shows that the generalization of \OM~to functions $\OM_z$ preserves its problem characteristics. In essence, the generalization is just a ``re-labeling'' of the search points. 

\textbf{The unrestricted black-box complexity of OneMax.} With Definition~\ref{def:BBCOM} at hand, we can study the unrestricted black-box complexity of this important class of benchmark functions. 

Interestingly, it turns out that the black-box complexity of $\onemax_n$ has been studied in several different contexts; much before Droste, Jansen, and Wegener introduced black-box complexity. In fact, already Erd{\H{o}}s and R\'enyi~\cite{ErdR63} as well as several other authors studied it in the early 60s of the last century, inspired by a question about so-called \emph{coin-weighing problems.} 

In our terminology, Erd{\H{o}}s and R\'enyi~\cite{ErdR63} showed that the unrestricted black-box complexity of \onemax is at least $(1-o(1)) n/\log_2(n)$ and at most $(1+o(1)) \log_2(9)n/\log_2(n)$. The upper bound was improved to $(1+o(1)) 2n/\log_2(n)$ in~\cite{Lindstrom64,Lindstrom65,CantorM66}. Identical or weaker bounds have been proven several times in the literature. Some works appeared at the same time as the work of Erd{\H{o}}s and R\'enyi (cf. the discussion in~\cite{Bshouty09}), some much later~\cite{DrosteJW06,AnilW09,Bshouty09}. 
\begin{theorem}[\cite{ErdR63,Lindstrom64,Lindstrom65,CantorM66}]
\label{thm:BBCunrestrictedOM}
The unrestricted black-box complexity of \onemax is at least $(1-o(1)) n/\log_2(n)$ and at most $(1+o(1)) 2n/\log_2(n)$. It is thus $\Theta(n/\log n)$.
\end{theorem}
The lower bound of Theorem~\ref{thm:BBCunrestrictedOM} follows from Yao's minimax principle, applied to $\onemax_n$ with the uniform distribution. Informally, we can use the arguments given after Theorem~\ref{thm:BBCinfotheo}: since the optimum can be anywhere in $\{0,1\}^n$, we need to learn the $n$ bits of the target string $z$. With each function evaluation, we receive at most $\log_2(n+1)$ bits of information, namely the objective value, which is an integer between $0$ and $n$. We therefore need at least (roughly) $n/\log_2(n+1)$ iterations. Using Theorem~\ref{thm:BBCinfotheo}, this reasoning can be turned into a formal proof. 

The upper bound given in Theorem~\ref{thm:BBCinfotheo} is quite interesting because it is obtained by a very simple strategy. Erd{\H{o}}s and R\'enyi showed that $O(n/\log n)$ bit strings sampled independently and uniformly at random from the hypercube $\{0,1\}^n$ have a high probability of revealing the target string. That is, an asymptotically optimal unrestricted black-box algorithm for \onemax can just sample $O(n/\log n)$ random samples. From these samples and the corresponding objective values, the target string can be identified without further queries. Its computation, however, may not be possibly in polynomial time. That $\onemax_n$ can be optimized in $O(n/\log n)$ queries also in polynomial time was proven in~\cite{Bshouty09}.\footnote{Bshouty mentions that also the constructions of Lindstr\"om and Cantor and Mills can be done in polynomial time. But this was not explicitly mentioned in their works. The work of Bshouty also has the advantage that it generalizes to \onemax functions over alphabets larger than $\{0,1\}$, cf. also Section~\ref{sec:BBCmastermind}.} The reader interested in a formal analysis of the strategy by Erd{\H{o}}s and R\'enyi may confer Section~3 of \cite{DoerrJKLWW11}, where a detailed proof of the $O(n/\log n)$ random sampling strategy is presented. 

In the context of \emph{learning}, it is interesting to note that the random sampling strategy by Erd{\H{o}}s and R{\'e}nyi is \emph{non-adaptive;} i.e., the $t$-th search point does not depend on the previous $t-1$ evaluations. In the black-box context, a last query, in which the optimal solution is evaluated, is needed. This query certainly depends on the previous $O(n/\log n)$ evaluations, but note that here we know the answer to this evaluation already (with high probability). For non-adaptive strategies, learning $z$ with $(1+o(1)) 2n/\log n$ queries is optimal~\cite{ErdR63}. The intuitive reason for this lower bound is that a random guess typically has an objective value close to $n/2$. More precisely, instead of using the whole range of $n+1$ possible answers, almost all function values are in an $O(\sqrt{n})$ range around $n/2$, giving, very informally, the lower bound $\log_2(n)/\log_2(O(\sqrt{n}))=\Omega(2n/\log n)$. 

Using the probabilistic method (or the constructive result by Bshouty~\cite{Bshouty09}), the random sampling strategy can be derandomized. This derandomization says that for every $n$, there is a sequence of $t = \Theta(n/\log n)$ strings $x^{(1)}, \ldots, x^{(t)}$ such that the objective values $\OM_z(x^{(1)}), \ldots, \OM_z(x^{(t)})$ uniquely determine the target string $z$. Such a derandomized version will be used in later parts of this chapter, e.g., in the context of the $k$-ary unbiased black-box complexity of \onemax studied in Section~\ref{sec:BBCunbiasedOM}. 

\begin{theorem}[from~\cite{ErdR63} and others]
\label{thm:BBCunrestrictedOMderandomized}
For every $n$ there is a sequence $x^{(1)}, \ldots, x^{(t)}$ of $t=\Theta(n/\log n)$ bit strings such that for every two length-$n$ bit strings $y \neq z$ there exists an index $i$ with $\OM_z(x^{(i)}) \neq \OM_y(x^{(i)})$.
%
\end{theorem}

For very concrete \onemax instances, i.e., for instances of bounded dimension $n$, very precise bounds for the black-box complexity are known, cf.~\cite{Buzdalov16} and the pointers in~\cite[Section~1.4]{DoerrDST16} for details. Here in this chapter we are only concerned with the asymptotic complexity of $\onemax_n$ with respect to the problem dimension $n$. Non-surprisingly, this benchmark problem will also be studied in almost all of the restricted black-box models that we describe in the subsequent sections. A summary of known results can be found in Section~\ref{sec:BBCtables}.

\subsection{BinaryValue}\label{sec:BBCunrestrictedBV}

Another intensively studied benchmark function is the binary-value function 
$\BV(x) := \sum_{i=1}^n 2^{i-1} x_i$, which assigns to each bit string the value of the binary number it represents.
As $2^{i} > \sum_{j=1}^{i} 2^{j-1}$, the bit value of the bit~$i+1$ dominates the effect of all bits $1, \ldots, i$ on the function value.

Two straightforward generalizations of $\BV$ to function classes exist. The first one is the collection of all functions 
$$\BV_z: \{0,1\}^n \rightarrow [0..2^{n}], x \mapsto \sum_{i=1}^n{2^{i-1} \mathds{1}(x_i,z_i)},$$
where $\mathds{1}(a,b):=1$ if and only if $a=b$ and $\mathds{1}(a,b):=0$ otherwise. In light of Definition~\ref{def:BBCOM}, this may seem like a natural extension of $\BV$ to a class of functions. It also satisfies our sought condition that for any two functions $\BV_z \neq \BV_{z'}$ the respective optima $z$ and $z'$ differ, so that the smallest set containing for each function its optimum is the full $n$-dimensional hypercube $\{0,1\}^n$. However, we easily see that the unrestricted black-box complexity of the so-defined set $\binval_n^*:=\{\BV_z \mid z \in \{0,1\}^n \}$ is very small. 

\begin{theorem}[Theorem~4 in~\cite{DrosteJW06}]
\label{thm:BBCunrestrictedBV1}
The unrestricted black-box complexity of $\binval_n^*$ is $2-2^{-n}$.
\end{theorem}
\begin{proof}
The lower bound follows from observing that for an instance $\BV_z$ for which $z$ is chosen uniformly at random, the probability to query the optimum $z$ in the first query is $2^{-n}$. In all other cases at least two queries are needed. 

For the upper bound, we only need to observe that for any two target strings $z \neq z'$ and for every search point $x \in \{0,1\}^n$ we have $\BV_{z}(x) \neq \BV_{z'}(x)$. More precisely, it is easy to see that from $\BV_{z}(x)$ we can easily determine for which bits $i \in [n]$ the bit value of $x_i$ is identical to $z_i$. This shows that from querying the objective value of a random string in the first query we can compute the optimum $z$, which we query in the second iteration, if the first one was not optimal already. 
\end{proof}

Theorem~\ref{thm:BBCunrestrictedBV1} is possible because the objective values disclose a lot of information about the target string. A second generalization of $\BV$ has therefore been suggested in the literature. In light of the typical behavior of black-box heuristics, which do not discriminate between the bit positions, and in particular with respect to the unbiased black-box model defined in Section~\ref{sec:BBCunbiased}, this variant seems to be the more ``natural'' choice in the context of evolutionary algorithms. This second generalization of $\BV$ collects all functions $\BV_{z,\sigma}$ defined as  
$$\BV_{z,\sigma}: \{0,1\}^n \rightarrow \N_0, x \mapsto \sum_{i=1}^n{2^{i-1} \delta(x_{\sigma(i)},z_{\sigma(i)})}\,.$$
Denting by $\sigma(x)$ the string $(x_{\sigma(1)} \ldots x_{\sigma(n)})$, we easily see that $\BV_{z,\sigma}(x)=\BV\big(\sigma(x \oplus z \oplus (1,\ldots,1))\big)$, thus showing that the class $\{ \BV_{z,\sigma} \mid z \in \{0,1\}^n, \sigma \in S_n \}$ can be obtained from $\BV$ by composing it with an $\oplus$-shift of the bit values and a permutation of the indices $i \in [n]$. Since $z=\arg\max \BV_{z,\sigma}$, we call $z$ the \emph{target string} of $\BV_{z,\sigma}$. Similarly, we call $\sigma$ the \emph{target permutation} of $\BV_{z,\sigma}$. 

Going through the bit string one by one, i.e., flipping one bit at a time, shows that at most $n+1$ function evaluations are needed to optimize any $\BV_{z,\sigma}$ instance. This simple upper bound can be improved by observing that for each query $x$ and for each $i \in [n]$ we can derive from $\BV_{z,\sigma}(x)$ whether or not $x_{\sigma(i)}=z_{\sigma(i)}$, even if we cannot yet locate $\sigma(i)$. Hence, all we need to do is to identify the target permutation $\sigma$. This can be done by a binary search, which gives the following result.

\begin{theorem}[Theorem~16 in~\cite{DoerrW14ranking}]
\label{thm:BBCunrestrictedBV}
The unrestricted black-box complexity of $\binval_n:=\{\BV_{z,\sigma} \mid z \in \{0,1\}^n, \sigma \in S_n \}$ is at most $\lceil \log_2 n \rceil+2$.
\end{theorem}
In a learning sense, in which we want to \emph{learn} both $z$ and $\sigma$, the bound of Theorem~\ref{thm:BBCunrestrictedBV} is tight, as, informally, the identification of $\sigma$ requires to learn $\Theta(\log (n!))=\Theta(n \log n)$ bits, while with every query we obtain $\log_2(2^n)=n$ bits of information. In our optimization context, however, we do not necessarily need to learn $\sigma$ in order to optimize $\BV_{z,\sigma}$. A similar situation will be discussed in Section~\ref{sec:BBCunrestrictedLO}, where we study the unrestricted black-box complexity of \leadingones. For \leadingones, it could be formally proven that the complexity of optimization and learning are identical (up to at most $n$ queries). We are not aware of any formal statement showing whether or not a similar argument holds for the class $\binval_n$.

\subsection{Linear Functions}\label{sec:BBCunrestrictedLin}
$\OM$ and $\BV$ are representatives of the class of linear functions $f:\{0,1\}^n \to \R, x \mapsto \sum_{i=1}^n{f_i x_i}$. We can generalize this class in the same way as above to obtain the collection
$$\textsc{Linear}_n:=\Big\{ f_z:\{0,1\}^n \to \R, x \mapsto \sum_{i=1}^n{f_i \mathds{1}(x_i,z_i)} \mid z \in \{0,1\}^n \Big\}$$ 
of generalized linear functions. $\onemax_n$ and $\binval_n$ are both contained in this class. 

Not much is known about the black-box complexity of this class. The only known bounds are summarized by the following theorem. 

\begin{theorem}[Theorem~\ref{thm:BBCunrestrictedOM} above and Theorem~4 in~\cite{DrosteJW06}]
\label{thm:BBCunrestrictedLin}
The unrestricted black-box complexity of the class $\textsc{Linear}_n$ is at most $n+1$ and at least $(1-o(1)) n/\log_2 n$.
\end{theorem}
The upper bound is attained by the algorithm that starts with a random or a fixed bit string $x$ and flips one bit at a time, using the better of parent and offspring as starting point for the next iteration. A linear lower bound seems likely, but has not been formally proven.

\subsection{Monotone and Unimodal Functions}\label{sec:BBCunrestrictedMonotone}

For the sake of completeness, we mention that the class $\textsc{Linear}_n$ is a subclass of the class of generalized monotone functions. 

\begin{definition}[Monotone functions]
\label{def:BBCmonotone}
Let $n \in \N$ and let $z \in \{0,1\}^n$. 
A function $f:\{0,1\}^n \rightarrow \R$ is said to be \emph{monotone with respect to $z$} if for all $y, y'\in \{0,1\}^n$ with $\{i \in [n] \mid y_i = z_i\} \subsetneq \{i \in [n] \mid y'_i = z_i\}$ it holds that $f(y) < f(y')$. The class $\textsc{Monotone}_n$ contains all such functions that are monotone with respect to some $z \in \{0,1\}^n$.
\end{definition}

The above-mentioned algorithm which flips one bit at a time (cf. discussion after Theorem~\ref{thm:BBCunrestrictedLin}) solves any of these instances in at most $n+1$ queries, giving the following theorem.
\begin{theorem}
\label{thm:BBCunrestrictedMonotone}
The unrestricted black-box complexity of the class $\textsc{Monotone}_n$ is at most $n+1$ and at least $(1-o(1)) n/\log_2 n$.
\end{theorem}

Monotone functions are instances of so-called \emph{unimodal functions.} A function $f$ is unimodal if and only if for every non-optimal search point $x$ there exists a direct neighbor $y$ of $x$ with $f(y)>f(x)$. The unrestricted black-box complexity of this class of unimodal functions is studied in~\cite[Section~8]{DrosteJW06}, where a lower bound that depends on the number of different function values that the objective functions can map to is presented.

\subsection{LeadingOnes}\label{sec:BBCunrestrictedLO}
After \onemax, the probably second-most investigated function in the theory of discrete black-box optimization is the leading-ones function $\LO:\{0,1\}^n \to [0..n]$, which assigns to each bit string $x$ the length of the longest prefix of ones; i.e., 
$\LO(x):=\max \{ i \in [0..n] \mid \forall j \in [i]: x_{j} = 1 \}$. Like for \binval, two generalizations have been studied, an $\oplus$-invariant version and an $\oplus$- and permutation-invariant version. In view of the unbiased black-box complexity model which we will discuss in Section~\ref{sec:BBCunbiased}, the latter is the more frequently studied.  

\begin{definition}[LeadingOnes function classes]
\label{def:leadingones}
Let $n \in \N$. For any $z \in \{0,1\}^n$ let 
$$\LO_z: \{0,1\}^n \rightarrow \N, x \mapsto \max \{ i \in [0..n] \mid \forall j \in [i]: x_{j} = z_{j} \}\,,$$ 
the length of the maximal joint prefix of $x$ and $z$. Let $\leadingones^*_n:=\{ \LO_{z} \mid z \in \{ 0,1\}^n \}\,.$

For $z \in \{0,1\}^n$ and $\sigma \in S_n$ let
$$\LO_{z,\sigma}:\{0,1\}^n \rightarrow \N, x \mapsto \max \{ i \in [0..n] \mid \forall j \in [i]: x_{\sigma(j)} = z_{\sigma(j)} \}\,,$$ 
the maximal joint prefix of $x$ and $z$ with respect to $\sigma$. 
The set $\leadingones_n$ is the collection of all such functions; i.e., 
$$\leadingones_n:=\{ \LO_{z, \sigma} \mid z \in \{ 0,1\}^n, \sigma \in S_n \}\,.$$
\end{definition}

The unrestricted black-box complexity of the set $\leadingones^*_n$ is easily seen to be around $n/2$. This is the complexity of the algorithm which starts with a random string $x$ and, given an objective value of $\LO_z(x)$, replaces $x$ by the string that is obtained from $x$ by flipping the $\LO_z(x)+1$-st bit in $x$. The lower bound is a simple application of Yao's minimax principle applied to the uniform distribution over all possible problem instances. It is crucial here to note that the algorithms do not have any information about the ``tail'' $(z_j \ldots z_n)$ until it has seen for the first time a search point of function value at least $j-1$.

\begin{theorem}[Theorem~6 in~\cite{DrosteJW06}]
\label{thm:BBCunrestrictedLO2}
The unrestricted black-box complexity of the set $\leadingones^*_n$ is $n/2 \pm o(n)$. The same holds for the set $\{ \LO_{z,\sigma} \mid z \in \{0,1\}^n\}$, for any fixed permutation $\sigma \in S_n$.
\end{theorem}

The unrestricted black-box complexity of $\leadingones_n$ is also quite well understood. 
\begin{theorem}[Theorem~4 in~\cite{AfshaniADLMW12}]
\label{thm:BBCunrestrictedLO}
The unrestricted black-box complexity of $\leadingones_n$ is $\Theta(n \log\log n)$.
\end{theorem}
Both the upper and the lower bounds of Theorem~\ref{thm:BBCunrestrictedLO} are quite involved. For the lower bound, Yao's minimax principle is applied to the uniform distribution over the instances $\LO_{z,\sigma}$ with $z_{\sigma(i)} := (i \mod 2), i=1,\dots,n$. Informally, this choice indicates that the complexity of the \leadingones problem originates in the difficulty of identifying the target permutation. Indeed, as soon as we know the permutation, we need at most $n+1$ queries to identify the target string $z$ (and only around $n/2$ on average, by Theorem~\ref{thm:BBCunrestrictedLO2}). To measure the amount of information that an algorithm can have about the target permutation $\sigma$, a potential function is designed that maps each search point $x$ to a real number. To prove the lower bound in Theorem~\ref{thm:BBCunrestrictedLO} it is shown that, for every query $x$, the expected increase in this potential, for a uniform problem instance $\LO_{z,\sigma}$, is not very large. Using drift analysis, this can be used to bound the expected time needed to accumulate the amount of information needed to uniquely determine the target permutation. 

The proof of the upper bound will be sketched in Section~\ref{sec:BBCunbiasedLO}, in the context of the unbiased black-box complexity of $\leadingones_n$. 

It may be interesting to note that the $O(n \log\log n)$ bound of Theorem~\ref{thm:BBCunrestrictedLO} cannot be achieved by deterministic algorithms. In fact, Theorem~3 in~\cite{AfshaniADLMW12} states that the \emph{deterministic unrestricted black-box complexity} of $\leadingones_n$ is $\Theta(n \log n)$.

\subsection{Jump Function Classes}\label{sec:BBCunrestrictedjump}
Another class of popular pseudo-Boolean benchmark functions are so-called ``jump'' functions. In black-box complexity, this class is currently one of the most intensively studied problems, with a number of surprising results, which in addition carry some interesting ideas for potential refinements of state-of-the-art heuristics. For this reason, we discuss this class in more detail, and compare the known complexity bounds to running time bounds for some standard and recently developed heuristics. 

For a non-negative integer $\ell$, the function $\jump_{\ell,z}$ is derived from the \onemax function $\OM_z$ with target string $z \in \{0,1\}^n$ by ``blanking out'' any useful information within the strict $\ell$-neighborhood of the optimum $z$ and its bit-wise complement $\bar{z}$; by giving all these search points a fitness value of $0$. In other words, 
\begin{align}\label{def:BBCjump}
\jump_{\ell,z}(x) := 
\begin{cases}
n,				&\mbox{if }\OM_z(x) = n,\\
\OM_z(x),		&\mbox{if }\ell < \OM_z(x) < n-\ell,\\
0,				&\mbox{otherwise.}
\end{cases}
\end{align}
This definition is mostly similar to the two, also not fully agreeing, definitions used in~\cite{DrosteJW02} and~\cite{LehreW10} that we shall discuss below.

\subsubsection{Known Running Time Bounds for Jump Functions} \label{sec:BBCjumpalgos}
We summarize known running time results for the optimization of jump functions via randomized optimization heuristics. The reader only interested in black-box complexity results can skip this section.

\cite{DrosteJW02} analyzed the optimization time of the (1+1)~EA on \jump functions. From their work, it is not difficult to see that the expected run time of the (1+1)~EA on $\jump_{\ell,z}$ is $\Theta(n^{\ell+1})$, for all $\ell \in \{1, \ldots, \lfloor n/2 \rfloor-1\}$ and all $z \in \{0,1\}^n$. This running time is dominated by the time needed to ``jump'' from a local optimum $x$ with function value $\OM_z(x)=n-\ell-1$ to the unique global optimum $z$. 

The \emph{fast genetic algorithm} proposed in~\cite{FastGA17} significantly reduces this running time by using a generalized variant of standard bit mutation, which goes through its input and flips each bit independently with probability $c/n$. Choosing in every iteration the expected step size $c$ in this mutation rate $c/n$ from a power-law distribution with exponent $\beta$ (more precisely, in every iteration, $c$ is chosen independently from all previous iterations, and independently of the current state of the optimization process), an expected running time of $O(\ell^{\beta-0.5} ((1+o(1))\tfrac{e}{\ell})^{\ell} n^{\ell})$ on $\jump_{\ell,z}$ is achieved, uniformly for all jump sizes $\ell>\beta-1$. This is only a polynomial factor worse than the $((1+o(1))e/\ell)^{\ell} n^{\ell}$ expected running time of the (1+1)~EA which for every jump size~$\ell$ uses a bit flip probability of~$\ell/n$, which is the optimal static choice.

\cite{DangL16} showed an expected running time of $O((n/c)^{\ell+1})$ for a large class of non-elitist population-based evolutionary algorithms with mutation rate $c/n$, where $c$ is supposed to be a constant.  

Jump had originally been studied to investigate the usefulness of crossover; i.e., the recombination of two or more search points into a new solution candidate. In~\cite{JansenW02} a $(\mu+1)$~genetic algorithm using crossover is shown to optimize any $\jump_{\ell,z}$ function in an expected $O(\mu n^2 (\ell-1)^3 + 4^{\ell-1}/p_c)$ number of function evaluations, where $p_c<1/(c(\ell-1) n)$ denotes the (artificially small) probability of doing a crossover. In~\cite{Dang2017} it was shown that for more ``natural'' parameter settings, most notably a non-vanishing probability of crossover, a standard $(\mu+1)$~genetic algorithm which uses crossover followed by mutation has an expected $O(n^{\ell} \log n)$ optimization time on any $\jump_{\ell,z}$ function, which is an $O(n/\log n)$ factor gain over the above-mentioned bound for the standard (1+1)~EA. In~\cite{KotzingGECCO16} it is shown that significant performance gains can be achieved by the usage of diversity mechanisms. We refer the interested reader to Chapter~\ref{chap:diversity} [link to the chapter on diversity will be added] of this book for a more detailed description of these mechanisms and running time statements, cf. in particular Sections~\ref{ds:sec:jump-emerging} and~\ref{ds:sec:jump-diversity-mechanisms} there.

\subsubsection{The Unrestricted Black-Box Complexity of Jump Functions} 
From the definition in~\eqref{def:BBCjump}, we easily see that for every $n \in \N$ and for all $\ell \in [0..n/2]$ there exists a function $f: [0..n] \to [0..n]$ such that $\jump_{\ell,z}(x) = f(\OM_z(x))$ for all $z,x \in \{0,1\}^n$. By Theorem~\ref{thm:BBClowerBoundForMapped}, we therefore obtain that for every class $\A$ of algorithms and for all $\ell$, the $\A$-black-box complexity of $\textsc{Jump}_{\ell,n}:=\{\jump_{\ell,z} \mid z \in \{0,1\}^n \}$ is at least as large as that of $\onemax_n$. Quite surprisingly, it turns out that this bound can be met for a broad range of jump sizes $\ell$. Building on the work~\cite{DoerrDK15jump} on the unbiased black-box complexity of jump (cf. Section~\ref{sec:BBCunbiasedjump} for a detailed description of the results proven in~\cite{DoerrDK15}), in~\cite{BuzdalovDK16} the following bounds are obtained. 

\begin{theorem}[\cite{BuzdalovDK16}]
\label{thm:BBCunrestrictedJump}
For $\ell<n/2-\sqrt{n} \log_2$, the unrestricted black-box complexity of $\textsc{Jump}_{\ell,n}$ is at most $(1+o(1)) 2n/\log_2 n$, while for $n/2-\sqrt{n} \log_2 \le \ell < \lfloor n/2 \rfloor -\omega(1)$ it is at most $(1+o(1)) n/\log_2(n-2\ell)$ (where in this latter bound $\omega(1)$ and $o(1)$ refer to $n-2\ell \to \infty$).

For the extreme case of $\ell=\lfloor n/2 \rfloor-1$, the unrestricted black-box complexity of $\textsc{Jump}_{\ell,n}$ is $n+\Theta(\sqrt{n})$.

For all $\ell$ and every odd $n$, the unrestricted black-box complexity of $\textsc{Jump}_{\ell,n}$ is at least 
$\lfloor \log_{\frac{n-2\ell+1}{2}} \big(2^{n-2}(n-2\ell-1) + 1 \big)\rfloor - \frac{2}{n-2\ell-1}$. For even $n$, it is at least 
$\lfloor \log_{\frac{n-2\ell+2}{2}} \big(1+2^{n-1}\frac{(n-2\ell)^2}{n-2\ell-1} \big)\rfloor - \frac{2}{n-2\ell}$. 
\end{theorem}

The proofs of the results in Theorem~\ref{thm:BBCunrestrictedJump} are build to a large extend on the techniques used in~\cite{DoerrDK15jump}, which we shall discuss in Section~\ref{sec:BBCunbiasedjump}. In addition to these techniques, \cite{BuzdalovDK16} introduces a matrix lower bound method, which allows to prove stronger lower bounds than the simple information-theoretic result presented in Theorem~\ref{thm:BBCinfotheo} by taking into account that the ``typical'' information obtained through a function evaluation can be much smaller than what the whole range $\{f(s) \mid s \in S\}$ of possible $f$-values suggests.

Note that even for the case of ``small'' $\ell<n/2-\sqrt{n} \log_2$, the region around the optimum in which the function values are zero is actually quite large. This plateau contains $2^{(1-o(1))n}$ points and has a diameter that is linear in $n$. 

For the case of the extreme jump functions also note that, apart from the optimum, only the points $x$ with $\OM_z(x)=\lfloor n/2 \rfloor$ and $\OM_z(x)=\lceil n/2 \rceil$ have a non-zero function value. It is thus quite surprising that these functions can nevertheless be optimized so efficiently. We shall see in Section~\ref{sec:BBCunbiasedjump} how this is possible.

One may wonder why in the definition of $\textsc{Jump}_{\ell,n}$ we have fixed the jump size $\ell$, as this way it is ``known'' to the algorithm. It has been argued in~\cite{DoerrKW11} that the algorithms can learn $\ell$ efficiently, if this is needed; in some cases, including those of small $\ell$-values, knowing $\ell$ may not be needed to achieve the above-mentioned optimization times. Whether or not knowledge of $\ell$ is needed can be decided adaptively. 

\subsubsection{Alternative Definitions of Jump}\label{sec:BBCjumpalternative}
Following up on results for the so-called unbiased black-box complexity of jump functions~\cite{DoerrDK15jump}, cf. Section~\ref{sec:BBCunbiasedjump}, Jansen~\cite{Jansen15} proposed an alternative generalization of the classical jump function regarded in~\cite{DrosteJW02}. To discuss this extension, we recall that the jump function analyzed by~\cite[Definition 24]{DrosteJW02} is $(1,\ldots,1)$-version of the maps $f_{\ell,z}$ that assign to every length-$n$ bit string $x$ the function value
$$
f_{\ell,z}(x) := 
\begin{cases}
\ell+ n,				&\mbox{if }\OM_z(x) = n;\\
\ell+\OM_z(x),		&\mbox{if } \OM_z(x) \leq n-\ell;\\
n-\OM_z(x),				&\mbox{otherwise.}
\end{cases}
$$
We first motivate the extension considered in Definition~\eqref{def:BBCjump}. To this end, we first note that in the unrestricted black-box complexity model, $f_{\ell,z}$ can be very efficiently optimized by searching for the bit-wise complement $\bar{z}$ of $z$ and then inverting this string to the optimal search point $z$. Note also that in this definition the region around the optimum $z$ provides more information than the functions $\jump_{\ell,z}$ defined via~\eqref{def:BBCjump}. When we are interested in bounding the expected optimization time of classical black-box heuristics, this additional information does most often not pose any problems. But for our sought black-box complexity studies this information can make a crucial difference. \cite{LehreW10} therefore designed a different set of jump functions consisting of maps $g_{\ell,z}$ that assign to each $x$ the function value
$$
g_{\ell,z}(x) := 
\begin{cases}
n,				&\mbox{if }\OM_z(x) = n;\\
\OM_z(x),		&\mbox{if } \ell<\OM_z(x) \leq n-\ell;\\
0,				&\mbox{otherwise.}
\end{cases}
$$
Definition~\eqref{def:BBCjump} is mostly similar to this definition, with the only difference being the function values for bit strings $x$ with $\OM_z(x) = n-\ell$. Note that in~\eqref{def:BBCjump} the sizes of the ``blanked out areas'' around the optimum and its complement are equal, while for $g_{\ell,z}$ the area around the complement is larger than that around the optimum.  

As mentioned, Jansen~\cite{Jansen15} introduced yet another version of the jump function. His motivation is that the spirit of the jump function is to ``[locate] an unknown target string that is hidden in some distance to points a search heuristic can find easily''. Jansen's definition also has black-box complexity analysis in mind. For a given $z\in\{0,1\}^n$ and a search point $x^*$ with $\OM_z(x^*)>n-\ell$ his jump function $h_{\ell,z,x^*}$ assigns to every bit string $x$ the function value
$$
h_{\ell,z,x^*}(x) := 
\begin{cases}
n+1, 			&\mbox{if }x=x^*;\\
n-\OM_z(x),				&\mbox{if } n-\ell<\OM_z(x) \leq n \text{ and } x \neq x^*;\\
\ell+\OM_z(x),		&\mbox{otherwise.}
\end{cases}
$$
Since these functions do not reveal information about the optimum other than its $\ell$-neighborhood, the unrestricted black-box complexity of the class $\{h_{\ell,z,x^*} \mid z \in \{0,1\}^n, \OM_z(x^*)>n-\ell\}$ is $\left( \sum_{i=0}^{\ell-1}\binom{n}{i} +1\right)/2$~\cite[Theorem 4]{Jansen15}. This bound also holds if $z$ is fixed to be the all-ones string, i.e., if we regard the unrestricted black-box complexity of the class $\{h_{\ell,(1,\ldots,1),x^*} \mid \OM_z(x^*)>n-\ell\}$. For constant $\ell$, the black-box complexity of this class of jump functions is thus $\Theta(n^{\ell-1})$, very different from the results on the unrestricted black-box complexity of the $\jump_{\ell,z}$ functions regarded above. In difference to the latter, the expected optimization times stated for crossover-based algorithms in Section~\ref{sec:BBCjumpalgos} above do not necessarily apply to the functions $h_{\ell,z,x^*}$, as for these functions the optimum $x^*$ is not located in the middle of the $\ell$-neighborhood of $z$. 

\subsection{Combinatorial Problems}\label{sec:BBCunrestrictedcombinatorial}
The above-described results are mostly concerned with benchmark functions that have been introduced to study some particular features of typical black-box optimization techniques; e.g., their hill-climbing capabilities (via the \onemax function), or their ability to jump a plateau of a certain size (this is the focus of the jump functions). Running time analysis, of course, also studies more ``natural'' combinatorial problems, like satisfiability problems, partition problems, scheduling, and graph-based problems like routing, vertex cover, max cut etc., cf.~\cite{NeumannW10} for a survey of running time results for combinatorial optimization problems. 

Apart from a few results for combinatorial problems derived in~\cite{DrosteJW06}\footnote{More precisely, in~\cite{DrosteJW06} the following combinatorial problems are studied: MaxClique (Section 3), Sorting (Section 4), and the single-source shortest paths problem (Sections 4 and 9).}, the first work to undertake a systematic investigation of the black-box complexities of combinatorial optimization problems is~\cite{DoerrKLW13}. In this work, the two well-known problems of finding a \emph{minimum spanning tree} (MST) in a graph as well as the \emph{single-source shortest path problem} (SSSP) are considered. The work revealed that for combinatorial optimization problems the precise formulation of the problem can make a decisive difference. Modeling such problems needs therefore be executed with care. 

We will not be able to summarize all results proven in~\cite{DoerrKLW13}, but the following summarizes the most relevant ones for the unrestricted black-box model. \cite{DoerrKLW13} also studies the MST and the SSSP problem in various restricted black-box models; more precisely in the unbiased black-box model (cf. Section~\ref{sec:BBCunbiased}), the ranking-based model (Section~\ref{sec:BBCcomparison}) and combinations thereof. We will briefly discuss results for the unbiased case in Sections~\ref{sec:BBCunbiasedMST} and~\ref{sec:BBCunbiasedSSSP}. 

\subsubsection{Minimum Spanning Trees}\label{sec:BBCunrestrictedMST}

For a given graph $G=(V,E)$ with edge weights $w:E\to \R$, the minimum spanning tree problem asks to find a connected subgraph $G'=(V,E')$ of $G$ such that the sum $\sum_{e \in E'}{w(e)}$ of the edge weights in $G'$ is minimized. This problem has a natural bit string representation. Letting $m:=|E|$, we can fix an enumeration $\nu:[m] \to \E$. This way, we can identify a length-$m$ bit string $x=(x_1,\ldots,x_m)$ with the subgraph $G_x:=(V,E_x)$, where $E_x:=\{ \nu(i) \mid i \in [m] \text{ with } x_i=1 \}$ is the set of edges $i$ for which $x_i=1$. Using this interpretation, the MST problem can be modeled as a pseudo-Boolean optimization problem $f:\{0,1\}^m \to \R$, cf.~\cite{NeumannW10} for details. This formulation is one of the two most common formulations of the MST problem in the evolutionary computation literature. The other common formulation uses a bi-objective fitness function $f:\{0,1\}^m \to \R^2$; the first component maps each subgraph to its number of connected components, while the second component measures the total weight of the edges in this subgraph. In the unrestricted black-box model, both formulations are almost equivalent.\footnote{Note that this is not the case for the \emph{ranking-based model} discussed in Section~\ref{sec:BBCcomparison}, since here it can make a decisive difference whether two rankings for the two components of the bi-objective function are reported or if this information is further condensed to one single ranking.}

\begin{theorem}[Theorems 10 and 12 in~\cite{DoerrKLW13}]
\label{thm:BBCunrestrictedMST}
For the bi-objective as well as the single-objective formulation of the MST problem in an arbitrary connected graph of $n$ nodes and $m$ edges, the unrestricted black-box complexity is strictly larger than $n-2$ and at most $2m+1$. 

These bounds also apply if instead of absolute function values $f(x)$ only their rankings are revealed; in other words, also the ranking-based black-box complexity (which will be introduced in Section~\ref{sec:BBCcomparison}) of the MST problem is at most $2m+1$. 
\end{theorem} 
The upper bound is shown by first learning the order of the edge weights and then testing in increasing order of the edge weights whether the inclusion of the corresponding edge forms a cycle or not. This way, the algorithm imitates the well-known MST algorithm by Kruskal. 

The lower bound of Theorem~\ref{thm:BBCunrestrictedMST} is obtained by applying Yao's minimax principle with a probability distribution on the problem instances that samples uniformly at random a spanning tree, gives weight $1$ to all of its edges, and weight $2$ to all other edges. By Cayley's formula, the number of spanning trees on $n$ vertices is $n^{n-2}$. In intuitive terms, a black-box algorithm solving the MST problem therefore needs to learn $(n-2)\log_2 n$ bits. Since each query reveals a number between $2k-n+1$ and $2k$ ($k$ being the number of edges included in the corresponding graph), it provides at most $\log_2 n$ bits of information. Hence, in the worst case, we get a running time of at least $n-2$ iterations. To turn this intuition into a formal proof, a drift theorem is used to show that in each iteration the number of consistent possible trees decreases by factor of at most $1/n$.

\subsubsection{Single-Source Shortest Paths}\label{sec:BBCunrestrictedSSSP}
For SSSP, which asks to connect all vertices of an edge-weighted graph to a distinguished source node through a path of smallest total weight, several formulations co-exist. The first one, which was also regarded in~\cite{DrosteJW06}, uses the following multi-criteria objective function. An algorithm can query arbitrary trees on $[n]$ and the objective value of any such tree is an $n-1$ tuple of the distances of the $n-1$ non-source vertices to the source $s=1$ (if an edge is traversed which does not exist in the input graph, the entry of the tuple is $\infty$). 

\begin{theorem}[Theorems~16 and~17 in~\cite{DoerrKLW13}]
\label{thm:BBCunrestrictedSSSPmulti}
For arbitrary connected graphs with $n$ vertices and $m$ edges, the unrestricted black-box complexity of the multi-objective formulation of the SSSP problem is $n-1$. For complete graphs, it is at least $n/4$ and at most $\lfloor (n+1)/2 \rfloor +1$.
\end{theorem}
Theorem~\ref{thm:BBCunrestrictedSSSPmulti} improves the previous $n/2$ lower and the $2n-3$ upper bound given in~\cite[Theorem~9]{DrosteJW06}. For the general case, the proof of the upper bound imitates Dijkstra's algorithm by first connecting all vertices to the source, then all but one vertices to the vertex of lowest distance to the source, then all but the two of lowest distance to the vertex of second lowest distance and so on, fixing one vertex with each query. The lower bound is an application of Yao's minimax principle to a bound on deterministic algorithms, which is obtained trough an additive drift analysis.
 
For complete graphs, it is essentially shown that the problem instance can be learned with $\lfloor(n+1)/2\rfloor$ queries, while the lower bound is again a consequence of Yao's minimax principle.

The bound in Theorem~\ref{thm:BBCunrestrictedSSSPmulti} is not very satisfactory as already the size of the input is $\Omega(m)$. The discrepancy is due to the large amount of information that the objective function reveals about the problem instance. To avoid such low black-box complexities, and to shed a better light on the complexity of the SSSP problem, \cite{DoerrKLW13} considers also an alternative model for the SSSP problem, in which a representation of a candidate solution is a vector $(\rho(2),\ldots,\rho(n)) \in [n]^{n-1}$. Such a vector is interpreted such that the predecessor of node $i$ is $\rho(i)$ (the indices run from $2$ to $n$ to match the index with the label of the nodes---node 1 is the source node to which a shortest path is sought). With this interpretation, the search space becomes the set $S_{[2..n]}$ of permutations of $[2..n]$; i.e., $S_{[2..n]}$ is the set all one-to-one maps $\sigma:[2..n] \to [2..n]$. For a given graph $G$, the single-criterion objective function $f_{G}$ is defined by assigning to each candidate solution $(\rho(2),\ldots,\rho(n))$ the function value $\sum_{i=2}^n{d_i}$, where $d_i$ is the distance of the $i$-th node to the source node. If an edge---including loops---is traversed which does not exist in the input graph, $d_i$ is set to $n$ times the maximum weight $w_{\max}$ of any edge in the graph). 

\begin{theorem}[Theorem~18 in~\cite{DoerrKLW13}]
\label{thm:BBCunrestrictedSSSPsingle}
The unrestricted black-box complexity of the SSSP problem with the single-criterion objective function is at most $\sum_{i=1}^{n-1}{i}=n (n-1)/2$.
\end{theorem} 
As in the multi-objective case, the bound of Theorem~\ref{thm:BBCunrestrictedSSSPsingle} is obtained by imitating Dijkstra's algorithm. In the single-objective setting, adding the $i$-th node to the shortest path tree comes at a cost of at most $n-i$ function evaluations.

\section{Memory-Restricted Black-Box Complexity}\label{sec:BBCmemory}

\begin{algorithm2e}[t]
 \textbf{Initialization:} \\
 \Indp
 $X \assign \emptyset$\;
 Choose a probability distribution $D^{(0)}$ over $S^{\mu}$, sample from it $x^{(1)}, \ldots, x^{(\mu)} \in S$, and query $f(x^{(1)}), \ldots, f(x^{(\mu)})$\;
 $X \assign \big\{\big(x^{(1)},f(x^{(1)})\big), \ldots, \big(x^{(\mu)},f(x^{(\mu)})\big)\big\}$\;
 \Indm
 \textbf{Optimization:}	
 \For{$t=1,2,3,\ldots$}{
 		Depending only on the multiset $X$ choose a probability distribution $D^{(t)}$ over $S^{\lambda}$, sample from it $y^{(1)}, \ldots, y^{(\lambda)} \in S$, and query $f(y^{(1)}), \ldots, f(y^{(\mu)})$\;
		Set $X \assign X \cup \big\{ \big(y^{(1)},f(y^{(1)})\big), \ldots, \big(y^{(\mu)},f(y^{(\lambda)})\big) \big\}$\;
  \lFor{$i=1,\ldots, \lambda$}{
  	Select $(x,f(x)) \in X$ and update $X \assign X \setminus \{(x,f(x))\}$}
	 }
 \caption{The $(\mu+\lambda)$ memory-restricted black-box algorithm for optimizing an unknown function $f:S \rightarrow \R$}
\label{alg:BBCmemory}
\end{algorithm2e}

As mentioned in the previous section, already in the first work on black-box complexity~\cite{DrosteJW06} it was noted that the unrestricted model can be too generous in the sense that it includes black-box algorithms that are highly problem-tailored and whose expected running time is much smaller than that of typical black-box algorithms. One potential reason for such a discrepancy is the fact that unrestricted algorithms are allowed to store and to access the full search history, while typical heuristics only store a subset (``population'' in evolutionary computation) of these previously evaluated samples and their corresponding objective values. Droste, Jansen, and Wegener therefore suggested to adjust the unrestricted black-box model to reflect this behavior. In their \emph{memory-restricted model of size $\mu$},\footnote{In the original work~\cite{DrosteJW06} a memory-restricted algorithm of size $\mu$ was called black-box algorithm with \emph{size bound}~$\mu$.} the algorithms can store up to $\mu$ pairs $(x,f(x))$ of previous samples. Based only on this information, they decide on the probability distribution $D$ from which the next solution candidate is sampled. Note that this also implies that the algorithm does not have any iteration counter or any other information about the time elapsed so far. Regardless of how many samples have been evaluated already, the sampling distribution $D$ depends only on the $\mu$ pairs $\big(x^{(1)},f(x^{(1)})\big), \ldots, \big(x^{(\mu)},f(x^{(\mu)})\big)$ stored in the memory. 

We extend this model to a \emph{$(\mu+\lambda)$ memory-restricted black-box model,} in which the algorithms have to query $\lambda$ solution candidates in every round, cf. Algorithm~\ref{alg:BBCmemory}. Following Definition~\ref{def:BBCcomplexity}, the \emph{$(\mu+\lambda)$ memory-restricted black-box complexity} of a function class $\F$ is the black-box complexity with respect to the class $\A$ of all $(\mu+\lambda)$ memory-restricted black-box algorithms.

The memory-restricted model of size $\mu$ corresponds to the $(\mu+1)$ memory-restricted one, in which only one search point needs to be evaluated per round. Since this variant with $\lambda=1$ allows the highest degree of adaptation, it is not difficult to verify that for all $\mu \in \N$ and for all $\lambda > 1$ the $(\mu+\lambda)$ memory-restricted black-box complexity of a problem $\F$ is at least as large as its $(\mu+1)$ black-box complexity. The effects of larger $\lambda$ have been studied in a \emph{parallel black-box complexity model}, which we will discuss in Section~\ref{sec:BBCparallel}.

\begin{figure}
\begin{framed}
\begin{center}
\includegraphics[scale=1.5]{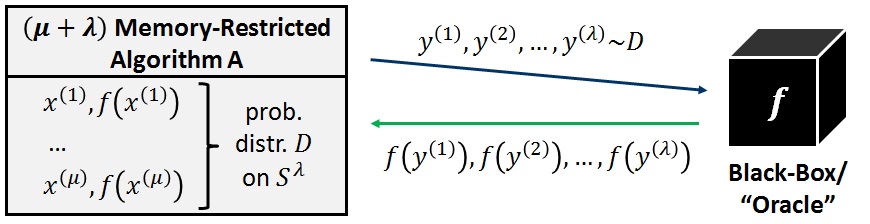}
\end{center}
\caption{In the $(\mu+\lambda)$ memory-restricted black-box model, the algorithms can store up to $\mu$ previously evaluated search points and their absolute function values. In each iteration, it queries the function values of $\lambda$ new solution candidates. It then has to decide which of the $\mu+\lambda$ search points to keep in the memory for the next iteration.}
\end{framed}
\label{fig:BBCmemory}
\end{figure}

While it seems intuitive that larger memory sizes yield to smaller optimization times, this in not necessarily true for all functions. Indeed, the following discussion shows that memory is not needed for the efficient optimization of \onemax. 

\subsection{OneMax}\label{sec:BBCmemoryOM}
Droste, Jansen, and Wegener conjectured in~\cite[Section~6]{DrosteJW06} that the $(1+1)$ memory-restricted black-box complexity of \onemax is $O(n \log n)$, in the belief that Randomized Local Search and the (1+1)~EA are asymptotically optimal representatives of this class. This conjecture was refuted in~\cite{DoerrW12IPL}, where a linear upper bound was presented. This bound was further reduced to $O(n/\log n)$ in~\cite{DoerrW14memory}, showing that even the most restrictive version of the memory-restricted black-box model does not increase the asymptotic complexity of \onemax. By the lower bound presented in Theorem~\ref{thm:BBCunrestrictedOM}, the $O(n/\log n)$ bound is asymptotically best possible.

\begin{theorem}[Theorem~1 in~\cite{DoerrW14memory}]
\label{thm:BBCmemoryOM}
The $(1+1)$ memory-restricted black-box complexity of \onemax is $\Theta(n/\log n)$.
\end{theorem}
The proof of Theorem~\ref{thm:BBCmemoryOM} makes use of the $O(n/\log n)$ unrestricted strategy by Erd{\H{o}}s and R{\'e}nyi. To respect the memory-restriction, the algorithm achieving the $O(n/\log n)$ expected optimization time works in rounds. In every round, a substring of size $s:=\sqrt{n}$ of the target string $z$ is identified, using $O(s/\log s)$ queries. The algorithm alternates between querying a string to obtain new information and queries which are used only to store the function value of the last query in the current memory. This works only if sufficiently many bits are available in the guessing string to store this information. It is shown that $O(n/\log n)$ bits suffice. These last remaining $O(n/\log n)$ bits of $z$ are then identified with a constant number of guesses per position, giving an overall expected optimization time of $O(n/s)O(s/\log s) + O(n/\log n)=O(n/\log n)$.

\subsection{Difference to the Unrestricted Model}\label{sec:BBCmemorygap}

In light of Theorem~\ref{thm:BBCmemoryOM} it is interesting to find examples for which the $(\mu+1)$ memory-restricted black-box complexity is strictly (and potentially much) larger than its $((\mu+1)+1)$ memory-restricted one. This question was addressed in~\cite{Storch16}. 

In a first step, it is shown that having a memory of one can make a decisive difference over not being able to store any information at all. In fact, it is easily seen that without any memory, every function class~$\F$ that for every $s \in S$ contains a function $f_s$ such that $s$ is the only optimal solution of $f_s$, then without any memory, the best one can do is random sampling, resulting in an expected optimization time of $|S|$. Assume now that there is a (fixed) search point $h\in S$ where a hint is given, in the sense that for all $s\in S$ the objective value $f_s(h)$ uniquely determines where the optimum $s$ is located. Then clearly the $(1+1)$ memory-restricted algorithm which first queries $h$ and then based on $(h,f_s(h))$ queries $s$ solves any problem instance $f_s$ in at most 2 queries. 

This idea can be generalized to a function class of functions with two hints hidden in two different distinguished search points $h_1$ and $h_2$. Only the combination of $(h_1,f_s(h_1))$ with $(h_2,f_s(h_2))$ defines where to locate the optimum $s$. This way, the $(2+1)$ memory-restricted black-box complexity of this class $\F(h_1,h_2)$ is at most three, while its $(1+1)$ memory-restricted one is at least $(S+1)/2$. For, say, $S=\{0,1\}^n$ we thus see that the discrepancies between the $(0+1)$ memory-restricted black-box complexity of a problem $\F$ and its $(1+1)$ memory-restricted one can be exponential, and so can be the difference between the $(1+1)$ memory-restricted black-box complexity and the $(2+1)$ memory-restricted one. We are not aware of any generalization of this result to arbitrary values of $\mu$. 

\begin{theorem}[\cite{Storch16}]
\label{thm:BBCmemorygap}
There are classes of functions $\F(h) \subset \{ f \mid f:\{0,1\}^n \to \R\}$ and $\F(h_1,h_2) \subset \{ f \mid f:\{0,1\}^n \to \R\}$ such that 
\begin{itemize}
	\item the $(0+1)$ memory-restricted black-box complexity of $\F(h)$ is exponential in $n$, while its $(1+1)$ memory-restricted one is at most two, and
	\item the $(1+1)$ memory-restricted black-box complexity of $\F(h_1,h_2)$ is exponential in $n$, while its $(2+1)$ memory-restricted one is at most three.
\end{itemize}
\end{theorem}
Storch~\cite{Storch16} also presents a class of functions that is efficiently optimized by a standard (2+1) genetic algorithm, which is a $(2+1)$ memory-restricted black-box algorithm, in $O(n^2)$ queries, on average, while its $(1+1)$ memory-restricted black-box complexity is exponential in $n$. This function class is build around so-called royal road functions; the main idea being that the genetic algorithm is guided towards the two ``hints'', between which the unique global optimum is located.

\section{Comparison- and Ranking-Based Black-Box Complexity}\label{sec:BBCcomparison}

Many standard black-box heuristics do not take advantage of knowing \emph{exact} objective values. Instead, they use these function values only to rank the search points. This ranking determines the next steps, so that the absolute function values are not needed. Such algorithms are often referred to as \emph{comparison-based} or \emph{ranking-based}. To understand their efficiency \emph{comparison-based} and \emph{ranking-based} black-box complexity models have been suggested in~\cite{TeytaudG06,FournierT11,DoerrW14ranking}.

\subsection{The Ranking-Based Black-Box Model}\label{sec:BBCranking}
In the ranking-based black-box model, the algorithms receive a ranking of the search points currently stored in the memory of the population. This ranking is defined by the objective values of these points.

\begin{definition}
Let $S$ be a finite set, let $f:S \rightarrow \R$ be a function, and let $\mathcal{C}$ be a subset of $S$. 
The \emph{ranking} $\rho$ of $\mathcal{C}$ with respect to $f$ assigns to each element $c \in \mathcal{C}$ the number of elements in $\mathcal{C}$ with a smaller $f$-value plus $1$, formally, $\rho(c):=1+\left|\{ c' \in \mathcal{C} \, |\, f(c')<f(c) \}\right|$. 
\end{definition}
Note that two elements with the same $f$-value are assigned the same ranking.

In the ranking-based black-box model without memory restriction, an algorithm receives thus with every query a ranking of \emph{all} previously evaluated solution candidates, while in the memory-restricted case, naturally, only the ranking of those search points currently stored in the memory is revealed. To be more precise, the ranking-based black-box model without memory restriction subsumes all algorithms that can be described via the scheme of Algorithm~\ref{alg:BBCranking}. Figure~\ref{fig:BBCranking} illustrates these ranking-based black-box algorithms.

\begin{algorithm2e}%
 \textbf{Initialization:}
	\Indp
		Sample $x^{(0)}$ according to some probability distribution $D^{(0)}$ over $S$\;
		$X \leftarrow \{x^{(0)}\}$\;
	\Indm
 \textbf{Optimization:}	
\For{$t=1,2,3,\ldots$}{
  \label{line:BBCrankingmut} Depending on $\{x^{(0)}, \ldots, x^{(t-1)}\}$ and its ranking $\rho(X,f)$ with respect to $f$, choose a probability distribution $D^{(t)}$ on $S$ and sample from it $x^{(t)}$\;
	$X \leftarrow X \cup \{x^{(t)}\}$\;
  Query the ranking $\rho(X,f)$ of $X$ induced by $f$\;
 }\caption{Blueprint of a ranking-based black-box algorithm without memory restriction.}
\label{alg:BBCranking}
\end{algorithm2e}

\begin{figure}
\begin{framed}
\begin{center}
\includegraphics[scale=1.5]{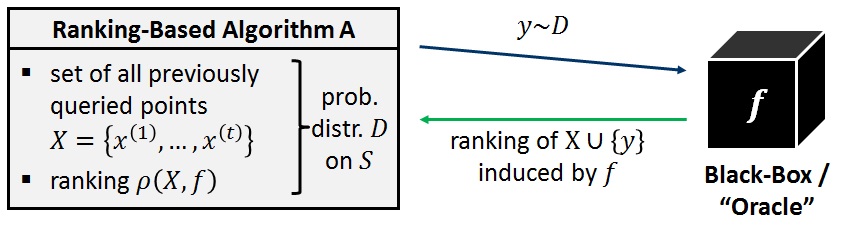}
\end{center}
\caption{A ranking-based black-box algorithm without memory restriction can store all previously evaluated search points. Instead of knowing their function values, it only has access to the ranking of the search points induced by the objective function $f$. Based on this, it decides upon a distribution $D$ from which the next search point is sampled.}
\end{framed}
\label{fig:BBCranking}
\end{figure}

Likewise, the $(\mu+\lambda)$ memory-restricted ranking-based model contains those $(\mu+\lambda)$ memory-restricted algorithms that follow the blueprint in Algorithm~\ref{alg:BBCrankingmemory}; Figure~\ref{fig:BBCrankingmemory} illustrates this pseudo-code. 

\begin{algorithm2e}
 \textbf{Initialization:} \\
 \Indp
 $X \assign \emptyset$\;
 Choose a probability distribution $D^{(0)}$ over $S^{\mu}$ and sample from it $X=\{ x^{(1)}, \ldots, x^{(\mu)} \} \subseteq S$\;
 Query the ranking $\rho(X,f)$ of $X$ induced by $f$\;
 \Indm
 \textbf{Optimization:}	
 \For{$t=1,2,3,\ldots$}{
 		Depending only on the multiset $X$ and the ranking $\rho(X,f)$ of $X$ induced by $f$
	\label{line:BBCranking-mut}	choose a probability distribution $D^{(t)}$ on $S^{\lambda}$ and 
		sample from it $y^{(1)},\ldots,y^{(\lambda)}$\;
		Set $X \assign X \cup \{y^{(1)},\ldots,y^{(\lambda)}\}$ and query the ranking $\rho(X,f)$ of $X$ induced by $f$\;
  \lFor{$i=1,\ldots, \lambda$}{
  	\label{line:BBCranking-selection} Based on $X$ and $\rho(X,f)$ select a (multi-)subset $Y$ of $X$ of size $\mu$ and update $X \assign Y$}
	 }
 \caption{The $(\mu+\lambda)$ memory-restricted ranking-based black-box algorithm for maximizing an unknown function $f:\{0,1\}^n \rightarrow \R$}
\label{alg:BBCrankingmemory}
\end{algorithm2e}

\begin{figure}
\begin{framed}
\begin{center}
\includegraphics[scale=1.5]{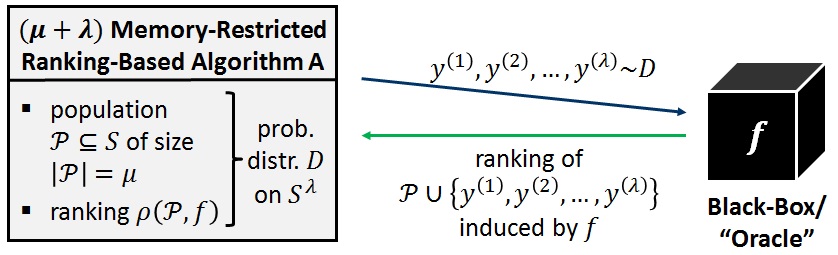}
\end{center}
\caption{A $(\mu+\lambda)$ memory-restricted ranking-based black-box algorithm can store up to $\mu$ previously evaluated search points and the ranking of this population induced by the objective function $f$. Using only this information, $\lambda$ new solution candidates are sampled in each iteration and the ranking of the $(\mu+\lambda)$ points is revealed. Based on this ranking, the algorithm needs to select which $\mu$ points to keep in the memory.}
\end{framed}
\label{fig:BBCrankingmemory}
\end{figure}

These ranking-based black-box models capture many common search heuristics, such as $(\mu+\lambda)$ evolutionary algorithms, some ant colony optimization algorithms (for example simple versions of Max-Min Ant Systems as analyzed in \cite{NeumannW09,KotzingNSW11}), and Randomized Local Search. They do not include algorithms with fitness-dependent parameter choices, such as fitness-proportional mutation rates or fitness-dependent selection schemes.

Surprisingly, the unrestricted and the non-memory-restricted ranking-based black-box complexity of \onemax coincide in asymptotic terms; the leading constants may be different. 
\begin{theorem}[Theorem~6 in~\cite{DoerrW14ranking}]
\label{thm:BBCrankingOM}
The ranking-based black-box complexity of \onemax without memory-restriction is $\Theta(n/\log n)$.  
\end{theorem}
The upper bound for \onemax is obtained by showing that, for a sufficiently large sample base, a median search point $x$ (i.e., a search point for which half of the search points have a ranking that is at most as large as that of $x$ and the other half of the the search points have rankings that are at least as large as that of $x$) is very likely to have $n/2$ correct bits. It is furthermore shown that with $O(n/\log n)$ random queries each of the function values in the interval $[n/2-\kappa \sqrt{n}, n/2+\kappa \sqrt{n}]$ appears at least once. This information is used to translate the ranking of the random queries into absolute function values, for those solution candidates $y$ for which $\OM_z(y)$ lies in the interval $[n/2-\kappa \sqrt{n}, n/2+\kappa \sqrt{n}]$. The proof is then concluded by showing that it suffices to regard only these samples in order to identify the target string $z$ of the problem instance~$\OM_z$. 

For \binval, in contrast, it makes a substantial difference whether absolute or relative objective values are available.
\begin{theorem}[Theorem~17 in~\cite{DoerrW14ranking}]
\label{thm:BBCrankingBV}
The ranking-based black-box complexity of $\binval_n$ and $\binval_n^*$ is strictly larger than $n-1$, even when the memory is not bounded.  
\end{theorem} 
This lower bound of $n-1$ is almost tight. In fact, an $n+1$ ranking-based algorithm is easily obtained by starting with a random initial search point and then, from left to right, flipping in each iteration exactly one bit. The ranking uniquely determines the permutation $\sigma$ and the string $z$ of the problem instance $\BV_{z,\sigma}$.

Theorem~\ref{thm:BBCrankingBV} is shown with Yao's minimax principle applied to the uniform distribution over the problem instances. The crucial observation is that when optimizing $\BV_{z,\sigma}$ with a ranking-based algorithm, then from $t$ samples we can learn at most $t-1$ bits of the hidden bit string $z$, and not $\Theta(t \log t)$ bits as one might guess from the fact that there are $t!$ permutations of the set~$[t]$.  

This last intuition, however, gives a very general lower bound. Intuitively, if $\F$ is such that every $z \in \{0,1\}^n$ is the unique optimum for a function $f_z\in \F$, and we only learn the ranking of the search points evaluated so far, then for the $t$-th query, we learn at most $\log_2(t!)=\Theta(t \log t)$ bits of information. Since we need to learn $n$ bits in total, the ranking-based black-box complexity of $\F$ is of order at least $n/\log n$. 

\begin{theorem}[Theorem~21 in~\cite{DoerrW14ranking}]
\label{thm:BBCrankinggeneral}
Let $\F$ be a class of functions such that each $f \in \F$ has a unique global optimum and such that for all $z\in \{0,1\}^n$ there exists a function $f_z \in \F$ with $\{z\}=\arg \max f_z$. Then the unrestricted ranking-based black-box complexity of $\F$ is $\Omega(n/\log n)$.
\end{theorem}

Results for the ranking-based black-box complexity of the two combinatorial problems MST and SSSP have been derived in~\cite{DoerrKLW13}. Some of these bounds were mentioned in Section~\ref{sec:BBCunrestrictedcombinatorial}.

\subsection{The Comparison-Based Black-Box Model}\label{sec:BBCcomparison2}
In the ranking-based model, the algorithms receive for every query quite a lot of information, namely the full ranking of the current population and its offspring. One may argue that some evolutionary algorithms use even less information. Instead of regarding the full ranking, they base their decisions on a few selected points only. This idea is captured in the \emph{comparison-based black-box model}. In contrast to the ranking-based model, here only the ranking of the queried points is revealed. In this model it can therefore make sense to query a search point more than once; to compare it with a different offspring, for example. Figure~\ref{fig:BBCcomparison} illustrates the $(\mu+\lambda)$ memory-restricted comparison-based black-box model. A comparison-based model without memory-restriction is obtained by setting $\mu=\infty$. 

We do not further detail this model, as it has received only marginal attention so far in the black-box complexity literature. We note, however, that Teytaud and co-authors~\cite{TeytaudG06,FournierT11} have presented some very general lower bounds for the convergence rate of
comparison-based and ranking-based evolutionary strategies in continuous domains. From these works results for the comparison-based black-box complexity of problems defined over discrete domains can be obtained. These bounds, however, seem to coincide with the information-theoretic ones that can be obtained through Theorem~\ref{thm:BBCinfotheo}. 

\begin{figure}
\begin{framed}
\begin{center}
\includegraphics[scale=1.5]{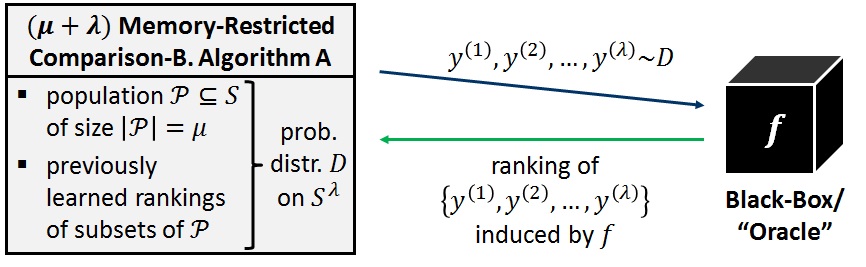}
\end{center}
\caption{A $(\mu+\lambda)$ memory-restricted comparison-based black-box algorithm can store up to $\mu$ previously evaluated search points and the comparison of these points that have been learned through previous queries. In the next iteration, $\lambda$ solution candidates are queried, possibly containing some of the current population. Only the ranking of the $\mu$ queried points is revealed. Based on this ranking and the previous information about the relative fitness values, the algorithm needs to select which $\mu$ points to keep in the memory.}
\end{framed}
\label{fig:BBCcomparison}
\end{figure}

\section{Unbiased Black-Box Complexity}\label{sec:BBCunbiased}

As previously commented, the quest to develop a meaningful complexity theory for evolutionary algorithms and other black-box optimization heuristics seemed to have come to an early end after 2006; the only work who picked up on this topic being the work of Anil and Wiegand on the unrestricted black-box complexity of \onemax~\cite{AnilW09} (cf. Section~\ref{sec:BBCunrestrictedOM}). In 2010, the situation has changed drastically. Black-box complexity was revived by Lehre and Witt in~\cite{LehreW10} (journal version appeared as~\cite{LehreW12}). To overcome the drawbacks of the previous unrestricted black-box model, they restrict the class of black-box optimization algorithms in a natural way that still covers a large class of classically used algorithms. 

In their \emph{unbiased black-box complexity model}, Lehre and Witt regard pseudo-Boolean optimization problems $\F \subseteq \{f:\{0,1\}^n \to \R \}$. The unbiased black-box model requires that all solution candidates must be sampled from distributions that are \emph{unbiased.} In the context of pseudo-Boolean optimization, unbiasedness means that the distribution can not discriminate between bit positions $1,2,\ldots,n$ nor between the bit entries 0 and 1; a formal definition will be given in Sections~\ref{sec:BBCunbiased-operators} and~\ref{sec:BBCunbiased-model}. The unbiased black-box model also admits a notion of arity. A $k$-ary unbiased black-box algorithm is one that employs only such variation operators that take up to $k$ arguments. This allows, for example, to talk about mutation-only algorithms (unary unbiased algorithms) and to study the potential benefits of recombining previously samples search points through distributions of higher arity. 

In crucial difference to the memory-restricted model, in the pure version of the unbiased black-box model, the memory is not restricted. That is, the~$k$ search points that form the input for the $k$-ary variation operator can be any random or previously evaluated solution candidate. As in the case of the comparison- and the ranking-based black-box model, combined unbiased memory-restricted models have also been studied, cf. Section~\ref{sec:BBCcombined}.

Before we formally introduce the unbiased black-box models for pseudo-Boolean optimization problems in Section~\ref{sec:BBCunbiased-model}, we define and discuss in Section~\ref{sec:BBCunbiased-operators} the concept of unbiased variation operators. Known black-box complexities in the unbiased black-box models are surveyed in Section~\ref{sec:BBCunbiased-results}. In Section~\ref{sec:BBCunbiased-other} we present extensions of the unbiased black-box models to search spaces different from $\{0,1\}^n$. 

\subsection{Unbiased Variation Operators}\label{sec:BBCunbiased-operators}

In order to formally define the unbiased black-box model, we first introduce the notion of \emph{$k$-ary unbiased variation operators.} Informally, a $k$-ary unbiased variation operator takes as input up to $k$ search points. It samples a new point $z\in\{0,1\}^n$ by applying some procedure to these previously evaluated solution candidates that treats all bit positions and the two bit values in an equal way. 

\begin{definition}[$k$-ary unbiased variation operator]
\label{def:BBCunbiased-variation}
Let $k \in \N$. 
A \emph{$k$-ary unbiased distribution} $(D(. \mid y^{(1)},\ldots,y^{(k)}))_{y^{(1)},\ldots,y^{(k)} \in \{0,1\}^n}$ is a family of probability distributions over 
$\{0,1\}^n$ such that for all inputs $y^{(1)},\ldots,y^{(k)} \in \{ 0,1\}^n$ the following two conditions hold.  
\begin{align*}
& (i)\,\, \forall x,z \in \{0,1\}^n: 
D(x \mid y^{(1)},\ldots,y^{(k)}) = D(x \oplus z \mid y^{(1)} \oplus z,\ldots,y^{(k)}\oplus z)\,,\\
& (ii)\, \forall x \in \{0,1\}^n \, \forall \sigma\in S_n: D(x \mid y^{(1)},\ldots,y^{(k)}) = D(\sigma(x) \mid \sigma(y^{(1)}), \ldots, \sigma(y^{(k)}))\,.
\end{align*} 
We refer to the first condition as \emph{$\oplus$-invariance} and we refer to the second as \emph{permutation invariance}.
A variation operator creating an offspring by sampling from a $k$-ary unbiased distribution is called a \emph{$k$-ary unbiased variation operator}. 
\end{definition}

To get some intuition for unbiased variation operators, we summarize a few characterizations and consequences of Definition~\ref{def:BBCunbiased-variation}. 

We first note that the combination of $\oplus$- and permutation invariance can be characterized as invariance under Hamming-automorphisms. A Hamming-automorphism is a one-to-one map $\alpha:\{0,1\}^n \to \{0,1\}^n$ that satisfies that for any two points $x,y\in \{0,1\}^n$ their Hamming distance $H(x,y)$ is equal to the Hamming distance $H(\alpha(x),\alpha(y))$ of their images. A formal proof for the following lemma can be found in~\cite[Lemma~3]{DoerrKLW13}.

\begin{lemma}\label{lem:BBCunbiased-characterization} 
A distribution $D(\cdot \mid x^1, \ldots, x^k)$ is unbiased if and only if, for all Hamming automorphisms $\alpha : \{0,1\}^n \rightarrow \{0,1\}^n$ and for all bit strings $y \in \{0,1\}^n$, the probability $D(y \mid x^1, \ldots, x^k)$ to sample $y$ from $(x^1, \ldots, x^k)$ equals the probability $D(\alpha(y) \mid \alpha(x^1), \ldots, \alpha(x^k))$ to sample $\alpha(y)$ from $(\alpha(x^1), \ldots, \alpha(x^k))$.
\end{lemma}

It is not difficult to see that the only $0$-ary unbiased distribution over $\{0,1\}^n$ is the uniform one.

$1$-ary operators, also called \emph{unary} operators, are sometimes referred to as \emph{mutation operators,} in particular in the field of evolutionary computation. Standard bit mutation, as used in several $(\mu+\lambda)$~EAs and $(\mu+\lambda)$~EAs, is a unary unbiased variation operator. The random bit flip operation used by RLS, which chooses at random a bit position $i \in [n]$ and replaces the entry $x_i$ by the value $1-x_i$, is also unbiased. In fact, all unary unbiased variation operators are of a very similar type, as the following definition and lemma, taken from~\cite{DoerrDY16} but known in a much more general form already in~\cite{DoerrKLW13}, shows.

\begin{definition}\label{def:BBCflip_r}
Let $n\in \N$ and $r \in [0..n]$. For every $x \in \{0,1\}^n$ let $\text{flip}_r$ be the variation operator that creates an offspring $y$ from $x$ by selecting $r$ positions $i_1, \ldots, i_r$ in $[n]$ uniformly at random (without replacement), setting $y_i:=1-x_i$ for $i \in \{i_1,\ldots,i_r\}$, and copying $y_i:=x_i$ for all other bit positions $i \in [n]\setminus \{i_1,\ldots,i_r\}$.
\end{definition}

Using this definition, unary unbiased variation operators can be characterized as follows. 
\begin{lemma}[Lemma~1 in \cite{DoerrDY16}]
\label{lem:BBCunbiased-unarycharacterization}
For every unary unbiased variation operator $(p(\cdot|x))_{x \in \{0,1\}^n}$ there exists a family of probability distributions $(r_{p,x})_{x \in \{0,1\}^n}$ on $[0..n]$ such that for all $x,y\in\{0,1\}^n$ the probability $p(y|x)$ that $(p(\cdot|x))_{x \in \{0,1\}^n}$ samples $y$ from $x$ equals the probability that the routine first samples a random number $r$ from $r_{p,x}$ and then obtains $y$ by applying $\text{flip}_r$ to $x$. On the other hand, each such family of distributions $(r_{p,x})_{x \in \{0,1\}^n}$ on $[0..n]$ induces a unary unbiased variation operator.
\end{lemma}

From this characterization, we easily see that neither the somatic contiguous hyper-mutation operator used in artificial immune systems (which selects a random position $i\in [n]$ and a random length $\ell \in [n]$ and flips the $\ell$ consecutive bits in positions $i,i+1 \mod n,\ldots,i+\ell \mod n$, cf.~\cite[Algorithm~3]{CorusHJOSZ17}), nor the asymmetric nor the position-dependent mutation operators regarded in~\cite{JansenS10} and~\cite{PKLFoga11,DoerrDK15}, respectively, are unbiased.

$2$-ary operators, also called \emph{binary} operators, are often referred to as \emph{crossover operators.} A prime example for a binary unbiased variation operator is \emph{uniform crossover}. Given two search points $x$ and $y$, the uniform crossover operator creates an offspring $z$ from $x$ and $y$ by choosing independently for each index $i \in [n]$ the entry $z_i \in \{x_i,y_i\}$ uniformly at random. In contrast, the standard \emph{one-point crossover operator}---which, given two search points $x,y \in \{0,1\}^n$ picks uniformly at random an index $k \in [n]$ and outputs from $x$ and $y$ one or both of the two offspring $x':=x_1 \ldots x_k y_{k+1} \ldots y_n$ and $y':=y_1 \ldots y_k x_{k+1} \ldots x_n$---is not permutation-invariant, and therefore not an unbiased operator.

Some works refer to the unbiased black-box model allowing variation operators of arbitrary arity as the \emph{$\ast$-ary unbiased black-box model.} Black-box complexities in the $\ast$-ary unbiased black-box model are of the same asymptotic order as those in the unrestricted model. This has been formally shown in~\cite{ABB}, for a general notion of unbiasedness that is not restricted to pseudo-Boolean optimization problems (cf. Definition~\ref{def:BBCABB}).
\begin{theorem}[Corollary~1 in~\cite{ABB}]
The $\ast$-ary unbiased black-box complexity of a problem class $\F$ is the same as its unrestricted one. 
\end{theorem}
Apart from the work~\cite{ABB}, most research on the unbiased black-box model assumes a restriction on the arity of the variation operators. We therefore concentrate in the remainder of this chapter on these restricted setting. 

\subsection{The Unbiased Black-Box Model for Pseudo-Boolean Optimization}\label{sec:BBCunbiased-model}

With Definition~\ref{def:BBCunbiased-variation} and its characterizations at hand, we can now introduce the unbiased black-box models. The $k$-ary unbiased black-box model covers all algorithms that follow the blueprint of Algorithm~\ref{alg:BBCunbiasedAlgo}. Figure~\ref{fig:BBCunbiased} illustrates these algorithms. As in previous sections, the \emph{$k$-ary unbiased black-box complexity} of some class of functions $\F$ is the complexity of $\F$ with respect to all $k$-ary unbiased black-box algorithms.

\begin{figure}[t]
\begin{framed}
\begin{center}
\includegraphics[scale=1.5]{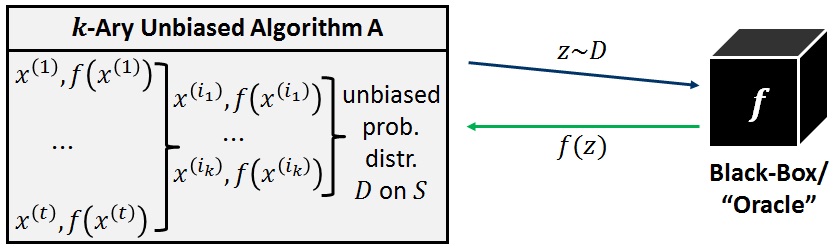}
\end{center}
\caption{In the $k$-ary unbiased black-box model, the algorithms can store the full query history. For every already evaluated search point $x$ the algorithm has access to the absolute function value $f(x) \in \R$. The distributions $D$ from which new solution candidates are sampled have to be unbiased. They can depend on up to $k$ previously evaluated solution candidates.}
\end{framed}
\label{fig:BBCunbiased}
\end{figure}

\begin{algorithm2e}
	\textbf{Initialization:} Sample $x^{(0)} \in \{0,1\}^n$ uniformly at random and query $f(x^{(0)})$\;
	 \textbf{Optimization:}
\For{$t=1,2,3,\ldots$}{
		Depending on $\big(f(x^{(0)}),\ldots, f(x^{(t-1)})\big)$ choose up to $k$ indices $i_1,\ldots,i_k \in [0..t-1]$ and a $k$-ary unbiased distribution $(D( \cdot \mid y^{(1)},\ldots,y^{(k)}))_{y^{(1)},\ldots,y^{(k)} \in \{0,1\}^n}$\;
		Sample $x^{(t)}$ according to $D(\cdot \mid x^{(i_1)},\ldots,x^{(i_k)})$ and query $f(x^{(t)})$\;
}
\caption{Scheme of a $k$-ary unbiased black-box algorithm}
\label{alg:BBCunbiasedAlgo}
\end{algorithm2e}

As Figure~\ref{fig:BBCunbiased} indicates, it is important to note, that in line~3 of Algorithm~\ref{alg:BBCunbiasedAlgo}, the $k$ selected previously evaluated search points $x^{(i_1)},\ldots,x^{(i_k)}$ do not necessarily have to be the $k$ \emph{immediately previously} queried ones. That is, the algorithm can store and is allowed to choose from \emph{all} previously sampled search points.

Note further that for all $k \leq \ell$, each $k$-ary unbiased black-box algorithm is contained in the $\ell$-ary unbiased black-box model. This is due to the fact that we do not require the indices to be pairwise distinct.

The unary unbiased black-box model captures most of the commonly used mutation-based algorithms like $(\mu + \lambda)$ and $(\mu,\lambda)$ EAs, Simulated Annealing, the Metropolis algorithm, and Randomized Local Search. The binary unbiased model subsumes many traditional genetic algorithms, such as $(\mu+\lambda)$ and $(\mu,\lambda)$~GAs using uniform crossover. As we shall discuss in Section~\ref{sec:BBCalgo}, the $(1+(\lambda,\lambda))$~GA introduced in~\cite{DoerrDE15} is also binary unbiased.

As a word of warning, we note that in~\cite{Sudholt13} and~\cite{Witt13j} lower bounds are proven for what the authors call \emph{mutation-based algorithms.} Their definitions are more restrictive than what Algorithm~\ref{alg:BBCunbiasedAlgo} proposes. The lower bounds proven in~\cite{Sudholt13,Witt13j} do therefore not (immediately) apply to the unary unbiased black-box model. A comparison of Theorem~12 from~\cite{Sudholt13} and Theorem~3.1(5) in~\cite{Witt13j} with Theorem~9 from~\cite{DoerrDY16} shows that there can be substantial differences (in this case, a multiplicative factor $\approx e$ in the lower bound for the complexity of \onemax with respect to all mutation-based and all unary unbiased black-box algorithms, respectively). One of the main differences between the different models is that in~\cite{Sudholt13,Witt13j} only algorithms using standard bit mutation are considered. This definition excludes algorithms like RLS for which the radius $r$ fed into the $\text{flip}_r$ variation operator is one deterministically and thus not sampled from a binomial distribution $\text{Bin}(n,p)$. When using the term ``mutation-based algorithms'', we should therefore always make precise which algorithmic framework we refer to. Here in this chapter, we will exclusively refer to the unary unbiased black-box algorithms defined via Algorithm~\ref{alg:BBCunbiasedAlgo}.

\subsection{Existing Results for Pseudo-Boolean Problems}\label{sec:BBCunbiased-results}
We survey existing bounds for the unbiased black-box complexity of several classical benchmark functions. As in previous sections, we proceed by function class, and not in historical order. 

\subsubsection{Functions with Unique Global Optimum}\label{sec:BBCunbiasedunique}

As discussed in Section~\ref{sec:BBCclasses}, the unrestricted black-box complexity of every function class $\F=\{f\}$ containing only one function $f$ is one, certified by the algorithm that simply queries a point $x \in \arg\max f$ in the first step. The situation is different in the unbiased black-box model, as the following theorem reveals.

\begin{theorem}[Theorem 6 in~\cite{LehreW10}]
\label{thm:BBCunbiasedunique}
Let $f:\{0,1\}^n \to \R$ be a function that has a single global optimum (i.e., in the case of maximization, the size of the set $\arg\max f$ is one). 
The unary unbiased black-box complexity of $f$ is $\Omega(n \log n)$.
\end{theorem}

Theorem~\ref{thm:BBCunbiasedunique} gives a $\Omega(n \log n)$ lower bound for the unary unbiased black-box complexity of several standard benchmark functions, such as \onemax, \leadingones, etc. We shall see below that for some of these classes, including \onemax, this bound is tight, since it is met by different unary unbiased heuristics, such as the $(1+1)$ EA or RLS. For other classes, including \leadingones, the lower bound can be improved through problem-specific arguments. 

The proof of Theorem~\ref{thm:BBCunbiasedunique} uses multiplicative drift analysis. To this end, the potential $P(t)$ of an algorithm at time $t$ is defined as the smallest Hamming distance of any of the previously queried search points $x^{(1)},\ldots,x^{(t)}$ to the unique global optimum $z$ or its bit-wise complement $\bar{z}$. The algorithm has identified $z$ (or its complement) if and only if $P(t)=0$. The distance to $\bar{z}$ needs to be regarded as the algorithm that first identifies $\bar{z}$ and then flips all bits obtains $z$ from $\bar{z}$ in only one additional query. As we have discussed for \textsc{jump} in Section~\ref{sec:BBCunrestrictedjump}, for some functions it can be substantially easier to identify $\bar{z}$ than to identify $z$ itself. This is true in particular if there are paths leading to $\bar{z}$ such as in the original jump functions $f_{\ell,z}$ discussed in Section~\ref{sec:BBCjumpalternative}. The key step in the proof of Theorem~\ref{thm:BBCunbiasedunique} is to show that in one iteration $P(t)$ decreases by at most $200 P(t)/n$, in expectation, provided that $P(t)$ is between $c\log\log n$ (for some positive constant $c>0$) and $n/5$. Put differently, in this case, $\E[P(t)-P(t+1)\mid P(t)] \leq \delta P(t)$ for $\delta:=200/n$. It is furthermore shown that the probability to make very large gains in potential is very small. These two statements allow the application of a multiplicative drift theorem, which bounds the total expected optimization time by $\Omega((\log(n/10)-\log\log(n))/\delta) = \Omega(n \log n)$, provided that the algorithm reaches a state $t$ with $P(t)\in (n/10,n/5]$. A short proof that every unary unbiased black-box algorithm reaches such a state with probability $1-e^{-\Omega(n)}$ then concludes the proof of Theorem~\ref{thm:BBCunbiasedunique}.

\subsubsection{OneMax}\label{sec:BBCunbiasedOM}

\textbf{The unary unbiased black-box complexity of OneMax.} Being a unimodal function, the lower bound of Theorem~\ref{thm:BBCunbiasedunique} certainly applies to \onemax, thus showing that no unary unbiased black-box optimization can optimize \onemax faster than in expected time $\Omega(n \log n)$. This bound is attained by several classical heuristics such as Randomized Local Search (RLS), the (1+1)~Evolutionary Algorithm (EA), and others. While the (1+1)~EA has an expected optimization time of $(1\pm o(1))e n \ln(n)$~\cite{DoerrFW10,Sudholt13}, that of RLS is only $(1\pm o(1)) n \ln(n)$. More precisely, it is $n \ln (n/2) + \gamma n \pm o(1)$~\cite{DoerrD16impact}, where $\gamma \approx 0.5772156649\ldots$ is the Euler-Mascheroni constant. The unary unbiased black-box complexity of \onemax is just slightly smaller than this term. It had been slightly improved by an additive $\sqrt{n \log n}$ term in~\cite{LaillevaultDD15} through iterated initial sampling. In~\cite{DoerrDY16} the following very precise bound for the unary unbiased black-box complexity of \onemax is shown. It is smaller than the expected running time of RLS by an additive term that is between $0.138 n \pm o(n)$ and $0.151 n \pm o(n)$. It is also proven in~\cite{DoerrDY16} that a variant of RLS that uses fitness-dependent neighborhood structures attains this optimal bound, up to additive $o(n)$ lower order terms. 

\begin{theorem}[Theorem~9 in~\cite{DoerrDY16}]
\label{thm:BBConemaxunary}
The unary unbiased black-box complexity of \onemax is $n \ln(n) - cn \pm o(n)$ for a constant $c$ between $0.2539$ and $0.2665$.
\end{theorem} 

\textbf{The binary unbiased black-box complexity of OneMax.} When Lehre and Witt initially defined the unbiased black-box model, they conjectured that also the binary black-box complexity of \onemax was $\Omega(n \log n)$ [personal communication of Lehre in 2010]. In light of understanding the role and the usefulness of crossover in black-box optimization, such a bound would have indicated that crossover cannot be beneficial for simple hill climbing tasks. Given that in 2010 all results seemed to indicate that it is at least very difficult, if not impossible, to rigorously prove advantages of crossover for problems with smooth fitness landscapes, this conjecture came along very naturally. It was, however, soon refuted. In~\cite{DoerrJKLWW11}, a binary unbiased algorithm is presented that achieves linear expected running time on \onemax. 
\begin{theorem}[Theorem~9 in~\cite{DoerrJKLWW11}]
\label{thm:BBConemaxbinary}
The binary unbiased black-box complexity of \onemax and that of any other monotone function is at most linear in the problem dimension $n$.
\end{theorem}

This bound is attained by the algorithm that keeps in the memory two strings $x$ and $y$ that agree in those positions for which the optimal entry has been identified already, and which differ in all other positions. In every iteration, the algorithm flips a fair random coin and, depending on the outcome of this coin flip, flips exactly one bit in $x$ or one bit in $y$. The bit to be flipped is chosen uniformly at random from those bits in which $x$ and $y$ disagree. The so-created offspring replaces its parent if and only if its function value is larger. In this case, the Hamming distance of $x$ and $y$ reduces by one. Since the probability to choose the right parent equals $1/2$, it is not difficult to show that, with high probability, for all constants $\eps>0$, this algorithm has optimized \onemax after at most $(2+\eps)n$ iterations. Together with Lemma~\ref{lem:BBChighprobability}, this proves Theorem~\ref{thm:BBConemaxbinary}.

A drawback of this algorithm is that it is very problem-specific, and it has been an open question whether or not a ``natural'' binary evolutionary algorithm can achieve an $o(n \log n)$ (or better) expected running time on \onemax. This question was affirmatively answered in~\cite{DoerrDE15} and~\cite{DoerrD15self}, as we shall discuss in Section~\ref{sec:BBCalgo}.

Whether or not the linear bound of Theorem~\ref{thm:BBConemaxbinary} is tight remains an open problem. In general, proving lower bounds for the unbiased black-box models of arities larger than one remains one of the biggest challenges in black-box complexity theory. Due to the greatly enlarged computational power of black-box algorithms using higher arity operators, proving lower bounds in these models seems to be significantly harder than in the unary unbiased model. As a matter of fact, the best lower bound that we have for the binary unbiased black-box complexity of \onemax is the $\Omega(n/\log n)$ one stated in Theorem~\ref{thm:BBCunrestrictedOM}, and not even constant-factor improvements of this bound exist.

\textbf{The $k$-ary unbiased black-box complexity of OneMax.} In~\cite{DoerrJKLWW11}, a general bound for the $k$-ary unbiased black-box complexity of \onemax of order $n/\log k$ had been presented (cf. Theorem~9 in~\cite{DoerrJKLWW11}). This bound has been improved in~\cite{DoerrW14arity}.

\begin{theorem}[Theorem~3 in~\cite{DoerrW14arity}]
\label{thm:BBConemaxkary}
For every $2 \leq k \leq \log n$, the $k$-ary unbiased black-box complexity of \onemax is of order at most $n/k$. For $k \ge \log n$, it is $\Theta(n/\log n)$.
\end{theorem}
Note that for $k \ge \log n$, the lower bound in Theorem~\ref{thm:BBConemaxkary} follows from the $\Omega(n/\log n)$ unrestricted black-box complexity of \onemax discussed in Theorem~\ref{thm:BBCunrestrictedOM}. 

The main idea to achieve the results of Theorem~\ref{thm:BBConemaxkary} can be easily described. For a given $k$, the bit string is split into blocks of size $k-2$. This has to be done in an unbiased way, so that the ``blocks'' are not consecutive bit positions, but some random $k-2$ not previously optimized ones. Similarly to the binary case, two reference strings $x$ and $y$ are used to encode which $k-2$ bit positions are currently under investigation; namely the $k-2$ bits in which $x$ and $y$ disagree. By the same encoding, two other strings $x'$ and $y'$ store which bits have been optimized already, and which ones have not been investigated so far. To optimize the $k-2$ bits in which $x$ and $y$ differ, the derandomized version of the result of Erd{\H{o}}s and R{\'e}nyi (Theorem~\ref{thm:BBCunrestrictedOMderandomized}) is used. Applied to our context, this result states that there exists a sequence of $\Theta(k/\log k)$ queries which uniquely determines the entries in the $k-2$ positions. Since $\Theta(n/k)$ such blocks need to be optimized, the claimed total expected optimization time of $\Theta(n/\log k)$ follows. Some technical difficulties need to be overcome to implement this strategy in an unbiased way. To this end, in~\cite{DoerrW14arity} a generally applicable \textbf{encoding strategy} is presented that with $k$-ary unbiased variation operators simulates a memory of $2^{k-2}$ bits that can be accessed in an unrestricted fashion.

\subsubsection{LeadingOnes}\label{sec:BBCunbiasedLO}

\textbf{The unary unbiased black-box complexity of LeadingOnes.} 
Being a classic benchmark problem, non-surprisingly, Lehre and Witt presented already in~\cite{LehreW12} a first bound for the unbiased back-box complexity of \leadingones. 

\begin{theorem}[Theorem~2 in~\cite{LehreW12}]
\label{thm:BBCLOunary}
The unary unbiased black-box complexity of \leadingones is $\Theta(n^2)$.
\end{theorem} 

Theorem~\ref{thm:BBCLOunary} can be proven by drift analysis. To this end, in~\cite{LehreW12} a potential function is defined that maps the state of the search process at time $t$ (i.e., the sequence $\{\big(x^{(1)},f(x^{(1)})\big), \ldots, \big(x^{(t)},f(x^{(t)})\big)\}$ of the pairs of search points evaluated so far and their respective function values) to the largest number of initial ones and initial zeros in any of the $t+1$ strings $x^{(1)},\ldots,x^{(t)}$. It is then shown that a given a potential $k$ cannot increase in one iteration by more than an additive $4/(k+1)$ term, in expectation, provided that $k$ is at least $n/2$. Since with probability at least $1-e^{-\Omega(n)}$ any unary unbiased black-box algorithm reaches a state in which the potential is between $n/2$ and $3n/4$, and since from this state a total potential of at least $n/4$ must be gained, the claimed $\Omega(n^2)$ bound follows from a variant of the additive drift theorem. More precisely, using these bounds, the additive drift theorem allows shows that the total optimization time of any unary unbiased black-box algorithm is at least $(n/4)/\big(4/(n/2))\big)=\Omega(n^2)$.

\textbf{The binary unbiased black-box complexity of LeadingOnes.} 
Similarly to the case of \onemax, the binary unbiased black-box complexity of \leadingones is much smaller than its unary one. 

\begin{theorem}[Theorem~14 in~\cite{DoerrJKLWW11}]
\label{thm:BBCLObinary}
The binary unbiased black-box complexity of \leadingones is $O(n \log n)$.
\end{theorem}
The algorithm achieving this bound borrows its main idea from the binary unbiased one used to optimize \onemax in linear time, which we have described after Theorem~\ref{thm:BBConemaxbinary}. We recall that the key strategy was to use two strings to encode those bits that have been optimized already. In the $O(n \log n)$ algorithm for \leadingones this approach is combined with a \emph{binary search} for the (unique) bit position that needs to be flipped next. Such a binary search step requires $O(\log n)$ steps in expectation. Iterating it $n$ times gives the claimed $O(n \log n)$ bound. 

As in the case of \onemax, it is not known whether the bound of Theorem~\ref{thm:BBCLObinary} is tight. The best known lower bound is the $\Omega(n \log\log n)$ one of the unrestricted black-box model discussed in Theorem~\ref{thm:BBCunrestrictedLO}.

\textbf{The complexity of LeadingOnes in the unbiased black-box models of higher arity.} 
The $O(n \log n)$ bound presented in Theorem~\ref{thm:BBCLObinary} further reduces to at most $O(n \log (n) / \log \log n)$ in the ternary unbiased black-box model. 

\begin{theorem}[Theorems~2 and~3 in~\cite{DoerrW11EA}]
\label{thm:BBCLOternary}
For every $k\ge 3$, the $k$-ary unbiased black-box complexity of \leadingones is of order at most $n \log (n) / \log \log n$. This bound also holds in the combined $k$-ary unbiased ranking-based black-box model, in which instead of absolute function values the algorithm can make use only of the ranking of the search points induced by the optimization problem instance $f$.
\end{theorem}

The algorithm that certifies the upper bound of Theorem~\ref{thm:BBCLOternary} uses the additional power gained through the larger arity to \emph{parallelize} the binary search of the binary unbiased algorithm described after Theorem~\ref{thm:BBCLObinary}. More precisely, the optimization process is split into phases. In each phase, the algorithm identifies the entries of up to $k:=O(\sqrt{\log n})$ positions. It is shown that each phase takes $O(k^3/\log k^2)$ steps in expectation. Since there are $n/k$ phases, a total expected optimization time of $O(nk^2/\log k^2) = O(n \log (n) / \log \log n)$ follows. 

The idea to parallelize the search for several indices was later taken up and further developed in~\cite{AfshaniADLMW12}; where an iterative procedure with \emph{overlapping} phases is used to derive the asymptotically optimal $\Theta(n \log\log n)$ unrestricted black-box algorithm that proves Theorem~\ref{thm:BBCunrestrictedLO}. 

It seems plausible that higher arities allow a larger degree of parallelization, but no formal proof of this intuition exists. In the context of \leadingones, it would be interesting to derive a lower bound for the smallest value of $k$ such that an asymptotically optimal $k$-ary unbiased $\Theta(n \log\log n)$ black-box algorithm for \leadingones exists. As a first step towards answering this question, the above-sketched encoding and sampling strategies could be applied to the algorithm presented in~\cite{AfshaniADLMW12}, to understand the smallest arity needed to implement this algorithm in an unbiased way.

\subsubsection{Jump}\label{sec:BBCunbiasedjump}

Jump functions are benchmark functions, which are observed as difficult for evolutionary approaches because of their large plateaus of constant and low fitness around the global optimum. One would expect that this is reflected in its unbiased back-box complexity, at least in the unary model. Surprisingly, this is not the case. In~\cite{DoerrDK15jump} it is shown that even extreme jump functions that reveal only the three different fitness values $0$, $n/2$, and $n$ have a small polynomial unary unbiased black-box complexity. That is, they can be optimized quite efficiently by unary unbiased approaches. This result indicates that efficient optimization is not necessarily restricted to problems for which the function values reveal a lot of information about the instance at hand.  

As discussed in Section~\ref{sec:BBCunrestrictedjump}, the literature is not unanimous with respect to how to generalize the jump function defined in~\cite{DrosteJW02} to a problem class. The results stated in the following apply to the jump function defined in~\eqref{def:BBCjump}. In the unbiased black-box model, we can assume without loss of generality that the underlying target string is the all-ones string $(1,\ldots,1)$. That is, to simplify our notation, we drop the subscript $z$ and assume that for every $\ell < n/2$ we regard the function that assigns to every $x \in \{0,1\}^n$ the function value
$$
\jump_{\ell}(x) := 
\begin{cases}
n,				&\mbox{if }|x|_1 = n,\\
|x|_1,		&\mbox{if }\ell < |x|_1 < n-\ell,\\
0,				&\mbox{otherwise.}
\end{cases}
$$

The results in~\cite{DoerrDK15jump} cover a broad range of different combinations of jump sizes~$\ell$ and arities~$k$. 

\begin{theorem}[\cite{DoerrDK15jump}]\label{thm:BBCunbiasedjump}
The following table summarizes the known bounds for the unbiased black-box complexity of $\jump_\ell$ in the different models\\
\vspace{0.5ex}
\begin{center}
\begin{tabular}{c|c|c|c}
& \text{Short Jump} & 
\text{Long Jump} & 
\text{Extreme Jump}\\
\rule{5mm}{0cm}Arity\rule{5mm}{0cm} & 
{$\ell = O(n^{1/2 - \eps})$}\rule{1mm}{0cm} & 
{$\ell = (1/2 - \varepsilon)n$}\rule{1mm}{0cm} & 
{$\ell = n/2 - 1$}\rule{1mm}{0cm}\\ 
\hline
$k = 1$ &  $\Theta(n \log n)$ 
& $O(n^2)$ & $O(n^{9/2})$ \\
$k = 2$ & $O(n)$ & $O(n \log n)$ & $O(n \log n)$\\
$3 \leq k \leq \log n$ & $O(n / k)$ & $O(n / k)$ & $\Theta(n)$
\end{tabular}	
\end{center}
\end{theorem}

To discuss the bounds of Theorem~\ref{thm:BBCunbiasedjump}, we proceed by problem type. Almost all proofs are rather involved, so that we sketch here only the main ideas.

\textbf{Short jumps, i.e., $\ell = O(n^{1/2 - \eps})$.} A comparison with the bounds discussed in Section~\ref{sec:BBCunbiasedOM} shows that the above-stated bounds for the $k$-ary unbiased black-box complexities of short jump functions are of the same order as those for $\onemax$ (which can be seen as a jump function with parameter $\ell = 0$). In fact, it is shown in~\cite[Lemma~3]{DoerrDK15jump} that a black-box algorithm having access to a jump function with $\ell= O(n^{1/2 - \eps})$ can retrieve (with high probability) the true $\onemax$ value of a search point using only a constant number of queries. The other direction is of course also true, since from the \onemax value we can compute the $\jump_{\ell}$ value without further queries. This implies that the black-box complexities of short jump functions are of the same asymptotic order as those of \onemax. Any improved bound for the $k$-ary unbiased black-box complexity \onemax therefore immediately carries over to short jump functions.

\textbf{Long jumps, i.e., $\ell = (1/2 - \varepsilon)n$.} Despite the fact that the above-mentioned Lemma~3 from~\cite{DoerrDK15jump} can probably not be directly extended to long jump functions, the bounds for arities $k \ge 3$ nevertheless coincides with those of \onemax. In fact, it is shown in~\cite[Theorem~6]{DoerrDK15jump} that for all $\ell<(1/2 - \varepsilon)n$ and for all $k\ge 3$ the $k$-ary unbiased black-box complexity of $\jump_{\ell}$ is at most of the same asymptotic order than the $(k-2)$-ary one of \onemax. For $k>3$ this proves the bounds stated in Theorem~\ref{thm:BBCunbiasedjump}. The linear bound for $k=3$ follows from the case of extreme jumps. 

A key ingredient for the bound on the \emph{unary} unbiased black-box complexity of long jump functions is a procedure that samples a number of neighbors at some fixed distance $d$ and that studies the empirical expected function values of these neighbors to decide upon the direction in which the search for the global optimum is continued. More precisely, it uses the samples to estimate the \onemax value of the currently investigated search point. Strong concentration bounds are used to bound the probability that this approach gives an accurate estimation of the correct \onemax values.  

The $O(n \log n)$ bound for the \emph{binary} unbiased black-box complexity of long jump functions follows from its extreme analog.

\textbf{Extreme jump, i.e., $\ell = n/2-1$.} \cite{DoerrDK15jump} first regards the \emph{ternary} unbiased black-box complexity of the extreme jump function. A strategy allowing to test individual bits is derived. Testing each bit individually in an efficient way (using the encoding strategies originally developed in~\cite{DoerrJKLWW11} and described in Section~\ref{sec:BBCunbiasedOM} above) gives the linear bound. 

In the \emph{binary} case, the bits cannot be tested as efficiently any more. The main idea is nevertheless to flip individual bits and to test if the flip was in a ``good'' or a ``bad'' direction. This test is done by estimating the distance to a reference string with $n/2$ ones. Implementing this strategy in $O(n\log n)$ queries requires to overcome a few technical difficulties imposed by the restriction to sample only from binary unbiased distributions, resulting in a rather complex bookkeeping procedure, and a rather technical 4.5 pages long proof.

Finally, the polynomial unary unbiased black-box complexity of extreme jump is proven as follows. Similarly to the cases discussed above, individual bits are flipped in a current ``best'' solution candidate $x$. A sampling routine is used to estimate if the bit flip was in a ``good'' or a ``bad'' direction, i.e., if it created a string that is closer to the global optimum or its bit-wise complement than the previous one. The sampling strategy works as follows. Depending on the estimated parity of $|x|_1$, exactly $n/2$ or $n/2-1$ bits are flipped in $x$. The fraction of so-created offspring with function value $n/2$ (the only value that is ``visible'' apart from that of the global optimum) is recorded. This fraction depends on the distance of $x$ to the global optimum $(1,\ldots,1)$ or its complement $(0,\ldots,0)$ and is slightly different for different distances. A key step in the analysis of the unary unbiased black-box complexity of extreme jump is therefore a proof that shows that a polynomial number of such samples are sufficient to determine the \onemax-value of $x$ with sufficiently large probability.   

\textbf{Comments on the upper bounds Theorem~\ref{thm:BBCunbiasedjump}.} Note that already for long jump functions, the search points having a function value of $0$ form plateaus around the optimum $(1,\ldots,1)$ and its complement $(0,\ldots,0)$ of exponential size. For the extreme jump function, even all but a $\Theta(n^{-1/2})$ fraction of the search points form one single fitness plateau. Problem-unspecific black-box optimization techniques will therefore typically not find the optimum of long and extreme jump functions in subexponential time.

\textbf{Lower bound.} The $\Omega(n \log n)$ lower bound in Theorem~\ref{thm:BBCunbiasedjump} follows from the more general result discussed in Theorem~\ref{thm:BBClowerBoundForMapped} and the $\Omega(n \log n)$ bound for \onemax in the unary unbiased black-box model, which we have discussed in Section~\ref{sec:BBCunbiasedOM}. Note also that Theorem~\ref{thm:BBClowerBoundForMapped}, together with the $\Omega(n/\log n)$ unrestricted black-box complexity of \onemax implies a lower bound for $\jump_{\ell}$ of the same asymptotic order (for all values of $\ell$). The linear lower bound for extreme jump can be easily proven by the information-theoretic arguments presented in Theorem~\ref{thm:BBCinfotheo}. Intuitively, the algorithm needs to learn a linear number of bits, while it receives only a constant number per function evaluation. 

\textbf{Insights from these bounds and open questions.} The proof sketches provided above highlight that one of the key novelties presented in~\cite{DoerrDK15jump} are the sampling strategies that are used to estimate the \onemax-values of a current string of interest. The idea to accumulate some statistical information about the fitness landscape could be an interesting concept for the design of novel heuristics; in particular for optimization in the presence of noisy function evaluations or for dynamic problems, which change over time. 

\subsubsection{Number Partition}\label{sec:BBCunbiasedpartition}

Number partition is one of the best-known NP-hard problems. Given a set $S \subset \N^n$ of $n$ positive integers, this partition problem asks to decide whether or not it is possible to split $S$ in two disjoint subsets such that the sum of the integers in these two subsets is identical, i.e., whether or not two disjoint subsets $S_1$ and $S_2$ of $S$ with $S_1 \cup S_2 = S$ and $\sum_{s\in S_1}{s}=\sum_{s\in S_2}{s}$ exist. The optimization version of partition asks to split $S$ into two disjoint subsets such that the absolute discrepancy $\big|\sum_{s\in S_1}{s}=\sum_{s\in S_2}{s} \big|$ is as small as possible.  

In~\cite{DoerrDK14} a subclass of partition is studied in which the integers in $S$ are pairwise different. The problem remains NP-hard under this assumption. It is thus unlikely that it can be solved efficiently. For two different formulations of this problem (using a signed and an unsigned function assigning to each partition $S_1,S_2$ of $S$ the discrepancy $\sum_{s\in S_1}{s}=\sum_{s\in S_2}{s}$ or the absolute value of this expression, respectively) it is shown that the unary unbiased black-box complexity of this subclass is nevertheless of a small polynomial order. More precisely, it is shown that there are unary unbiased black-box algorithms that need only $O(n \log n)$ function evaluations to optimize any Partition$_{\neq}$ instance. The proof techniques are very similar to the ones presented in Section~\ref{sec:BBCNPhard}: the algorithm achieving the $O(n \log n)$ expected optimization time first uses $O(n \log n)$ steps to learn the problem instance at hand. After some (possibly---and probably---non-polynomial-time) offline computation of an optimal solution for this instance, this optimum is then created via an additional $O(n \log n)$ function evaluations, needed to move the integers of the partition instance to the right subset. Learning and moving the bits can be done in linear time in the unrestricted model. The $\log n$ factor stems from the fact that here in this unary unbiased model, in every iteration a random bit is moved, so that a coupon collector process results in the logarithmic overhead. 

This result and those for the different jump versions described in Section~\ref{sec:BBCunbiasedjump} show that the unary unbiased black-box complexity can be much smaller than the typical performance of mutation-only black-box heuristics. This indicates that the unary unbiased black-box model does not always give a good estimation for the difficulty of a problem when optimized by mutation-based algorithms.
As we shall discuss in Section~\ref{sec:BBCcombined}, a possible direction to obtain more meaningful results can be to restrict the class of algorithms even further, e.g., through bounds on the memory size or the selection operators. 

\subsubsection{Minimum Spanning Trees}\label{sec:BBCunbiasedMST}
Having a formulation over the search space $\{0,1\}^m$, the minimum spanning tree problem regarded in Section~\ref{sec:BBCunrestrictedMST} can be directly studied in the unbiased black-box model proposed by Lehre and Witt. The following theorem summarizes the bounds proven in~\cite{DoerrKLW13} for this problem. We see here that~\cite{DoerrKLW13} also studied the black-box complexity of a model that combines the restrictions imposed by the ranking-based and the unbiased black-box models. We will discuss this model in Section~\ref{sec:BBCcombined} but, for the sake of brevity, state the bounds already for this combined model.

\begin{theorem}[Theorem~10 in~\cite{DoerrKLW13}]
\label{thm:BBCunbiasedMST}
The \emph{unary unbiased} black-box complexity of the MST problem is $O(m n \log(m/n))$ if there are no duplicate weights and $O(mn \log n)$ if there are.
The \emph{ranking-based unary unbiased} of the MST problem black-box complexity is $O(m n \log n)$.
Its \emph{ranking-based binary unbiased} black box-complexity is $O(m \log n)$ and its \emph{ranking-based $3$-ary unbiased} black-box complexity is $O(m)$.

For every $k$, the $k$-ary unbiased black-box complexity of MST for $m$ edges is at least as large as the $k$-ary unbiased black-box complexity of $\onemax_m$.
\end{theorem}
As in the unrestricted case of Theorem~\ref{thm:BBCunrestrictedMST}, the upper bounds in Theorem~\ref{thm:BBCunbiasedMST} are obtained by modifying Kruskal's algorithm to fit the black-box setting at hand. For the lower bound, the path $P$ on $m+1$ vertices and unit edge weights shows that $\onemax_m$ is a sub-problem of the MST problem. More precisely, for all bit strings $x \in \{0,1\}^m$, the function value $f(x) = (\onemax_m(x),m+1-\onemax_m(x))$ of the associated MST fitness function reveals the \onemax-value of $x$. 

\subsubsection{Other Results}\label{sec:BBCunbiasedresultsother}

Motivated to introduce a class of functions for which the unary unbiased black-box complexity is $\Theta(2^m)$, for some parameter $m$ that can be scaled between $1$ and $n$, Lehre and Witt introduced in~\cite{LehreW12} the following function. 
\begin{align}\label{def:BBCOMNeedle}
\textsc{OM-Needle}: \{0,1\}^n \to [0..n], x \mapsto \sum_{i=1}^{n-m}{x_i} + \prod_{i=1}^n{x_i}. 
\end{align}
It is easily seen that this function has its unique global optimum in the all-ones string $(1,\ldots,1)$. All other search points whose first $n-m$ entries are equal to one are located on a plateau of function value $n-m$. In the unbiased model, this part is thus similar to the \textsc{Needle} functions discussed in Section~\ref{sec:BBCunrestrictedneedle}. Lehre and Witt show that for $0 \le m \le n$ the unary unbiased black-box complexity of this function is at least $2^{m-2}$~\cite[Theorem~3]{LehreW12}. Note that this function is similar in flavor to the jump version proposed in~\cite{Jansen15} (cf. Section~\ref{sec:BBCjumpalternative}).

\subsection{Beyond Pseudo-Boolean Optimization: Unbiased Black-Box Models for Other Search Spaces}\label{sec:BBCunbiased-other}

In this section we discuss an extension of the pseudo-Boolean unbiased black-box model by Lehre and Witt~\cite{LehreW12} to more general search spaces. To this end, we first recall from Definition~\ref{def:BBCunbiased-variation} that the unbiased model was defined through a set of invariances that must be satisfied by the probability distributions from which unbiased algorithms sample their solution candidates. It is therefore quite natural to first generalize the notion of an unbiased operator in the following way. 

\begin{definition}[Definition~1 in~\cite{DoerrKLW13}]
\label{def:BBCunbiased-generalized}
Let $k \in \N$, let $S$ be some arbitrary set, and let~$\mathcal{G}$ be a set of bijections on $S$ that forms a group, i.e., a set of one-to-one maps $g:S \to S$ that is closed under composition $(\cdot \circ \cdot)$ and under inversion $(\cdot)^{-1}$. We call $\mathcal{G}$ the \emph{set of invariances}.
 
A \emph{$k$-ary $\mathcal{G}$-unbiased distribution} is a family of probability distributions\linebreak[4] $\big(D(\cdot \,|\, y^{1},\ldots,y^{k})\big)_{y^{1},\ldots,y^{k} \in S}$ over 
$S$ such that for all \emph{inputs} $y^{1},\ldots,y^{k} \in S$ the condition
\[
\forall x \in S \, \forall g \in \mathcal{G}: D(x \mid y^{1},\ldots,y^{k}) = D(g(x) \mid g(y^{1}), \ldots, g(y^{k}))
\]
holds. An operator sampling from a $k$-ary $\mathcal{G}$-unbiased distribution is called a \emph{$k$-ary $\mathcal{G}$-unbiased variation operator}. 
\end{definition}

For $S:=\{0,1\}^n$ and for $\mathcal{G}$ being the set of Hamming-automorphisms, it is not difficult to verify that Definition~\ref{def:BBCunbiased-generalized} extends Definition~\ref{def:BBCunbiased-variation}. A $k$-ary $\mathcal{G}$-unbiased black-box algorithm is one that samples all search points from $k$-ary $\mathcal{G}$-unbiased variation operators.

In~\cite{ABB}, Rowe and Vose give the following very general, but rather indirect, definition of unbiased distributions. 
\begin{definition}[Definition~2 in~\cite{ABB}]
\label{def:BBCgeneralunbiased}\label{def:BBCABB}
Let $\F$ be a class of functions from search space $S$ to some set $Y$. We say that a one-to-one map $\alpha: S \rightarrow S$ \emph{preserves} $\F$ if for all $f \in \F$ it holds that $f\circ \alpha \in \F$. Let $\Pi(\F)$ be the class of all such $\F$-preserving bijections $\alpha$. 

A \emph{$k$-ary generalized unbiased distribution (for $\F$)} is a $k$-ary $\Pi(\F)$-unbiased distribution. 
\end{definition}

It is argued in~\cite{ABB} that $\Pi(\F)$ indeed forms a group, so that Definition~\ref{def:BBCgeneralunbiased} satisfies the requirements of Definition~\ref{def:BBCunbiased-generalized}. 

To apply the framework of Definition~\ref{def:BBCgeneralunbiased}, one has to make precise the set of invariances covered by the class $\Pi(\F)$. This can be quite straightforward in some cases~\cite{ABB} but may require some more effort in others~\cite{DoerrKLW13}. In particular, it is often inconvenient to define the whole family of unbiased distributions from which a given variation operator originates. Luckily, in many cases this effort can be considerably reduced to proving only the unbiasedness of the variation operator itself. The following theorem demonstrates this for the case $S=[n]^{n-1}$, which is used, for example in the single-source shortest path problem regarded in the next subsection. In this case, condition (ii) states that it suffices to show the $k$-ary $\mathcal{G}$-unbiasedness of the distribution $D_{\vec{z}}$, without making precise the whole family of distributions associated to it. 
\begin{theorem}
\label{thm:BBCunbiasedsimplifiaction}
Let $\mathcal{G}$ be a set of invariances, i.e., a set of permutations of the search space $S=[n]^{n-1}$ that form a group. Let $k\in\N$, and $\vec{z} = (z^{1},\ldots,z^{k})\in S^k$ be a $k$-tuple of search points. Let 
\[
\mathcal{G}_0 := \{g \in \mathcal{G} \mid g(z^{i}) = z^{i} \text{ for all } i\in[k]\}
\] 
be the set of all invariances that leave $z^{1},\ldots,z^{k}$ fixed.

Then for any probability distribution $D_{\vec{z}}$ on $[n]^{n-1}$, the following two statements are equivalent.
\begin{enumerate}
\item[(i)] 
There exists a $k$-ary $\mathcal{G}$-unbiased distribution $(D(\cdot \,|\, \vec{y}))_{\vec{y}\in S^k}$ on $S$ such that $D_{\vec{z}} = D(\cdot \,|\, \vec{z})$.
\item[(ii)] For every $g\in\mathcal{G}_0$ and for all $x \in S$ it holds that $D_{\vec{z}}(x) = D_{\vec{z}}(g(x))$.
 \end{enumerate}
\end{theorem}

\subsubsection{Alternative Extensions of the Unbiased Black-Box Model for the SSSP problem}\label{sec:BBCunbiasedSSSP}

As discussed in Section~\ref{sec:BBCunrestrictedSSSP}, several formulation of the single-source shortest path problem (SSSP) co-exist. In the unbiased black-box setting, the multi-criteria formulation is not very meaningful, as the function values explicitly distinguish between the vertices, so that treating them in an unbiased fashion seems ill-natured. For this reason, in \cite{DoerrKLW13} only the single-objective formulation is investigated in the unbiased black-box model. Note that for this formulation, the unbiased black-box model for pseudo-Boolean functions needs to be extended to the search space $S_{[2..n]}$. \cite{DoerrKLW13} discusses three different extensions: 
\begin{enumerate}
	\item a \emph{structure preserving} unbiased model in which, intuitively speaking, the operators do not regard the \emph{labels} of different nodes, but only their local structure (e.g., the size of their neighborhoods), 
	\item the \emph{generalized} unbiased model proposed in~\cite{ABB} (this model follows the approach presented in Section~\ref{sec:BBCunbiased-other} above), and 
	\item a \emph{redirecting} unbiased black-box model in which, intuitively, a node may choose to change its predecessor in the shortest path tree but if it decides to do so, then all possible predecessors must be equally likely to be chosen. 
\end{enumerate}
Whereas all three notions a priori seem to capture different aspects of what unbiasedness in the SSSP problem could mean, two of them are shown to be too powerful. More precisely, it is shown that already the \emph{unary} structure preserving unbiased black-box complexity of SSSP as well as its \emph{unary} generalized unbiased black-box complexity are almost identical to the unrestricted one. The three models are proven to differ by at most one query in~\cite[Theorem~25 and Corollary~32]{DoerrKLW13}. 

It is then shown that the redirecting unbiased black-box model yields more meaningful black-box complexities. 

\begin{theorem}[Corollary~28, Theorem~29 and Theorem~30 in~\cite{DoerrKLW13}]
\label{thm:BBCunbiasedSSSPredirecting}
The unary ranking-based redirecting unbiased black-box complexity of SSSP is $O(n^3)$. Its binary ranking-based redirecting unbiased black-box complexity is $O(n^2 \log n)$. For all $k \in \N$, the $k$-ary redirecting unbiased black-box complexity of SSSP is $\Omega(n^2)$.
\end{theorem}
The unary bound is obtained by a variant of RLS which redirects in every step one randomly chosen node to a random predecessor. For the binary bound, the problem instance is learned in a 2-phase step. An optimal solution is then created by an imitation of Dijkstra's algorithm. For the lower bound, drift analysis is used to prove that every redirecting unbiased algorithm needs $\Omega(n^2)$ function evaluations to reconstruct a given path on $n$ vertices.

\section{Combined Black-Box Complexity Models}\label{sec:BBCcombined}

The black-box models discussed in previous sections study either the complexity of a problem with respect to \emph{all} black-box algorithms (in the unrestricted model) or they restrict the class of algorithms with respect to \emph{one} particular feature of common optimization heuristics, such as the size of their memory, their selection behavior, or their sampling strategies. As we have seen, many classical black-box optimization algorithms are members of several of these classes. At the same time, a non-negligible number of the upper bounds stated in the previous sections can, to date, only be certified by algorithms that satisfy an individual restriction, but clearly violate other requirements that are not controlled by the respective model. In the unbiased black-box model, for example, several of the upper bounds are obtained by algorithms that make use of a rather large memory size. It is therefore natural to ask if and how the black-box complexity of a problem increases if two or more of the different restrictions proposed in the previous sections are combined into a new black-box model. This is the focus of this section, which surveys results obtained in such combined black-box complexity models.

\subsection{Unbiased Ranking-Based Black-Box Complexity}
Already some of the very early works on the unbiased black-box model regarded a combination with the ranking-based model. In fact, the binary unbiased algorithm from~\cite{DoerrJKLWW11}, which solves \onemax with $\Theta(n)$ queries on average, only uses comparisons, and does not make use of knowing absolute fitness values. It was shown in~\cite{DoerrW14ranking} that also the other upper bounds for the $k$-ary black-box complexity of \onemax proven in~\cite{DoerrJKLWW11} hold also in the ranking-based version of the $k$-ary unbiased black-box model. 
\begin{theorem}[Theorem~6 and Lemma~7 in~\cite{DoerrW14ranking}]
\label{thm:BBCunbiasedrankingOM}
The unary unbiased ranking-based black-box complexity of $\onemax_n$ is $\Theta(n \log n)$. For constant $k$, the $k$-ary unbiased ranking-based black-box complexity of $\onemax_n$ and that of every strictly monotone function is at most $4n-5$. For $2 \leq k \leq n$, the $k$-ary unbiased ranking-based black-box complexity of $\onemax_n$ is $O(n/\log k)$.
\end{theorem}
In light of Theorem~\ref{thm:BBConemaxkary}, it seems plausible that the upper bounds for the case $2 \leq k \leq n$ can be reduced to $O(n/k)$ but we are not aware of any result proving such a claim. 

Also the binary unbiased algorithm achieving an expected $O(n \log n)$ optimization time on \leadingones uses only comparisons. 
\begin{theorem}[follows from the proof of Theorem~14 in~\cite{DoerrJKLWW11}, cf. Theorem~\ref{thm:BBCLObinary}]
\label{thm:BBCunbiasedrankingLO}
The binary unbiased ranking-based black-box complexity of \leadingones is $O(n \log n)$.
\end{theorem}
For the ternary black-box complexity we have mentioned already in Theorem~\ref{thm:BBCLOternary} that the $O(n \log (n) / \log \log n)$ bound also holds in the ranking-based version of the ternary unbiased black-box model.

Also for the two combinatorial problems MST and SSSP it has been mentioned already in Theorems~\ref{thm:BBCunbiasedMST} and \ref{thm:BBCunbiasedSSSPredirecting} that the bounds hold also in the models in which we require the algorithms to base all decisions only on the ranking of previously evaluated search points, and not on absolute function values.  

\subsection{Parallel Black-Box Complexity}\label{sec:BBCparallel}

The (unary) unbiased black-box model was also the starting point for the authors of~\cite{BadkobehLS14}, who introduce a black-box model to investigate the effects of a parallel optimization. Their model can be seen as a unary unbiased $(\infty+\lambda)$ memory-restricted black-box model. More precisely, their model covers all algorithms following the scheme of Algorithm~\ref{alg:BBCparallel}. 

The model covers $(\mu + \lambda)$ and $(\mu,\lambda)$ EAs, cellular EAs and unary unbiased EAs working in the island model. The restriction to unary unbiased variation operators can of course be relaxed to obtain general models for $\lambda$-parallel $k$-ary unbiased black-box algorithms. 

We see that Algorithm~\ref{alg:BBCparallel} forces the algorithms to query $\lambda$ new solution candidates in every iteration. Thus, intuitively, for every two positive integers $k$ and $\ell$ with $k/\ell \in \N$ and for all problem classes $\F$, the $\ell$-parallel unary unbiased black-box complexity of $\F$ is at most as large as its $k$-parallel unary unbiased one.

\begin{algorithm2e}
 \textbf{Initialization:} \\
 \Indp
	\lFor{$i=1,\ldots,\mu$}{Sample $x^{(i,0)}$ uniformly at random from $S$ and query $f\big(x^{(i,0)}\big)$}
	$\mathcal{I} \assign \{ f\big(x^{(1,0)}\big), \ldots, f\big(x^{(\lambda,0)}\big)\}$\;
 \Indm
 \textbf{Optimization:}	
 \For{$t=1,2,3,\ldots$}{
		\For{$i=1,\ldots, \lambda$}{
		Depending only on the multiset $\mathcal{I}$ choose a pair of indices $(j,\ell) \in [\lambda] \times [0..t-1]$\;
 		Depending only on the multiset $\mathcal{I}$ choose a unary unbiased probability distribution $D^{(i,t)}(\cdot)$ on $S$, 
		sample $x^{(i,t)} \assign D^{(i,t)}(x^{(j,\ell)})$ and query $f\big(x^{(i,t)}\big)$\;}
	$\mathcal{I} \assign \mathcal{I} \cup \{ f\big(x^{(1,t)}\big), \ldots, f\big(x^{(\lambda,t)}\big)\}$\;
	 }
 \caption{A blueprint for $\lambda$-parallel unary unbiased black-box algorithms for the optimization of an unknown function $f:S \rightarrow \R$}
\label{alg:BBCparallel}
\end{algorithm2e}

The following bounds for the $\lambda$-parallel unary unbiased black-box complexity are known. 

\begin{theorem}[Theorems~1, 3 and~4 in~\cite{BadkobehLS14}]
The $\lambda$-parallel unary unbiased black-box complexity of \leadingones is $\Omega\big(\frac{\lambda n}{\max\{1,\log(\lambda/n)\}}+n^2 \big)$. It is of order at most $\lambda n + n^2$.

For any $\lambda \le e^{\sqrt{n}}$, the $\lambda$-parallel unary unbiased black-box complexity of any function having a unique global optimum is $\Omega\big(\frac{\lambda n}{\log(\lambda)}+n \log n \big)$. This bound is asymptotically tight for \onemax.
\end{theorem}
For \leadingones, the upper bound is attained by a $(1+\lambda)$~EA investigated in~\cite{LassigS14TCS}. The lower bound is shown by means of drift analysis, building upon the arguments used in~\cite{LehreW12} to prove Theorem~\ref{thm:BBCLOunary}. 

For \onemax, a $(1+\lambda)$~EA with fitness-dependent mutation rates is shown to achieve the $O\big(\frac{\lambda n}{\log(\lambda)}+n \log n \big)$ expected optimization time in~\cite[Theorem~4]{BadkobehLS14}, confer Chapter~\ref{chap:adaptive} [link to the chapter on non-static parameter choices will be added] for details. 

The lower bound for the $\lambda$-parallel unary unbiased black-box complexity of functions having a unique global optimum uses additive drift analysis. The proof is similar to the proof of Theorem~\ref{thm:BBCunbiasedunique} in~\cite{LehreW12}, but requires a very precise tail bound for hypergeometric variables (Lemma~2 in~\cite{BadkobehLS14}).

\subsection{Distributed Black-Box Complexity}\label{sec:BBCdistributed}

To study the effects of the migration topology on the efficiency of distributed evolutionary algorithms, the $\lambda$-parallel unary unbiased black-box model was extended in~\cite{BadkobehLS15} to a distributed version, in which the islands exchange their accumulated information along a given graph topology. Commonly employed topologies are the complete graph (in which all nodes exchange information with each other), the ring topology, the grid of equal side lengths, and the torus. \cite{BadkobehLS15} presents an unrestricted and a unary unbiased version of the distributed black-box model. In this context, it is interesting to study how the black-box complexity of a problem increases with sparser migration topologies or with the infrequency of migration. The model of \cite{BadkobehLS15} allows all nodes to share all the information that they have accumulated so far. Another interesting extension of the distributed model would be to study the effects of bounding the amount of information that can be shared in any migration phase. We do not present the model nor all results obtained in~\cite{BadkobehLS15} in detail. The main result which is interpretable and comparable to the others presented in this book chapter is summarized by the following theorem. 

\begin{theorem}[Table~1 in~\cite{BadkobehLS15}]
\label{thm:BBCdistributed}
The $\lambda$-distributed unary unbiased black-box complexity of the class of all unimodal functions with $\Theta(n)$ different function values satisfies the following bounds:\\
\vspace{0.5ex}
\begin{center}
\begin{tabular}{c|c|c|c}
& \text{Ring Topology} & 
\text{Grid/Torus} & 
\text{Complete Topology}\\ 
\hline
Upper Bound 
& $O(\lambda n^{3/2}+n^2)$ 
& $O(\lambda n^{4/3}+n^2)$  
& \multirow{2}{*}{$\Theta(\lambda n +n^2)$}\\
Lower Bound 
& $\Omega(\lambda n+ \lambda^{2/3} n^{5/3}+n^2)$ 
& $\Omega(\lambda n+ \lambda^{3/4} n^{3/2}+n^2)$ &
\end{tabular}	
\end{center}
The lower bound for the grid applies to arbitrary side lengths, while the upper bound holds for the grid with $\sqrt{\lambda}$ islands in each of the two dimensions. 
\end{theorem}
The upper bounds in Theorem~\ref{thm:BBCdistributed} are achieved by a parallel (1+1)~EA, in which every node migrates its complete information after every round. The lower bounds are shown to hold already for a sub-problem called ``random short path'', which is a collection of problems which all have a global optimum in some point with exactly $n/2$ ones. A short path of Hamming-1 neighbors leads to this optimum. The paths start at the all-ones string. Search points that do not lie on the path lead the optimization process towards the all-ones string; their objective values equal the number of ones in the string.

\subsection{Elitist Black-Box Complexity}\label{sec:BBCelitist}

One of the most relevant questions in black-box optimization is how to avoid getting stuck in local optima. Essentially, two strategies have been developed. 

\textbf{Non-elitist selection.} The first idea is to allow the heuristics to direct their search towards search points that are, a priori, less favorable than the current-best solutions in the memory. This can be achieved, for example, by accepting into the memory (``population'') search points of function values that are smaller than the current best solutions. We refer to such selection procedures as \emph{non-elitist selection}. Non-elitist selection is used, for example, in the Metropolis algorithm~\cite{Metropolis53}, Simulated Annealing~\cite{SA83}, and, more recently, in the biology-inspired ``Strong Selection, Weak Mutation'' framework~\cite{SSWM}. 

\textbf{Global Sampling.} A different strategy to overcome local optima is \emph{global sampling.} This approach is used, most notably, by evolutionary and genetic algorithms, but also by swarm optimizers like ant colony optimization techniques~\cite{ACOBuch} and estimation-of-distribution algorithms (EDAs, cf.~Chapter~\ref{chap:EDA} of this book [link to the chapter on EDAs will be added]). The underlying idea of global sampling is to select new solution candidates not only locally in some pre-defined neighborhood of the current population, but to reserve some positive probability to sample far away from these solutions. Very often, a truly global sampling operation is used, in which \emph{every} point $x \in S$ has a positive probability of being sampled. This probability typically decreases with increasing distance to the current-best solutions. Standard bit mutation with bit flip probabilities $p<1/2$ is such a global sampling strategy. 

Global sampling and non-elitist selection can certainly be combined, and several attempts in this direction have been made. The predominant selection strategy used in combination with global sampling, however, is \emph{truncation selection}. Truncation selection is a natural implementation of Darwin's ``survival of the fittest'' paradigm in an optimization context: given a collection $\mathcal{P}$ of search points, and a population size $\mu$, truncation selection chooses from $\mathcal{P}$ the $\mu$ search points of largest function values and discards the other, breaking ties arbitrarily or according to some rule such as favoring offspring over parents or favoring geno- or phenotypic diversity. 

\begin{figure}[t]
\begin{framed}
\begin{center}
\includegraphics[scale=1.5]{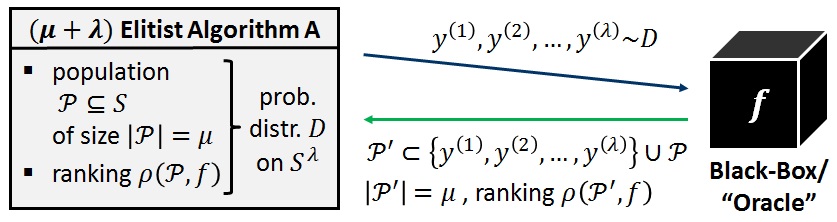}
\end{center}
\caption{A $(\mu+\lambda)$ elitist black-box algorithm stores the $\mu$ previously evaluated search points of largest function values (ties broken arbitrarily or according to some specified rule) and the ranking of these points induced by $f$. Based on this information, it decides upon a strategy from which the next $\lambda$ search points are sampled. From the $\mu+\lambda$ parent and offspring solutions, those $\mu$ search points that have the largest function values form the population for the next iteration.}
\end{framed}
\label{fig:BBCelitist}
\end{figure}

To understand the influence that this \emph{elitist selection} behavior has on the performance of black-box heuristics, the \emph{elitist black-box model} has been introduced in~\cite{DoerrL15model} (the journal version is to appear as~\cite{DoerrL17ECJ}). The elitist black-box model combines features of the memory-restricted and the ranking-based black-box models with an enforced truncation selection. More precisely, the $(\mu+\lambda)$ \emph{elitist black-box model} covers all algorithms that follow the pseudo-code in Algorithm~\ref{alg:BBCelitist}. We use here an adaptive initialization phase. A non-adaptive version as in the $(\mu+\lambda)$ memory-restricted black-box model can also be considered. This and other subtleties like the tie-breaking rules for search points of equal function values can result in different black-box complexities. It is therefore important to make very precise the model with respect to which a bound is claimed or shown to hold.

\begin{algorithm2e}
 \textbf{Initialization:} \\
 \Indp
 $X \assign \emptyset$\;
 \For{$i=1,\ldots,\mu$}{
 Depending only on the multiset $X$ and the ranking $\rho(X,f)$ of $X$ induced by $f$, choose a probability distribution $D^{(i)}$ over $S$ and sample from it $x^{(i)}$\;
 Set $X \assign X \cup \{ x^{(i)}\}$ and query the ranking $\rho(X,f)$ of $X$ induced by $f$\;
 }
 \Indm
 \textbf{Optimization:}	
 \For{$t=1,2,3,\ldots$}{
 		Depending only on the multiset $X$ and the ranking $\rho(X,f)$ of $X$ induced by $f$
	\label{line:BBCelitist-mut}	choose a probability distribution $D^{(t)}$ on $S^{\lambda}$ and 
		sample from it $y^{(1)},\ldots,y^{(\lambda)} \in S$\;
		Set $X \assign X \cup \{y^{(1)},\ldots,y^{(\lambda)}\}$ and query the ranking $\rho(X,f)$ of $X$ induced by $f$\;
  \lFor{$i=1,\ldots, \lambda$}{
  	\label{line:BBCelitist-selection} Select $x \in \arg\min X$ and update $X \assign X \setminus \{x\}$}
	 }
 \caption{The $(\mu+\lambda)$ elitist black-box algorithm for maximizing an unknown function $f:S \rightarrow \R$}
\label{alg:BBCelitist}
\end{algorithm2e}

The elitist black-box model models in particular all $(\mu+\lambda)$~EAs, RLS, and other hill climbers. It does not cover algorithms using non-elitist selection rules like Boltzmann selection, tournament selection, or fitness-proportional selection. Figure~\ref{fig:BBCelitist} illustrates the $(\mu+\lambda)$ elitist black-box model. As a seemingly subtle, but possible influential difference to the parallel black-box complexities introduced in Section~\ref{sec:BBCparallel}, note that in the elitist black-box model the offspring sampled in the optimization phase do not need to be independent of each other. If, for example, an offspring $x$ is created by crossover, in the $(\mu+\lambda)$ elitist black-box model with $\lambda \ge 2$ we allow to also create another offspring $y$ from the same parents whose entries $y_i$ in those positions $i$ in which the parents do not agree equals $1-x_i$. These two offspring are obviously not independent of each other. It is nevertheless required in the $(\mu+\lambda)$ elitist black-box model that the $\lambda$ offspring are created \emph{before} any evaluation of the offspring happens. That is, the $k$-th offspring may \emph{not} depend on the ranking or fitness of the first $k-1$ offspring.

Combining already several features of previous black-box model, the elitist black-box model can be further restricted to cover only those elitist black-box algorithms that sample from unbiased distributions. For this \emph{unbiased elitist black-box model,} we require that the distribution $p^{(t)}$ in line~\ref{line:BBCelitist-mut} of Algorithm~\ref{alg:BBCelitist} is unbiased (in the sense of Section~\ref{sec:BBCunbiased}). Some of the results mentioned below also hold for this more restrictive class. 

\subsubsection{(Non-)Applicability of Yao's Minimax Principle}\label{sec:BBCelitistYao} 
An important difficulty in the analysis of elitist black-box complexities is the fact that Yao's minimax principle (Theorem~\ref{thm:BBCYao}) cannot be directly applied to the elitist black-box model, since in this model the previously exploited fact that randomized algorithms are convex combinations of deterministic ones does not apply, cf.~\cite[Section 2.2]{DoerrL17ECJ} for an illustrated discussion. As discussed in the previous sections, Yao's minimax principle is \emph{the} most important tool for proving lower bounds in the black-box complexity context, and we can hardly do without. A natural workaround that allows to nevertheless employ this technique is to extend the collection $\A$ of elitist black-box algorithms to some superset $\A'$ in which every randomized algorithm \emph{can} be expressed as a probability distribution over deterministic ones. A lower bound shown for this broader class $\A'$ trivially applies to all elitist black-box algorithms. Finding extensions $\A'$ that do not decrease the lower bounds by too much is the main difficulty to overcome in this strategy.


\subsubsection{Exponential Gaps to Previous Models}\label{sec:BBCeitistexponential}

In~\cite[Section~3]{DoerrL17ECJ} it is shown that already for quite simple function classes there can be an exponential gap between the efficiency of elitist and non-elitist black-box algorithms; and this even in the very restrictive (1+1) unary unbiased elitist black-box complexity model. This shows that heuristics can sometimes benefit quite crucially from eventually giving preference to search points of fitness inferior to that of the current best search points. The underlying intuition for these results is that elitist algorithms do not work very well if there are several local optima that the algorithm needs to explore in order to determine the best one of them. 

\subsubsection{The Elitist Black-Box Complexity of OneMax}\label{sec:BBCeitistOM}
 
As we have discussed in Sections~\ref{sec:BBCmemory} and~\ref{sec:BBCcomparison}, respectively, the (1+1) memory-restricted and the ranking-based black-box complexity of \onemax is only of order $n/\log n$. In contrast, it is easy to see that the combined (1+1) memory-restricted ranking-based black-box model does not allow for algorithms that are faster than linear in $n$, as can easily be seen by standard information-theoretic considerations. In~\cite{DoerrL15OM} (journal version is to appear as~\cite{DoerrL17OM}) it is shown that this linear bound is tight. Whether or not it applies to the (1+1) elitist model remains unsolved, but it is shown in~\cite{DoerrL17OM} that the expected time needed to optimize \onemax with probability at least $1-\eps$ is linear for every constant $\eps>0$. This is the so-called Monte Carlos black-box complexity that we shall briefly discuss in Section~\ref{sec:BBCconclusions}. The following theorem summarizes the bounds presented in~\cite{DoerrL17OM}. Without detailing this further, we note that~\cite[Section~9]{DoerrL17OM} also introduces and studies a comma-variant of the elitist black-box model. 

\begin{theorem}[\cite{DoerrL17OM}] 
\label{thm:BBCelitistOM}
The (1+1) memory-restricted ranking-based black-box complexity of \onemax is $\Theta(n)$. 

For $1<\lambda < 2^{n^{1-\eps}}$, $\eps>0$ being an arbitrary constant, the $(1+\lambda)$ memory-restricted ranking-based black-box complexity of \onemax is $\Theta(n/\! \log \lambda)$ (in terms of generations), 
while for $\mu=\omega(\log^2(n)/\log\log n)$ its $(\mu+1)$ memory-restricted ranking-based black-box complexity is $\Theta(n/\! \log \mu)$. 

For every constant $0<\eps<1$ there exists a (1+1) elitist black-box algorithm that finds the optimum of any \onemax instance in time $O(n)$ with probability at least $1-\eps$, and this running time is asymptotically optimal.

For constant $\mu$, the $(\mu+1)$ elitist black-box complexity of \onemax is at most $n+1$.

For $\delta>0$, $C >0$, $2 \leq \lambda< 2^{n^{1-\delta}}$, and suitable chosen $\eps = O(\log^2(n)\log \log(n)\log( \lambda)/n)$ there exists a $(1+\lambda)$ elitist black-box algorithm that needs at most $O(n/\! \log \lambda)$ generations to optimize \onemax with probability at least $1-\eps$.

For $\mu=\omega(\log^2 (n)/\! \log\log n) \cap O(n/\! \log n)$ and every constant $\eps>0$, there is a $(\mu+1)$ elitist black-box algorithm optimizing \onemax in time $\Theta(n/\! \log \mu)$ with probability at least $1-\eps$.

There exists a constant $C>1$ such that for $\mu\geq Cn/\! \log n$, the $(\mu+1)$ elitist black-box complexity is $\Theta(n/\! \log n)$.
\end{theorem}

\subsubsection{The Elitist Black-Box Complexity of LeadingOnes}\label{sec:BBCeitistLO}

The (1+1) elitist black-box complexity of \leadingones is studied in~\cite{DoerrL16} (journal version is to appear as~\cite{DoerrL17LO}). Using the approach sketched in Section~\ref{sec:BBCelitistYao}, the following result is derived.

\begin{theorem}[Theorem~1 in~\cite{DoerrL17LO}] 
\label{thm:BBCelitistLO}
The (1+1) elitist black-box complexity of \leadingones is $\Theta(n^2)$. This bound holds also in the case that the algorithms have access to (and can make use of) the absolute fitness values of the search points in the population, and not only their rankings, i.e., in the (1+1) memory-restricted black-box model with enforced truncation selection. 
\end{theorem}

The $(1+1)$ elitist complexity of \textsc{LeadingOnes} is thus considerably larger than its unrestricted one, which is known to be of order $n\log\log n$, as discussed in Theorem~\ref{thm:BBCunrestrictedLO}.

It is well known that the quadratic bound of Theorem~\ref{thm:BBCelitistLO} is matched by classical $(1+1)$-type algorithms such as the (1+1)~EA, RLS, and others. 

\subsubsection{The Unbiased Elitist Black-Box Complexity of Jump}\label{sec:BBCeitistjump}

Some shortcomings of previous models can be eliminated when they are combined with an elitist selection requirement. This is shown in~\cite{DoerrL17ECJ} for the already discussed $\jump_k$ function. 
\begin{theorem}[Theorem~9 in~\cite{DoerrL17ECJ}] 
\label{thm:BBCelitistjump}
For $k=0$ the unary unbiased (1+1) elitist black-box complexity of $\jump_k$ is $\Theta(n \log n)$. For $1 \le k \le \tfrac{n}{2}-1$ it is of order $\binom{n}{k+1}$.
\end{theorem}
The bound in Theorem~\ref{thm:BBCelitistjump} is non-polynomial for $k=\omega(1)$. This is in contrast to the unary unbiased black-box complexity of $\jump_k$, which, according to Theorem~\ref{thm:BBCunbiasedjump}, is polynomial even for extreme values of~$k$.

\section{Summary of Known Black-Box Complexities for OneMax and LeadingOnes}
\label{sec:BBCtables}

For a better identification of open problems concerning the black-box complexity of the two benchmark functions \onemax and \leadingones, we summarize the bounds that have been presented in previous sections. 

The first table summarizes known black-box complexities of $\onemax_n$ in the different models. The bound for the $\lambda$-parallel black-box model assumes $\lambda \le e^{\sqrt{n}}$. The bounds for the $(1+\lambda)$ and the $(1,\lambda)$ elitist model assume $1< \lambda < 2^{n^{1-\eps}}$ for some $\eps>0$. Finally, the bound for the $(\mu+1)$ model assumes that $\mu = \omega(\log^2 n/\! \log \log n)$ and $\mu \leq n$. \vspace{2ex}

\begin{center}
\begin{tabular}{lcc}
\textbf{Model}    
& {\textbf{Lower Bound}} 
& {\textbf{Upper Bound}}\\ \hline 
unrestricted 
	& \multicolumn{2}{c}{$\Theta(n/\! \log n)$}\\
	\hline
unbiased, arity $1$
	& \multicolumn{2}{c}{$\Theta(n \log n)$}
	\\ 
unbiased, arity $2 \leq k \leq \log n$ 
	& $\Omega(n/\! \log n)$ 
	& $O(n/k)$  
	\\ \hline
ranking-based (unrestricted)  
	&  \multicolumn{2}{c}{$\Theta(n/\! \log n)$}
	\\ 	
ranking-based unbiased, arity  $1$
	&  \multicolumn{2}{c}{$\Theta(n \log n)$}
	\\ 
ranking-based unbiased, arity $ 2 \leq k \leq n$
	& $\Omega(n/\! \log n)$ 
	& $O(n /\log k)$  
	\\
	\hline
(1+1) comparison-based
	&  \multicolumn{2}{c}{$\Theta(n)$}
	\\ \hline				
(1+1) memory-restricted  
	&  \multicolumn{2}{c}{$\Theta(n/\! \log n)$}
	\\ \hline		
$\lambda$-parallel unbiased, arity $1$
	&  \multicolumn{2}{c}{$\Theta\big(\frac{\lambda n}{\log(\lambda)}+n \log n \big)$}
	\\ \hline					
(1+1) elitist Las Vegas
	& $\Omega(n)$ 
	& $O(n\log n)$  
	\\ 				
(1+1) elitist $\log n/n$-Monte Carlo
	&  \multicolumn{2}{c}{$\Theta(n)$}
	\\ 				
(2+1) elitist Monte Carlo/Las Vegas
	&  \multicolumn{2}{c}{$\Theta(n)$}
	\\ 				
(1+$\lambda$) elitist Monte Carlo ($\#$ generations)	
	&  \multicolumn{2}{c}{$\Theta(n/\! \log \lambda)$}
	\\ 		
($\mu$+1) elitist Monte Carlo
	&  \multicolumn{2}{c}{$\Theta(n/\! \log \mu)$}
	\\ 		
($1,\lambda$) elitist Monte Carlo/Las Vegas ($\#$ generations)
	& \multicolumn{2}{c}{$\Theta(n/\! \log \lambda)$}
\end{tabular}
\end{center}

%

The next table summarizes known black-box complexities of $\leadingones_n$. The upper bounds for the unbiased black-box models also hold in the ranking-based variants.\vspace{2ex}

\begin{center}
\begin{tabular}{lcc}
\textbf{Model}    
& {\textbf{Lower Bound}} 
& {\textbf{Upper Bound}}\\ \hline 
unrestricted 
	& \multicolumn{2}{c}{$\Theta(n \log\log n)$} 
	\\ \hline
unbiased, arity $1$
	&\multicolumn{2}{c}{$\Theta(n^2)$}
	\\ 
unbiased, arity $2$ 
	& $\Omega(n \log\log n)$
	& $O(n \log n)$ 
	\\ 
unbiased, arity $\geq 3$ 
	& $\Omega(n \log\log n)$ 
	& $O(n \log (n)/\log\log n)$ 
	\\ 	\hline	
$\lambda$-parallel unbiased, arity $1$ 
	&  $\Omega\big(\frac{\lambda n}{\log(\lambda/n)}+n^2 \big)$
	& $O(\lambda n + n^2)$
	\\ \hline				
(1+1) elitist 
	& $\Omega(n^2)$ 
	& $O(n^2)$ 
	\\ 	\hline		
\end{tabular}
\end{center}

\section{From Black-Box Complexity to Algorithm Design}\label{sec:BBCalgo}

In the previous sections, the focus of our attention has been on computing performance limits for black-box optimization heuristics. In some cases, e.g., for \onemax in the unary unbiased black-box model and for \leadingones in the (1+1) elitist black-box model, we have obtained lower bounds that are matched by the performance of well-known standard heuristics like RLS or the (1+1)~EA. For several other models and problems, however, we have obtained black-box complexities that are much smaller than the expected running times of typical black-box optimization techniques. As discussed in the introduction, two possible reasons for this discrepancy exist. Either the respective black-box models do not capture very well the complexity of the problems for heuristics approaches, or there are ways to improve classical heuristics by novel design principles. 

In the restrictive models discussed in Sections~\ref{sec:BBCmemory} to~\ref{sec:BBCcombined}, we have seen that there is some truth in the first possibility. For several optimization problems, we have seen that their complexity increases if the class of black-box algorithms is restricted to subclasses of heuristics that all share some properties that are commonly found in state-of-the-art optimization heuristics. Here in this section we shall demonstrate that this is nevertheless not the end of the story. We discuss two examples where a discrepancy between black-box complexity and the performance of classical heuristics can be observed, and we show how the analysis of the typically rather artificial problem-tailored algorithms can inspire the design of new heuristics. 

\subsection{The \texorpdfstring{$(1+(\lambda,\lambda))$~Genetic Algorithm}{(1+(lambda,lambda))~Genetic Algorithm}}\label{sec:BBCalgoga}

Our first example is a binary unbiased algorithm, which optimizes \onemax more efficiently than any classical unbiased heuristic, and provably faster than any unary unbiased black-box optimizer.
 
We recall from Theorems~\ref{thm:BBConemaxunary}, and~\ref{thm:BBConemaxbinary} that the unary unbiased black-box complexity of $\onemax_n$ is $\Theta(n \log n)$, while its binary unbiased one is only $O(n)$. The linear-time algorithm flips one bit at a time, and uses a simple, but clever encoding to store which bits have been flipped already. This way, it is a rather problem-specific algorithm, since it ``knows'' that a bit that has been tested already does not need to be tested again. The algorithm is therefore not very suitable for non-separable problems, where the influence of an individual bit depends on the value of several or all other bits. 

Until recently, all existing running time results indicated that general-purpose unbiased heuristics need $\Omega(n \log n)$ function evaluations to optimize \onemax, so that the question if the binary unbiased black-box model is too ``generous'' imposed itself. In~\cite{DoerrDE15,DoerrD15self} this question was answered negatively; through the presentation of a novel binary unbiased black-box heuristic that optimizes \onemax in expected linear time. This algorithm is the $(1+(\lambda,\lambda))$~GA. Since the algorithm itself will be discussed in more detailed in Chapter~\ref{chap:adaptive} [link to the chapter on non-static parameter choices will be added], we present here only the main ideas behind it.

Disregarding some technical subtleties, one observation that we can make when regarding the linear-time binary unbiased algorithm for \onemax is that when it test the value of a bit, the amount of information that it obtains is the same whether or not the offspring has a better function value. In other words, the algorithm equally benefits from offspring that are better or worse than the previously best. A similar observation applies to all $O(n/\log n)$ algorithms for \onemax discussed in Sections~\ref{sec:BBCunrestrictedResults} to~\ref{sec:BBCcombined}. All these algorithms do not strive to sample search points of large objective value, but rather aim at maximizing the amount of information that they can learn about the problem instance at hand. This way, they substantially benefit also from those search points that are (much) worse than other already evaluated ones. 

Most classical black-box heuristics are different. They store only the best so far solutions, or use inferior search points only to create diversity in the population. Thus, in general, they are not very efficient in learning from ``bad'' samples (where we consider a search point to be ``bad'' if it has small function value). When a heuristic is close to a local or a global optimum (in the sense that it has identified search points that are not far from these optima) it samples, in expectation, a fairly large number of search points that are wore than the current-best solutions. Not learning from these offspring results in a significant number of ``wasted'' iterations from which the heuristics do not benefit. This observation was the starting point for the development of the $(1+(\lambda,\lambda))$~GA. 

Since the unary unbiased black-box complexity of \onemax is $\Omega(n \log n)$, it was clear for the development of the $(1+(\lambda,\lambda))$~GA that a $o(n \log n)$ unbiased algorithm must be at least binary. This has led to the question how recombination can be used to learn from inferior search points. The following idea emerged. For illustration purposes, assume that we have identified a search point $x$ of function value $\OM_z(x)=n-1$. From the function value, we know that there exists exactly one bit that we need to flip in order to obtain the global optimum $z$. Since we want to be unbiased, the best mutation seems to be a random 1-bit flip. This has probability $1/n$ of returning $z$. If we did this until we identified $z$, the expected number of samples would be $n$, and even if we stored which bits have been flipped already, we would need $n/2$ samples on average. 

Assume now that in the same situation we flip $\ell>1$ bits of $x$. Then, with a probability that depends on $\ell$, we have only flipped already optimized bits (i.e., bits in positions $i$ for which $x_i=z_i$) to $1-z_i$, thus resulting in an offspring of function value $n-1-\ell$. However, the probability that the position $j$ in which $x$ and $z$ differ is among the $\ell$ positions is $\ell/n$. If we repeat this experiment some $\lambda$ times, independently of each other and always starting with $x$ as ``parent'', then the probability that $j$ has been flipped in at least one of the offspring is $1-(1-\ell/n)^{\lambda}$. For moderately large $\ell$ and $\lambda$ this probability is sufficiently large for us to assume that among the $\lambda$ offspring there is at least one in which $j$ has been flipped. Such an offspring distinguishes itself from the others by a function value of $n-1-(\ell-1)+1=n-\ell+1$ instead of $n-\ell-1$. Assume that there is one such offspring $x'$ among the $\lambda$ independent samples crated from $x$. When we compare $x'$ with $x$, they differ in exactly $\ell$ positions. In $\ell-1$ of these, the entry of $x$ equals that of $z$. Only in the $j$-th position the situation is reversed: $x'_j=z_j \neq x_j$. We would therefore like to identify this position $j$, and to incorporate the bit value $x'_j$ into $x$. 

So far we have only used mutation, which is a unary unbiased operation. At this point, we want to compare and merge two search points, which is one of the driving motivations behind \emph{crossover}. Since $x$ clearly has more ``good'' bits than $x'$, a uniform crossover, which takes for each position $i$ its entry uniformly at random from any of its two parents, does not seem to be a good choice. We would like to add some bias to the decision making process, in favor of choosing the entries of $x$. This yields to a biased crossover, which takes for each position $i$ its entry $y_i$ from $x'$ with some probability $p<1/2$ and from $x$ otherwise. The hope is to choose $p$ in a way that in the end only good bits are chosen. Where $x$ and $x'$ are identical, there is nothing to worry, as these positions are correct already (and, in general, we have no indication to flip the entry in this position). So we only need to look at those $\ell$ positions in which $x$ and $x'$ differ. The probability to make only good choices; i.e., to select $\ell-1$ times the entry from $x$ and only for the $j$-th position the entry from $x'$ equals $p(1-p)^{\ell-1}$. This probability may not be very large, but when we do again $\lambda$ independent trials, then the probability to have created $z$ in at least one of the trials equals $1-(1-p(1-p)^{\ell-1})^{\lambda}$. As we shall see, for suitable parameter values $p$, $\lambda$, and $\ell$, this expression is sufficiently large to gain over the $O(n)$ strategies discussed above. Since we want to sample exactly one out of the $\ell$ bits in which $x$ and $x'$ differ, it seems intuitive to set $p=1/\ell$, cf. the discussion in~\cite[Section~2.1]{DoerrDE15}. 

Before we summarize the main findings, let us briefly reflect on the structure of the algorithm. In the \emph{mutation step} we have created $\lambda$ offspring from $x$, by a mutation operator that flips some $\ell$ random bits in $x$. This is a unary unbiased operation. From these $\lambda$ offspring, we have selected one offspring $x'$ with largest function value among all offspring (ties broken at random). In the \emph{crossover phase,} we have then created again $\lambda$ offspring, by recombining $x$ and $x'$ with a biased crossover. This biased crossover is a binary unbiased variation operator. The algorithm now chooses from these $\lambda$ recombined offspring one that has largest function vales (For \onemax ties can again be broken at random, but for other problems it can be better to favor individuals that are different from $x$, cf.~\cite[Section~4.3]{DoerrDE15}). This selected offspring $y$ replaces $x$ if it is at least as good as $x$, i.e., if $f(y) \ge f(x)$. 

We see that we have only employed unbiased operations, and that the largest arity in use is two. Both variation operators, mutation and biased crossover, are standard operators in the evolutionary computation literature. What is novel is that crossover is used as a \emph{repair mechanism,} and after the mutation step. 

We also see that this algorithm is ranking-based, and even comparison-based in the sense that it can be implemented in a way in which instead of querying absolute function values only a comparison of the function values of two search points is asked for. Using information-theoretic arguments as described in Section~\ref{sec:BBClower} it is then not difficult to show that for any (adaptive or non-adaptive) parameter setting the best expected performance of the $(1+(\lambda,\lambda))$~GA on $\onemax_n$ is at least linear in the problem dimension $n$.  

The following theorem summarizes some of the results on the expected running time of the $(1+(\lambda,\lambda))$~GA on \onemax. An exhaustive discussion of these results can be found in~\cite{DoerrD18ga}. In particular the fitness-dependent and the self-adjusting choice of the parameters will also be discussed in Section~\ref{sec:SAselfga} of this book [link to the chapter on non-static parameter choices will be added]. 

\begin{theorem}[from \cite{DoerrDE15,DoerrD15self,DoerrD15tight,Doerr16}]
\label{thm:BBCga}
The $(1+(\lambda,\lambda))$~GA is a binary unbiased black-box algorithm. For mutation strength $\ell$ sampled from the binomial distribution $\text{Bin}(n,k/n)$ and a crossover bias $p=1/k$ the following holds:
\begin{itemize}
	\item For $k=\lambda=\Theta(\sqrt{\log(n) \log\log(n)/\log\log\log(n)})$ the expected optimization time of the $(1+(\lambda,\lambda))$~GA on $\onemax_n$ is $O(n \sqrt{\log(n) \log\log\log(n) / \log\log(n)})$.
	\item No static parameter choice of $\lambda \in [n]$, $k \in [0..n]$, and $p \in [0,1]$ can give a better expected running time. 
	\item There exists a fitness-dependent choice of $\lambda$ and $k=\lambda$ such that the $(1+(\lambda,\lambda))$~GA has a linear expected running time on \onemax. 
	\item A linear expectec running time can also be achieved by a self-adjusting choice of $k=\lambda$. 
\end{itemize}
\end{theorem}

Note that these results answer one of the most prominent long-standing open problems in evolutionary computation: the usefulness of crossover in an optimization context. Previous and other recent examples exist where crossover has been shown to be beneficial~\cite{JansenW02,FischerW04,Sudholt05,DoerrT09,DoerrHK12,Sudholt12,DoerrJKNT13tcs,KotzingGECCO16,Dang2017}, but in all of these works, either non-standard problems or operators are regarded or the results hold only for uncommon parameter settings, or substantial additional mechanisms like diversity-preserving selection schemes are needed to make crossover really work. To our knowledge, Theorem~\ref{thm:BBCga} is thus the first example that proves advantages of crossover in a natural algorithmic setting for a simple hill climbing problem. 
 
Without going into detail, we mention that the $(1+(\lambda,\lambda))$~GA has also been analyzed on a number of other benchmark problems, both by theoretical~\cite{BuzdalovD17} and by empirical~\cite{DoerrDE15,MironovichB15,GoldmanP15} means. These results indicate that the concept of using crossover as a repair mechanism can be useful far beyond \onemax.

\subsection{Randomized Local Search with Self-Adjusting Mutation Strengths}\label{sec:BBCalgoRLS}
Another example highlighting the impact that black-box complexity studies can have on the design of heuristic optimization techniques was presented in~\cite{DoerrDY16PPSN}. This works build on~\cite{DoerrDY16}, where the tight bound for the unary unbiased black-box complexity of \onemax stated in Theorem~\ref{thm:BBConemaxunary} was presented. This bound is attained, up to an additive difference that is sublinear in $n$, by a variant of Randomized Local Search that in each iterations chooses a value $r$ that depends on the function value $\OM_z(x)$ of the current-best search point $x$ and then uses the $\text{flip}_r$ variation operator introduced in Definition~\ref{def:BBCflip_r} to create an offspring~$y$. The offspring $y$ replaces $x$ if and only if $\OM_z(y) \ge \OM_z(x)$. The dependence of $r$ on the function value $\OM_z$ is rather complex and difficult to compute directly, cf. the discussion in~\cite{DoerrDY16}. Quite surprisingly, a self-adjusting choice of $r$ is capable of identifying the optimal mutation strengths $r$ in all but a small fraction of the iterations. This way, RLS with this self-adjusting parameter choice achieves an expected running time on \onemax that is only by an additive $o(n)$ term worse than that that of the theoretically optimal unary unbiased black-box algorithm. 

The algorithm from~\cite{DoerrDY16PPSN} will be discussed in Chapter~\ref{chap:adaptive} of this book [link to the chapter on non-static parameter choices will be added]. For the context of black-box optimization, it is interesting to note that the idea to take a closer look into self-adjusting parameter choices, as well as our ability to investigate the optimality of such non-static parameter choices are deeply rooted in the study of black-box complexities.

\section{From Black-Box Complexity to Mastermind}
\label{sec:BBCmastermind}

In~\cite{DoerrDST16} the black-box complexity studies for \onemax were extended to the following generalization of \onemax to functions over an alphabet of size $k$. For a given string $z \in [0..k-1]^n$, the \emph{Mastermind} function $f_{z}$ assigns to each search point $x \in [0..k-1]^n$ the number of positions in which $x$ and $z$ agree. Thus, formally, 
$$f_z:[0..k-1]^n \to \R, x \mapsto |\{ i \in [n] \mid x_i=z_i\}|.$$ 
The collection $\{f_{z} \mid z \in [0..k-1]^n \}$ of all such Mastermind functions forms the Mastermind problem of $n$ \emph{positions} and $k$ \emph{colors.}

The Mastermind problem models the homonymous board game, which had been very popular in North America and in the Western parts of Europe in the seventies and eighties of the last century. More precisely, it models a variant of this game, as in the original Mastermind game information is provided also about \emph{colors} $x_i$ that appear in $z$ but not in the same position $i$; cf.~\cite{DoerrDST16} for details and results about this Mastermind variant using \emph{black and white pegs}.

Mastermind and similar \emph{guessing games} have been studied in the Computer Science literature much before the release of Mastermind as a commercial board game. As we have discussed in Section~\ref{sec:BBCunrestrictedResults}, the case of $k=2$ colors (this is the \onemax problem) has already been regarded by Erd\H os and R\'enyi and several other authors in the early 1960s. These authors were mostly interested in the complexity- and information-theoretic aspects of this problem, and/or its cryptographic nature. The playful character of the problem, in turn, was the motivation of Knuth~\cite{Knuth77}, who computed an optimal strategy that solves any Mastermind instance with $n=4$ positions and $k=6$ colors in at most five guesses. 

The first to study the general case with arbitrary values of $k$ was Chv\'atal~\cite{Chvatal83}. 

\begin{theorem}[Theorem~1 in~\cite{Chvatal83}]
\label{thm:BBCmastermindchvatal} 
For every $k\ge 2$ the unrestricted black-box complexity of the Mastermind game with $n$ positions and $k$ colors is $\Omega(n \log k / \log n)$. For $\eps>0$ and $k \le n^{1-\eps}$, it is at most $(2+\eps) n (1+2 \log k)/\log (n/k)$. 
\end{theorem}
Note that for $k \le n^{1-\eps}$, $\eps>0$ being a constant, Theorem~\ref{thm:BBCmastermindchvatal} gives an asymptotically tight bound of $\Theta(n \log k / \log n)$ for the $k$-color, $n$-position Mastermind game. Similarly to the random guessing strategy by Erd\H os and R\'enyi, it is sufficient to query this many \emph{random} queries, chosen independently and uniformly at random from $[0..k-1]^n$. That is, no adaptation is needed for such combinations of $n$ and $k$ to \emph{learn} the secret target vector $z$.

This situation changes for the regime around $k=n$, which had been the focus of several subsequent works~\cite{Chen96,Goodrich09,JagerP11}. These works all show bounds of order $n \log n$ for the $k=n$ Mastermind problem. Originally motivated by the study of black-box complexities for randomized black-box heuristics, in~\cite{DoerrDST16} these bounds were improved to $O(n \log\log n)$. 

\begin{theorem}[Theorems~2.1 and in~\cite{DoerrDST16}]
\label{thm:BBCmastermindunser}
For Mastermind with $n$ positions and $k=\Omega(n)$ colors, the unrestricted black-box complexity of the $n$-position, $k$-color Mastermind game is $O(n \log \log n + k)$. For $k= o(n)$ it is $O\left(n \log \left(\frac{\log n}{\log (n/k)}\right)\right)$. 
\end{theorem}

Like the $O(n/\log n)$ bound for the case $k=2$, the bounds in Theorem~\ref{thm:BBCmastermindunser} can be achieved by \emph{deterministic} black-box algorithms~\cite[Theorem~2.3]{DoerrDST16}. On the other hand, and unlike the situation regarded in Theorem~\ref{thm:BBCmastermindchvatal}, it can be shown that any (deterministic or randomized) $o(n \log n)$ algorithm for the Mastermind game with $k=\Theta(n)$ colors has to be \emph{adaptive,} showing that in this regime adaptive strategies are indeed more powerful than non-adaptive ones. 

\begin{theorem}[Theorem~4.1 and Lemma~4.2 in~\cite{DoerrDST16}] 
\label{thm:BBCmastermindnonadaptive}
The non-adaptive unrestricted black-box complexity of the Mastermind problem with $n$ positions and $k$ colors is $\Omega \left( \frac{n \log k}{\max\{\log(n/k), 1\}} \right)$. For $k=n$ this bound is tight, i.e., the non-adaptive unrestricted black-box complexity of the Mastermind problem with $n$ positions and $n$ colors is $\Theta(n \log n)$. 
\end{theorem}

Whether or not the $O(n \log\log n)$ upper bound of Theorem~\ref{thm:BBCmastermindunser} can be further improved remains a---seemingly quite challenging---open problem. To date, the best known lower bound is the linear one reported in~\cite{Chvatal83}. Some numerical results for different values of $k=n$ can be found in~\cite{Buzdalov16}, but extending these numbers to asymptotic results may require a substantially new idea or technique for proving lower bounds in the unrestricted black-box complexity model.


\section{Conclusion and Selected Open Problems}\label{sec:BBCconclusions}

In this chapter we have surveyed theory-driven approaches that shed light on the performance limits of black-box optimization techniques such as local search strategies, nature-inspired heuristics, and pure random search. We have presented a detailed discussion of existing results for these black-box complexity measures. We now highlight a few avenues for future work in this young research discipline. 

\textbf{Extension to other optimization problems.} In line with the existing literature, our focus has been on classes of classical benchmark problems such as \onemax, \leadingones, \textsc{jump}, MST, and SSSP problems, since for these problems we can compare the black-box complexity results with known running time results for well-understood heuristics. As in running time analysis, it would be highly desirable to extend these results to other problem classes. 

\textbf{Systematic investigation of combined black-box models.} In the years before around 2013 most research on black-box complexity was centered around the question how individual characteristics of state-of-the-art heuristics influence their performance. With this aim in mind, various black-box models have been developed that each restrict the algorithms with respect to some specific property; e.g., their memory-sizes, properties of their variation operators or of the selection mechanisms in use. Since 2013 we observe an increasing interest in combining two or more of such restrictions to obtain a better picture of what is needed to design algorithms that significantly excel over existing approaches. A systematic investigation of such combined black-box models constitutes one of the most promising avenues for future research.

\textbf{Tools to derive lower bounds.} To date, the most powerful technique to prove lower bounds for the black-box complexity of a problem is the information-theoretic approach, most notably in the form of Yao's minimax principle, and the simple information-theoretic lower bound presented in Theorem~\ref{thm:BBCinfotheo}. Refined variants of this theorem, designed to capture the situation in which the number of possible function values depend on the state of the optimization process, or where the probabilities for different objective values are non-homogeneous. Unfortunately, either the verification that the conditions under which these theorems apply, or the computation of a closed expression that summarizes the resulting bounds, are often very tedious, making these extension rather difficult to apply. Alternative tools for the derivation of lower bounds in black-box complexity contexts form another of the most desirable directions for future work. 

In particular for the $k$-ary unbiased black-box complexity of arities $k\ge 2$ we do not have any model-specific lower bounds. We therefore do not know, for example, if the linear bound for the binary unbiased black-box complexity of $\onemax_n$ or the $O(n \log n)$ bound for the binary unbiased black-box complexity of $\leadingones_n$ are tight, or whether the power of recombination is even larger than what these bounds, in comparison to the unary unbiased black-box complexities, indicate.

Another specifically interesting problem is raised by the $\Omega(n^2)$ lower bound for the (1+1) elitist black-box complexity of $\leadingones_n$ presented in Theorem~\ref{thm:BBCelitistLO}. It is conjectured in~\cite[Section~4]{DoerrL17LO} that this bound holds already for the (1+1) memory-restricted setting. A more systematic investigation of lower bounds in the memory-restricted black-box models would help to understand better the role of large populations in evolutionary computation, a question that is not very well understood to date. 

\textbf{Beyond worst-case expected optimization time as unique performance measure.} 
Black-box complexity as introduced in this chapter takes \emph{worst-case expected optimization time} as performance measure. This measure reduces the whole optimization procedure to one single number. This, naturally, carries several disadvantages. The same critique applies to running time analysis in general, which is very much centered around this single performance measure. Complementing performance indicators like \emph{fixed-budget} (cf.~\cite{JansenZ14}) and \emph{fixed-target} (cf.~\cite{CarvalhoD17}) results have been proposed in the literature, but unfortunately have not yet been able to attract significant resonance. Since these measure give a better picture on the \emph{anytime behavior} of black-box optimization techniques, we believe that an extension of existing black-box complexity results to such anytime statements would make it easier to communicate and to discuss the results with practitioners, for whom the anytime performance is often at least as important as expected optimization time. 

In the same context, one may ask if the \emph{expected} optimization time should be the only measure regarded. Clearly, when the optimization time $T(A,f)$ of an algorithm $A$ on a function $f$ is highly concentrated, its expectation is often much similar, and in particular of the same or a very similar asymptotic order than its median. Such concentration can be observed for the running time of classical heuristics on most of the benchmark problems regarded here in this chapter. At the same time, it is also not very difficult to construct problems for which such a concentration does provably not hold. In particular for multimodal problems, in which two or more local optima exist, the running time is often not concentrated. In~\cite[Section~3]{DoerrL17ECJ} examples are presented for which the probability to find a solution within a small polynomial given bound is rather large, but where---due to excessive running times in the remaining cases---the expected optimization time is very large. This motivated the authors to introduce the concept of \emph{$p$-Monte Carlo black-box complexity}. The \emph{$p$-Monte Carlo black-box complexity} of a class $\F$ of functions measures the time it takes to optimize any problem $f \in \F$ with failure probability at most $p$. It is shown that even for small~$p$, the $p$-Monte Carlo black-box complexity of a function class $\F$ can be smaller by an exponential factor than its traditional (expected) black-box complexity, which is referred to as \emph{Las Vegas black-box complexity} in~\cite{DoerrL17ECJ}. 

\subsection*{Acknowledgments}
This work was supported by a public grant as part of the Investissement d'avenir project, reference ANR-11-LABX-0056-LMH, LabEx LMH, in a joint call with the Gaspard Monge Program for optimization, operations research, and their interactions with data sciences. 

}
\newcommand{\etalchar}[1]{$^{#1}$}

\end{document}